\documentclass{article}

\usepackage{microtype}
\usepackage{graphicx}
\usepackage{subcaption}
\usepackage{booktabs}
\usepackage{hyperref}
\PassOptionsToPackage{numbers,compress}{natbib}

\usepackage[preprint]{neurips_2026}

\usepackage{amsmath}
\usepackage{amssymb}
\usepackage{mathtools}
\usepackage{algorithm}
\usepackage{wrapfig}

\usepackage{algpseudocodex}

\usepackage{amsthm}
\usepackage[acronym]{glossaries}
\usepackage{xcolor}
\usepackage{colortbl}
\usepackage{booktabs}
\usepackage{multirow}
\usepackage{enumerate}
\usepackage{enumitem}
\usepackage[subrefformat=parens]{subcaption}
\usepackage{tikz}
\usetikzlibrary{arrows.meta, positioning, calc, shapes.geometric, decorations.pathmorphing, shadows}

\usepackage[capitalize,noabbrev]{cleveref}

\theoremstyle{plain}
\newtheorem{theorem}{Theorem}[section]
\newtheorem{proposition}[theorem]{Proposition}

\theoremstyle{definition}

\theoremstyle{remark}

\usepackage[textsize=tiny]{todonotes}

\setacronymstyle{long-short}
\makenoidxglossaries

\newacronym{dm}{DM}{diffusion model}
\newacronym{ddpm}{DDPM}{Denoising Diffusion Probabilistic Models}
\newacronym{sgm}{SGM}{Score-based Generative Model}
\newacronym{ncsn}{NCSN}{Noise-Conditional Score Network}
\newacronym{ssde}{Score SDE}{stochastic differential equation}
\newacronym{mhvae}{MHVAE}{Markovian hierarchical variational Autoencoder}
\newacronym{lmc}{LMC}{Langevin Monte Carlo}
\newacronym{ebm}{EBM}{energy-based model}

\newacronym{sde}{SDE}{stochastic differential equation}
\newacronym{ode}{ODE}{ordinary differential equation}
\newacronym{pfode}{PF-ODE}{probability flow ordinary differential equations}
\newacronym{vae}{VAE}{variational Autoencoder}
\newacronym{oc}{OC}{optimal control}

\usepackage{amsmath,amsfonts,bm}

\def\eqref#1{Eq.~(\ref{#1})}

\def\1{\bm{1}}

\DeclareMathAlphabet{\mathsfit}{\encodingdefault}{\sfdefault}{m}{sl}
\SetMathAlphabet{\mathsfit}{bold}{\encodingdefault}{\sfdefault}{bx}{n}

\def\sR{{\mathbb{R}}}

\DeclareMathOperator*{\argmin}{arg\,min}

\usepackage{amsthm,amsmath,amssymb,amscd}
\usepackage{amsfonts}
 
\usepackage{xspace}

\theoremstyle{definition}

\theoremstyle{remark}
\newtheorem{rmk}{Remark}

\newtheorem*{proposition*}{Proposition}

\newcommand{\etc}{\textit{etc.}\xspace}
\newcommand{\eg}{{\it e.g.}}
\newcommand{\ie}{{\it i.e.}}

\newcommand{\apdxref}[1]{Appendix~{\bf \ref{#1}}}

\newcommand{\tabref}[1]{{Table~\ref{#1}}}

\newcommand{\kldiv}[2]{\mathop{\mathbb{D}_{\text{KL}}} \left[ #1 \,\|\, #2\right]}

\newcommand{\gauss}[2]{\mathcal{N}\left( {#1},\, {#2} \right)}

\def\x{\mathbf{x}}
\def\y{\mathbf{y}}
\def\z{\mathbf{z}}
\def\u{\mathbf{u}}

\newcommand{\X}{\mathbf{X}}

\newcommand{\I}{\boldsymbol{I}}

\PassOptionsToPackage{normalem}{ulem}

\usepackage{enumitem}
\usepackage{graphicx}

\floatstyle{ruled}
\newfloat{algorithm}{tbp}{loa}
\providecommand{\algorithmname}{Algorithm}
\floatname{algorithm}{\protect\algorithmname}

\definecolor{myRed}{HTML}{E63946}
\definecolor{errorRed}{HTML}{FF0000}    
\definecolor{myBlue}{HTML}{0077B6}
\definecolor{electricBlue}{HTML}{007BFF}
\definecolor{myOrange}{HTML}{FB8500}
\definecolor{myGreen}{HTML}{28A745}

\title{Inference-Time Attribute Distribution Alignment\\ for Unconditional Diffusion}

\author{
  Hao~Luan\textsuperscript{\rm 1} \qquad 
  See-Kiong~Ng\textsuperscript{\rm 1,2}  \qquad
  Chun~Kai~Ling\textsuperscript{\rm 1} \\
  \textsuperscript{\rm 1}School of Computing, National University of Singapore\\
  \textsuperscript{\rm 2}Institute of Data Science, National University of Singapore\\ 
  \texttt{
    haoluan@comp.nus.edu.sg \quad \{seekiong, chunkail\}@nus.edu.sg
  }
}

\begin{document}

\maketitle

\begin{abstract}

Inference-time controllable generation is essential for real-world applications of unconditional diffusion models. However, most existing techniques focus on individual samples, struggling in applications that require the sample \emph{population} to follow specific \emph{attribute} distributions (\emph{e.g.}, demographic balance or semantic proportions). We formalize this setting as the inference-time attribute distributional alignment problem for pretrained unconditional diffusion models. To address this, we cast inference-time attribute distributional alignment as an optimal control problem over the reverse diffusion process, viewing the process as the rollout of a dynamical system and augmenting it with additive, time-dependent perturbations as control. We solve for the perturbations using an optimal-control-based algorithm to optimize a differentiable distribution-matching objective while penalizing control effort to preserve data fidelity. Experiment results in image generation demonstrate that our proposed plug-and-play approach can better align attribute distributions to diverse and flexible test-time targets compared to baselines, without retraining or finetuning the pretrained diffusion model.\footnote{Code is available at {\texttt{\url{https://github.com/EdmundLuan/diffusion_attribute_distribution_alignment}}}.} 

\end{abstract}

\addtocontents{toc}{\protect\setcounter{tocdepth}{-10}}

\section{Introduction}

Diffusion models are powerful generative models that excel in representing complex data distributions and have been applied in a wide range of domains, including 
images \cite{ho2020denoising,Rombach_2022_CVPR}, 
videos \cite{ho2022video, bar-tal2024lumiere, hayakawa2025mmdisco}, 
graphs \citep{niu2020permutation,maderia2024generative,luan2025ddps}, 
and robotics \citep{janner2022diffuser,chi2023diffusion,feng2025doppler,carvalho2025mpdtro}, among others.
It is common that many user requirements are not known \emph{a priori} during training and may shift over time during deployment. 
This makes it impractical to retrain or finetune the model for every new requirement, and in turn motivates a plethora of \emph{inference-time} techniques that do not necessitate modifying the trained parameters of unconditional diffusion models for controllable generation, \eg, guidance \cite{dhariwal2021diffusion,chung2023diffusion,ye2024tfg,feng2024ltldog}, 
projection \citep{fishman2023metropolis,zampini2025trainingfree,liang2025simultaneous}, 
joint correlation \citep{ruan2023mm,zeng2024jedi,luan2025projected,hao2025chd}, \etc 

In many controllable generation applications, however, the goal is not merely for each \emph{individual} sample to satisfy a condition, but for the \emph{sample population} to follow a desired \emph{distribution} over \emph{oracle-induced abstract attributes}, \ie, the outputs obtained by applying an oracle function to generated samples. 
Such attributes may include styles or semantic categories in images, as identified by off-the-shelf classifiers, or behavior patterns exhibited by robotic trajectories, as determined by evaluation models. 
Controlling the \emph{distribution} of such induced attributes is essential in many real-world scenarios.
For example, fairness objectives in image generation involving human subjects naturally require balancing demographic attributes toward a uniform \emph{distribution} \cite{choi2020fair,parihar2024balancing,kang2025fairgen}.

We refer to this goal as the inference-time \underline{d}iffusion \emph{\underline{a}ttribute \underline{d}istributional \underline{a}lignment} (DADA) problem: given a pretrained diffusion model, an \emph{attribute oracle}, and a target attribute \emph{distribution} specified only \emph{at test time}, generate samples whose attribute distribution matches the target.
This problem is distinct from most common diffusion alignment and conditional generation settings \citep{wallace2024diffusion,liu2024alignment,li2024aligning,uehara2025inference}, 
because distributional alignment is a \emph{population-level} objective and thus generally cannot be evaluated by a reward function $r(\mathbf{x})$ defined over a \emph{single sample}.
The most conceptually related topics include fairness-aware generation and sample diversity promotion in diffusion- and flow-based models. 
However, existing approaches \cite{kang2025fairgen,he2024debiasing,choi2024fair,friedrich2025auditing,jiang2025fairgen,li2025fair,morshed2025diverseflow} are often tailored to text-to-image (T2I) diffusion models, require retraining or finetuning for new targets, or lack flexibility when the target attribute distribution changes at test time.
To our knowledge, there are few generic methods that can align pretrained diffusion models to different test-time target attribute distributions in a plug-and-play fashion, \emph{without retraining or finetuning the diffusion models or training extra components}.

In this work, we formulate attribute distributional alignment as an \gls{oc} problem over the reverse diffusion process. 
Concretely, we view sampling from a pretrained diffusion model as rolling out a dynamical system that defines a prior over realistic samples, and augment its learned dynamics with an additive, time-dependent perturbation (see \cref{fig:teaser}).
An \gls{oc}-based algorithm then solves for perturbations that optimize a differentiable distribution-alignment objective.
This \gls{oc} perspective is appealing for several reasons: 
(\romannumeral1) It provides a \emph{principled} mechanism for balancing distributional alignment and data fidelity through an explicit control-effort penalty.
(\romannumeral2) It is inherently \emph{target-flexible} and naturally suited to \emph{inference-time} use, since changing the target attribute distribution only modifies the objective while leaving the pretrained model intact.
(\romannumeral3) By optimizing the reverse trajectory as a whole, it avoids the single-step clean-sample approximation $\hat{\mathbf x}_0$ used by most inference-time guidance methods, which can be unreliable in the high-noise stages of the reverse process.
(\romannumeral4) It replaces hand-tuned, step-wise guidance-strength schedules, which often play a critical role in guidance methods, with control updates derived from \gls{oc} optimality conditions.

We make the following contributions in this paper: 
(\romannumeral1) We formulate the attribute distributional alignment problem for diffusion models as an \gls{oc} problem. 
(\romannumeral2) We propose a practical inference-time, training-free algorithm for pretrained unconditional diffusion to align the \emph{attribute distribution} with flexible target distributions. 
(\romannumeral3) We demonstrate the empirical effectiveness of our method in aligning attribute distributions with test-time targets in image generation.

\begin{figure*}[t]
    \centering
    \includegraphics[width=\textwidth]{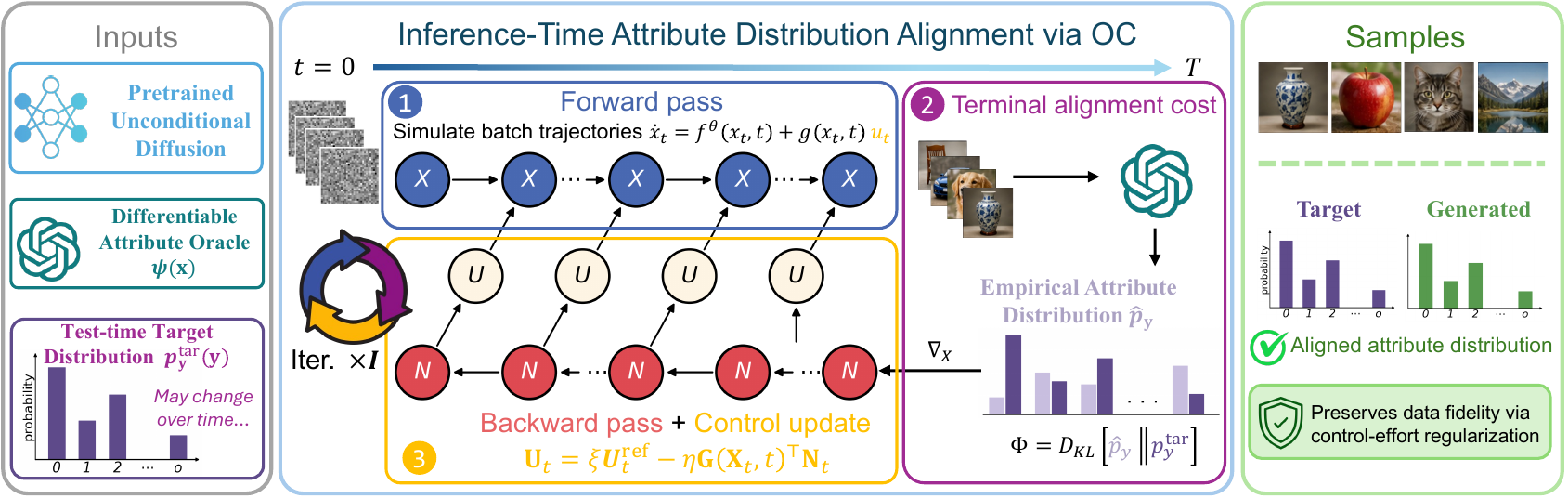}
    \caption{
        An overview of our proposed method for aligning attribute distribution. 
    }
    \label{fig:teaser}
    \vspace{-1em}
\end{figure*}

\section{Related Work}

\textbf{Diffusion Models with Conditional Generation.} 
Diffusion models define the generative process as iterative denoising \citep{ho2020denoising} or, equivalently, as score-based Langevin dynamics \citep{song2019generative,song2021scorebased}. 
While DDIM \citep{song2021denoising} and EDM \citep{karras2022elucidating} improved sampling efficiency and design clarity, controllable generation often relies on guidance mechanisms. 
\citet{dhariwal2021diffusion} introduced classifier guidance to steer pretrained models, while \citet{ho2022classifierfree} proposed classifier-free guidance for joint training with conditioning signals. 
\citet{chung2023diffusion} proposed posterior sampling for inverse problems, upon which many more general training-free guidance techniques were built \cite{ye2024tfg,feng2024ltldog,guo2024gradient,corso2024particle}. 

\textbf{Fairness and Diversity in Diffusion.} 
Fairness generation with diffusion is an instance of the DADA problem. 
\citet{shen2023finetuning} propose a supervised finetuning method using distributional alignment loss and adjusted direct finetuning to mitigate demographic biases in T2I diffusion models. 
\citet{miao2024training} take a reinforcement learning approach for fine-tuning with policy gradient methods and a diversity reward. 
As for test-time methods, 
\citet{friedrich2025auditing} conduct a preliminary model audit to generate a lookup table of biased prompts and attributes, and then employs Semantic Guidance to promote fair attribute classes and suppress biased terms. 
\citet{jiang2025fairgen} train extra attribute-specific adapters and guide diffusion generation in a plug-and-play fashion; 
Fair~Mapping \cite{li2025fair} adopts a linear network to map input embeddings into a debiased representation space. 
FairGen \citep{kang2025fairgen} uses an additionally-trained latent module and an extra memory module to steer generations toward user-specified fair distributions. 
Distribution-focused techniques such as IDA \cite{he2024debiasing} optimize the weights of multi-directional text descriptions, while 
\citet{parihar2024balancing} trains a predictor to map $h$-space features to attribute distributions during the denoising diffusion process. 
Notably, \citet{corso2024particle} propose a training-free guidance-based method for promoting sample diversity in diffusion models. 
However, most of the above methods are specifically designed for T2I conditional models and leverage text conditioning mechanisms that are baked inside the T2I models, thus not generalizable to unconditional models nor to other domains.

\textbf{Optimal Control in Diffusion- and Flow-based Models.}
Early theoretical connections between Stochastic Optimal Control (SOC) and diffusion were established in \citep{berner2024an}. 
\citet{domingo-enrich2024stochastic, domingo-enrich2025adjoint} propose Adjoint Matching and Stochastic Optimal Control Matching for fine-tuning diffusion- and flow-based models by learning optimal control fields via regression objectives, as other approaches follow \cite{han2025stochastic,liu2025value}. 
\citet{park2024stochastic,zhu2025unidb} study diffusion bridges from the SOC perspective. 
In contrast, training-free methods such as OC-Flow \cite{wang2025training} and DTM \cite{pandey2025variational} steer pre-trained diffusion or flow models by differentiating through the generative ODE or SDE to minimize control costs without updating model parameters. 
The OC formulation also extends to solving inverse problems \cite{li2024solving}, adaptive guidance  strength scheduling \cite{azangulov2024adaptive}, stylization \cite{rout2025rbmodulation} and multi-subject generation \cite{bill2025optimal} by modulating trajectories or maximizing likelihoods directly during the sampling process. 
While our method resonates with these works in terms of methodology, none of them are addressing the attribute distribution alignment problem. 

\section{Preliminaries and Formulation: Diffusion Attribute Distribution Alignment}
\label{sec:prelim}

\paragraph{Notations.} 
Let $n, m, o \in \mathbb Z^+$ be positive integers and $t \in \mathbb R_{\ge0}$ be the time variable.  
Let $\x, \z\in \mathbb R^n$, $\u \in \mathbb R^m$, $\y \in \sR^o$ be vector random variables. 
Vector constants are in bold symbol as $\boldsymbol{x} \in \mathbb R^n$. 
$\I_n$ stands for an $n\times n$ identity matrix. 
The Euclidean norm is denoted by $\|\cdot\|$. 
We note the probability density function of $\x$ (or probability mass function for discrete variables) by $p_\x(\x)$, and may, for notational brevity, omit the subscript of random variable when it is clear from the context. 
$\mathcal N(\boldsymbol{\mu}, \boldsymbol{\Sigma})$ stands for a Gaussian distribution with mean $\boldsymbol\mu \in \sR^n$ and covariance $\boldsymbol{\Sigma}\in \sR^{n\times n}$. 
Let $\kldiv{\cdot}{\cdot}$ denote the Kullback-Leibler (KL) divergence.

\subsection{Preliminaries in Diffusion Models}
\label{sub:diffusion_review}

\citet{song2021scorebased} describe a forward diffusion process with \gls{sde}: 
\[
    \mathrm d \x_t = \bm f(\x_t, t)\,\mathrm dt + g(t)\, \mathrm d \mathbf w_t, \quad t \in [0, T ], 
\]
where $\bm f(\x_t, t)$ is the drift coefficient, $g(t)$ is the diffusion coefficient, and $\mathbf w_t$ is the standard Wiener process. 
The reverse diffusion process, in which noises are transformed into data, is given by \citet{anderson1982reverse,song2021denoising} as
\begin{equation}
\label{eq:diffusion_sde}
    \mathrm d \x_t = [\bm f(\x_t, t) - g(t)^2 \nabla_{\x_t} \log p_{t}(\x_t) ]\,\mathrm dt + g(t)\, \mathrm d \bar{\mathbf w}_t,  
\end{equation}
wherein $\nabla_{\x_t} \log p_t(\x_t)$ is the \emph{(Stein) score function} of the marginal distribution over $\x_t$ at time $t$, and $\bar{\mathbf w}_t$ is the standard Wiener process in reverse time. 
\citet{song2021denoising} also identify the \gls{pfode} that yields the same marginal distribution over $\x_t$ at $t$ as \gls{sde} \eqref{eq:diffusion_sde}: 
\begin{equation}
\label{eq:pfode}
    \mathrm d \x_t = \left[ \bm f(\x_t, t) - \frac{1}{2} g(t)^2 \nabla_{\x_t} \log p_t(\x_t) \right] \,\mathrm dt 
.\end{equation}
For complex, high-dimensional distributions, the score function is analytically intractable in general, so it is often approximated by a neural network $\mathbf s^\theta(\x, t)$ parameterized by $\theta$ and learned via score matching \cite{song2019generative}. 
Equivalent to learning the score, the DDPM \cite{ho2020denoising} and DDIM \cite{song2021denoising} formulations, in which the forward diffusion process is constructed as sequentially adding noise to the initial data distribution, learn a noise prediction neural network $\boldsymbol{\epsilon}^\theta(\x_t, t)$ parameterized by $\theta$ for removing the noise during the reverse process to reconstruct the data distribution. 

In this work, we focus on the \emph{\gls{pfode} of the reverse diffusion process} 
and identify two instances; the details are deferred to \apdxref{apdxsub:pfodes}. 
Without loss of generality, we denote both in a unified form
\begin{equation}
\label{eq:diff_ode}
    \dot{\mathbf x}_t = \mathbf f^\theta(\mathbf x_t, t), \quad
    \mathbf x_0 \sim \mathcal N(\mathbf 0, \I_n),
    \quad t \in [0, T].
\end{equation}

\begin{rmk}[Flow matching ODE]
\eqref{eq:diff_ode} is similar to the ODE in flow matching (FM) \cite{lipman2023flow,liu2023flow}. 
We also cover the details in \apdxref{apdxsub:pfodes}. 
\end{rmk}

\paragraph{Steering diffusion at inference.}
A common paradigm of inference-time controllable generation is through classifier guidance (CG) \citep{dhariwal2021diffusion}, wherein the key is to obtain the conditional score $\nabla_{\x_t}\log p_t(\x_t\mid \y)$ via adding a likelihood term\footnote{It is also called the ``adversarial gradient''\citep{luo2022understanding}.} 
$\nabla_{\x_t}\log p(\y\mid\x_t)$ to the pretrained score $\mathbf s^\theta(\x_t,t)$. 
Obtaining this term requires either additionally \emph{training a noise-aware classifier}, or, as methods such as \citep{chung2023diffusion,song2023pseudoinverseguided}, factorizing it over $\x_0$ $\nabla_{\x_t}\log p(\y\mid\x_t) \approx \nabla_{\x_t}\log p(\y\mid\hat{\x}_0(\x_t))$ and approximating $\x_0$ by the posterior mean $\hat{\x}_0(\x_t):=\mathbb E[\x_0\mid\x_t]$ via Tweedie's approach \citep{efron2011tweedie,kim2021noise2score}. 
The latter is training-free and plug-and-play, but the posterior mean is observed to be unreliable at the early stage of the reverse process \cite{li2024solving}.
Therefore, the guidance strength schedule across time steps requires careful and often case-by-case tuning with handcrafted heuristics.

\subsection{Problem Formulation: Diffusion Attribute Distribution Alignment}
Suppose that there is a pretrained diffusion model $\mathbf f^\theta(\mathbf x,t)$ and the reverse diffusion process is as \eqref{eq:diff_ode}. 
Following \citet{domingo-enrich2024stochastic,domingo-enrich2025adjoint}, we treat the diffusion \gls{pfode} as a dynamical system and further introduce a time-dependent, additive perturbation $\mathbf u_t \in \mathbb R^m$ called control into it: 
\begin{equation}
\label{eq:ctrl_ode}
    \dot{\mathbf x}_t
    := \mathbf f^\theta(\mathbf x_t, t) + \mathbf g(\mathbf x_t, t) \mathbf u_t,
\end{equation}
for $t\in [0,T]$, 
with $\mathbf x_0 \sim \mathcal N(\mathbf 0, \I_n)$ and $\mathbf g: \mathbb R^n \times \sR_{\ge0} \to \mathbb R^{n\times m}$ as an actuation field. 
In this work, we instantiate \eqref{eq:ctrl_ode} under the EDM \cite{karras2022elucidating} and DDIM \cite{song2021denoising} formulations; details are deferred to \apdxref{apdxsub:pfodes}.
We further let $p_{\x_T}^{u}$ denote the distribution of $\mathbf x_T$ sampled from process \eqref{eq:ctrl_ode}. 

\vspace{.5em}
\begin{rmk}[Time Direction]
From this point on, we follow the convention in dynamical systems and take the time direction as from $0$ to $T$ to describe \emph{reverse} diffusion. 
\end{rmk}

\begin{rmk}[Connection with diffusion guidance]
The perturbed \gls{pfode} \eqref{eq:ctrl_ode} can conceptually encompass classifier guidance: Let $\mathbf g(\x_t, t) \propto \nabla_{\x_t} \log p(\y \mid \x_t)$ and let the control $\mathbf u_t := w_t$ be a scheduled scalar weight. 
\end{rmk}

Given a continuously differentiable function $\Psi: \mathbb R^n \to \mathbb R^o$ as an \emph{attribute oracle}\footnote{
The attribute oracle is a known function as part of the problem setting, akin to the setup in general \emph{inverse problems} in diffusion generation \citep{song2022solving,kawar2022denoising,chung2023diffusion}. 
}, we note the attribute of interest as 
$
    \mathbf y = \Psi (\mathbf x_T) .
$
The distribution of $\mathbf y$ for samples generated by the perturbed process is the pushforward of $p_{\x_T}^{u} (\mathbf x)$ by $\Psi$:
\[
    p_\y^u(\y) = \int_{\mathbb R^n} \delta\bigl(\y - \Psi(\x)\bigr)\, p_{\x_T}^{u}(\x)\, \mathrm d \x
,\]
where $\delta(\cdot)$ is the Dirac-delta function.

\paragraph{Inference-time Attribute Distribution Alignment.} 
Given an attribute oracle $\Psi(\x)$ and a target attribute \emph{distribution} $p_\y^\text{tar}$ at \emph{test-time}, our goal is to find control $\mathbf u_t$ for $t\in [0,T]$ such that $p_\y^u(\mathbf y)$ aligns with the target \emph{without retraining or fine-tuning} the pretrained generative model $\mathbf f^\theta(\x, t)$, \ie, 
\begin{equation}
\label{eq:objective} 
    \min_{\mathbf u_t, t\in[0,T]} \; \mathbb D[p_\y^u ||  p^\text{tar}_\y] \qquad \text{s.t. }\; \text{Eq. (\ref{eq:ctrl_ode})}, 
\end{equation}
where $\mathbb D[\cdot \, || \, \cdot]$ is a statistical distance or divergence between two 
distributions with the same support.

\section{Optimal Control for Distribution Alignment}
\label{sec:method}

\subsection{Review of Sample-wise Optimal Control}
Let $\x$ denotes the state and $\mathbf u$ the \emph{control input} for a dynamical system. 
Let $L: \mathbb R^{n} \to \sR$ and $\ell : \mathbb R^n \times \mathbb R^m \to \sR$ be continuously differentiable functions. 
Let $\mathbf f: \sR^n \times \sR^m \times \sR_{\ge0} \to \sR^n$ be a continuous vector function with continuous first partial derivatives w.r.t. the first argument. 
An \gls{oc} problem is formulated as follows: 
\begin{subequations}
\label{eq:oc}
\begin{align}
    \min_{\mathbf u_t, t\in[0,T]} &\;   J(\x, \mathbf u) = L(\mathbf x_T) + \int_0^T \ell (\x_t, \mathbf u_t, t)  \, \mathrm dt \\
    \textrm{s.t.}\quad  &  \dot{\mathbf x}_t = \mathbf f (\mathbf x_t, \mathbf u_t, t) \label{eq:oc_dyn} \\
    & \mathbf x_0 = \boldsymbol{x}_\text{init}, \quad \mathbf u_t \in \mathcal U, \quad t\in [0,T] \label{eq:oc_init_adm_ctrl}
,\end{align} 
where $\mathcal U \subseteq \sR^m$ is a nonempty closed set called the admissible control set. 
\end{subequations}
$L$ is the terminal cost function, $\ell$ is the running cost or transient cost, and $J$ is the cost functional. 
The constraint \eqref{eq:oc_dyn} is the equation of motion of a dynamical system. 
See \cite{fleming2012deterministic} for more rigorous definitions. 
The formulation \eqref{eq:oc} is explicitly minimizing a scalar total cost that encodes some desired effects at the terminal time and the control efforts exerted to the system, while simultaneously respecting the system dynamics, initial state conditions, and admissible control constraints.

\subsection{Optimal Control Formulation for DADA}

\paragraph{Finite-sample Batched OC.}
Different from previous work that applies \gls{oc} to diffusion generation to achieve sample-wise objectives, our alignment objective in \eqref{eq:objective} naturally depends on multiple samples. 
As such, we stack the $M$-sample batches of states, controls, and costates (adjoints) into vectors:
$\X := [{\x^{[1]}}^\top, \ldots, {\x^{[M]}}^\top]^\top \in \sR^{Mn}$,
$\mathbf U := [{\u^{[1]}}^\top, \ldots, {\u^{[M]}}^\top]^\top \in \sR^{Mm}$,
and define the concatenated dynamics
$\mathbf F^\theta(\X, \mathbf U, t) := [{\mathbf f^\theta(\x^{[1]}, \u^{[1]}, t)}^\top, \ldots, {\mathbf f^\theta(\x^{[M]}, \u^{[M]}, t)}^\top]^\top \in \sR^{Mn}$
and concatenated actuation 
$    
\mathbf G(\X, t)
    = \mathrm{blkdiag} \left(\mathbf g (\x^{[1]}, t), \ldots, \mathbf g (\x^{[M]}, t) \right)
    \in \sR^{Mn \times Mm}
$. 
We stack the i.i.d.\ initial states $\x^{[i]}_{\rm init} \sim \gauss{\mathbf 0}{\I_n}$ as $\X_{\rm init} \in \sR^{Mn}$.
We also write $\mathcal U^M := \mathcal U \times \cdots \times \mathcal U \subseteq \sR^{Mm}$
and use $\Pi_{\mathcal U^M}$ for the component-wise projection onto $\mathcal U^M$.

With the above notations, we formulate the DADA problem as a finite-sample batched \gls{oc} problem:
\begin{subequations}
\label{eq:oc_dada}
\begin{align}
    \min_{\mathbf U_t\in\mathcal U^M,\ t\in[0,T]}&\,  \Phi \left(\hat p^u_{\y}(\X_T)\right) + \frac{\rho}{2}\int_0^T \|\mathbf  U_t\|^2\, \mathrm dt
    \label{eq:oc_ada_obj}\\
    \text{s.t.} \qquad &
     \dot \X_t = \mathbf F^\theta(\X_t, \mathbf U_t, t), \; \X_0 = \X_{\rm init} \label{eq:oc_ada_dyn} \\
     &\mathbf U_t \in \mathcal U^M \label{eq:oc_ada_ctrl}
\end{align}
\end{subequations}
where
$\hat p^u_{\y}(\X_T)$ denotes attribute distribution estimated by the batched samples $\X_T$, and we
adopt the reverse KL divergence against the target as the terminal cost:
\begin{equation}
\label{eq:oc_ada_term_cost}
\Phi \big(\hat p^u_{\y}(\X_T)\big) := \kldiv{\hat p^u_{\y}(\X_T)}{p_{\y}^{\rm tar}}.
\end{equation}
While the dynamics \eqref{eq:oc_ada_dyn} are separable across samples, \eqref{eq:oc_ada_term_cost} couples the batch through $\hat p^u_{\y}(\X_T)$.

\subsection{Optimal Controller for DADA OC}

Problem~\eqref{eq:oc_dada} is an \gls{oc} problem in $\sR^{Mn}$.
Define the Hamiltonian associated with \eqref{eq:oc_dada} as
\begin{align*}
\widetilde H(\X_t, \mathbf N_t, \mathbf U_t, t):=\frac{\rho}{2}\|\mathbf U_t\|^2+ \mathbf N_t^\top \mathbf F^\theta(\X_t, \mathbf U_t, t)
\end{align*}
where $\mathbf N_t = [{\bm\nu_t^{[1]}}^\top, \ldots, {\bm\nu_t^{[M]}}^\top]^\top \in \sR^{Mn}$ stacks the adjoint (costate) vectors. 
By Pontryagin's Maximum Principle (PMP) \cite{fleming2012deterministic,pmp1972levine},
the necessary conditions for an optimal trajectory
$(\tilde\X_t^\ast, \tilde{\mathbf N}_t^\ast, \tilde{\mathbf U}_t^\ast)$
of \eqref{eq:oc_dada} are:
\begin{subequations}
\label{eq:pmp}
\begin{align}
    \dot{\tilde\X}_t^\ast
    &= \mathbf F^\theta(\tilde\X_t^\ast, \tilde{\mathbf U}_t^\ast, t),
    \label{eq:pmp_x}\\
    \dot{\tilde{\mathbf N}}_t^\ast
    &= -\,\nabla_{\X}\mathbf F^\theta(\tilde\X_t^\ast, \tilde{\mathbf U}_t^\ast, t)^\top \tilde{\mathbf N}_t^\ast,
    \label{eq:pmp_adj}\\
    \tilde{\mathbf N}_T^\ast
    &= \nabla_{\X}\Phi \left(\hat p^u_{\y}(\tilde\X_T^\ast)\right),
    \label{eq:pmp_term}\\
    \tilde{\mathbf U}_t^\ast
    &\in \argmin_{\mathbf U\in\mathcal U^M}\ \widetilde H(\tilde\X_t^\ast, \tilde{\mathbf N}_t^\ast, \mathbf U, t).
    \label{eq:pmp_u_orig}
\end{align}
\end{subequations}
Since $\mathbf F^\theta$ concatenates per-sample dynamics \eqref{eq:oc_dyn}, its Jacobian is block-diagonal:
\begin{equation}
\label{eq:jac_blkdiag}
    \nabla_\X \mathbf F^\theta(\X, \mathbf U, t)
    = \mathrm{blkdiag} \big(\mathbf J_t^{[1]}, \ldots, \mathbf J_t^{[M]}\big)
    \in \sR^{Mn \times Mn},
\end{equation}
where $\mathbf J_t^{[i]} := \nabla_{\x} \mathbf f^\theta(\x^{[i]}_t, \u^{[i]}_t, t) \in \sR^{n\times n}$
is the standard Jacobian of the per-sample vector field w.r.t.\ the state.
Consequently, the adjoint dynamics \eqref{eq:pmp_adj} decouple across samples:
$\dot{\bm\nu}_t^{\ast,[i]} = -\mathbf (J_t^{[i]})^{\top} \bm\nu_t^{\ast,[i]}$ for each $i$.
The only cross-sample coupling enters through the terminal condition \eqref{eq:pmp_term},
which depends on all $M$ terminal states via the batch cost~$\Phi$.

However, jointly solving all the conditions in \eqref{eq:pmp} is essentially a two-point boundary value problem; for general nonlinear, nonconvex dynamics, it is challenging to solve.
Rather than directly solving them in joint,
we follow \citet{li2018maximum,wang2025training} and employ the Extended Method of Successive Approximations (E-MSA) to solve a \emph{proximalized} \gls{oc} subproblem in an iterative fashion for stability.
Specifically, given a reference control trajectory $\bm U_t^{\rm ref}$, we consider a subproblem
\begin{equation}
\label{eq:oc_dada_proximal}
\min_{\mathbf U_t\in \mathcal U^M} \, \Phi \big(\hat p^u_{\y}(\X_T)\big) + \frac{1}{2}\int_0^T \left( \rho \|\mathbf U_t\|^2  + \gamma \|\mathbf U_t - \bm U^\textrm{ref}_t\|^2 \right) \mathrm dt, 
\quad \text{s.t.\; \eqref{eq:oc_ada_dyn} and (\ref{eq:oc_ada_ctrl})}
\end{equation}
where $\gamma$ is a hyperparameter. 
Applying PMP to this \emph{proximalized} subproblem yields the \emph{necessary} conditions akin to \eqref{eq:pmp} but in terms of the extended Hamiltonian
\begin{align*}
H(\X_t, \mathbf N_t, \mathbf U_t, t):=& \widetilde H(\X_t, \mathbf N_t, \mathbf U_t, t)+\frac{\gamma}{2} \|\mathbf U_t - \bm U_t^{\rm ref}\|^2
\end{align*}
and the optimal trajectories $({\mathbf X}_t^{\ast}, {\mathbf N}_t^{\ast}, {\mathbf U}_t^{\ast})$ that solves the \emph{proximalized} problem.
Accordingly,
\begin{align}
    \mathbf U_t^\ast
    &\in \argmin_{\mathbf U\in\mathcal U^M}\, H(\X_t^\ast, \mathbf N_t^\ast, \mathbf U, t) \label{eq:pmp_u}
.\end{align}

\paragraph{Closed-form Control Update.}
Since the running cost is quadratic in $\mathbf U$ and the dynamics are control-affine, \eqref{eq:pmp_u} admits a \emph{closed-form} minimizer:
\begin{align}
\label{eq:emsa_u_update}
\mathbf U_t^\ast
&= \Pi_{\mathcal U^M} \left(\xi \bm U_t^{\rm ref}-\eta\, \mathbf G\left(\X_t^\ast, t \right) \mathbf  N_t^\ast\right), 
\end{align}
where $\xi:=\frac{\gamma}{\rho+\gamma}$ and $\eta:=\frac{1}{\rho+\gamma}$. 

\begin{proposition}
\label{prop:optimal_u}
For any fixed $t\in[0,T]$ and fixed $(\bm x_t^{[i]},\bm\nu_t^{[i]})$, with $\rho+\gamma>0$,
(1) the sample-wise optimal control $\mathbf u_t^{\ast,[i]}$ for \eqref{eq:pmp_u} is given by
\begin{align*}
    \mathbf u_t^{\ast,[i]} \in \Pi_{\mathcal U}\left(\bar{\bm u}_t^{[i]}\right)
    := \argmin_{\mathbf u\in\mathcal U} \left\|\mathbf u-\bar{\bm u}_t^{[i]} \right\|^2,
    \; \text{where} \;
    \bar{\bm u}_t^{[i]}
    := \frac{1}{\rho+\gamma}\left(
        \gamma\,\bm u_t^{\rm ref,[i]} - \mathbf g\left(\bm x_t^{[i]},t \right)^\top \bm\nu_t^{[i]}
    \right),
\end{align*}
and (2) $\Pi_{\mathcal U}\left(\bar{\bm u}_t^{[i]}\right)$ is nonempty.
Further, $\mathbf u_t^{\ast,[i]}$ is unique if $\mathcal U$ is convex.
\end{proposition}

The proof of \textbf{Proposition~\ref{prop:optimal_u}} is in \apdxref{apdx:derivations}.
Proposition~\ref{prop:optimal_u} indicates that the optimal control at each time is a function of the state and adjoint at that time.
The resulting \emph{discrete-time} algorithm (\cref{alg:dada_emsa_1} in \apdxref{apdxsub:algo}) performs the following three steps iteratively:
(i) Forward pass: Simulate \eqref{eq:oc_ada_dyn} in forward time.  
(ii) Backward pass: Evaluate the distribution-alignment cost, take the gradients, and simulate the adjoint dynamics in backward time. 
(iii) Control update with \eqref{eq:emsa_u_update}. 
An illustration is in \cref{fig:teaser}. 
In practice, we adopt Euler discretizations for the forward and backward passes and leverage the vector-Jacobian product (VJP) at backward pass without materializing the gigantic dynamics Jacobian.
We analyze the computational complexity in \apdxref{apdxsub:complexity}.

\paragraph{Differentiable Approximations for Terminal Cost.} 
For discrete attributes in practice, the attribute oracle is often a continuously differentiable neural network that maps an input $\mathbf x$ to an attribute-logit vector. 
To maintain differentiability, we \emph{estimate} the empirical distribution $\hat p^u_{\mathbf y}$ over a batch of $M$ generated samples by averaging their softmax probabilities across the batch:
\begin{equation}
    \hat p^u_{\mathbf y} = \frac{1}{M} \sum_{i=1}^M \textrm{softmax}\left( \Psi\left( \mathbf{x}_T^{[i]} \right) \right).
\end{equation}
The terminal cost is then computed as the element-wise KL divergence:
\begin{equation}
    \Phi \left( \hat p^u_{\mathbf y} \right) 
    = \kldiv{\hat p^u_\mathbf{y}}{p^\text{tar}_\mathbf{y}} 
    = \sum_{j=1}^{o} \hat p^{u, (j)}_{\mathbf y}  \log \frac{\hat p^{u, (j)}_{\mathbf y} }{ p^{\text{tar}, (j)}_{\mathbf y} } 
,\end{equation}
where $o$ is the total number of attribute classes, and the superscript $(j)$ indexes the probability mass associated with the $j$-th class.

\section{Experiments}
\label{sec:experiments}

\vspace{-1em}
We address the following research questions through experiments in image generation:
\begin{enumerate}[label={\textbf{RQ~\arabic*}:}, leftmargin=*, topsep=-1pt, itemsep=-1pt]
    \item How effective is our method at aligning attribute distributions toward test-time targets?
    \item Can our method better preserve sample quality while achieving distributional alignment?
\end{enumerate}
We evaluate our method in two settings: 
(\romannumeral1) CIFAR-100 tests whether the method can match flexible target distributions over semantic labels with different support sizes; 
(\romannumeral2) Human face generation tests fairness-motivated attribute alignment over age, gender, race, and their joint distributions, and further examines whether the same formulation and our method applies to flow matching in latent space.

\vspace{-0.5em}
\subsection{CIFAR-100 with Hierarchical Semantic Classes}

\begin{wrapfigure}{R}{0.45\textwidth}
    \vspace{-1em}
    \centering
    \begin{subfigure}[t]{0.32\linewidth}
        \centering
        \includegraphics[width=\linewidth]{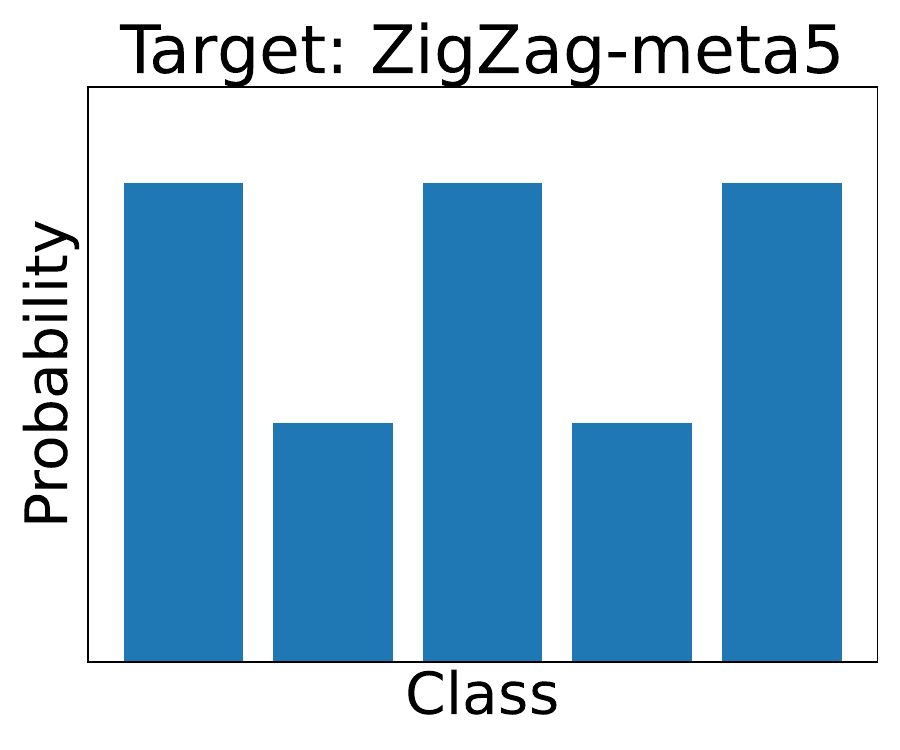}\\
        \includegraphics[width=\linewidth]{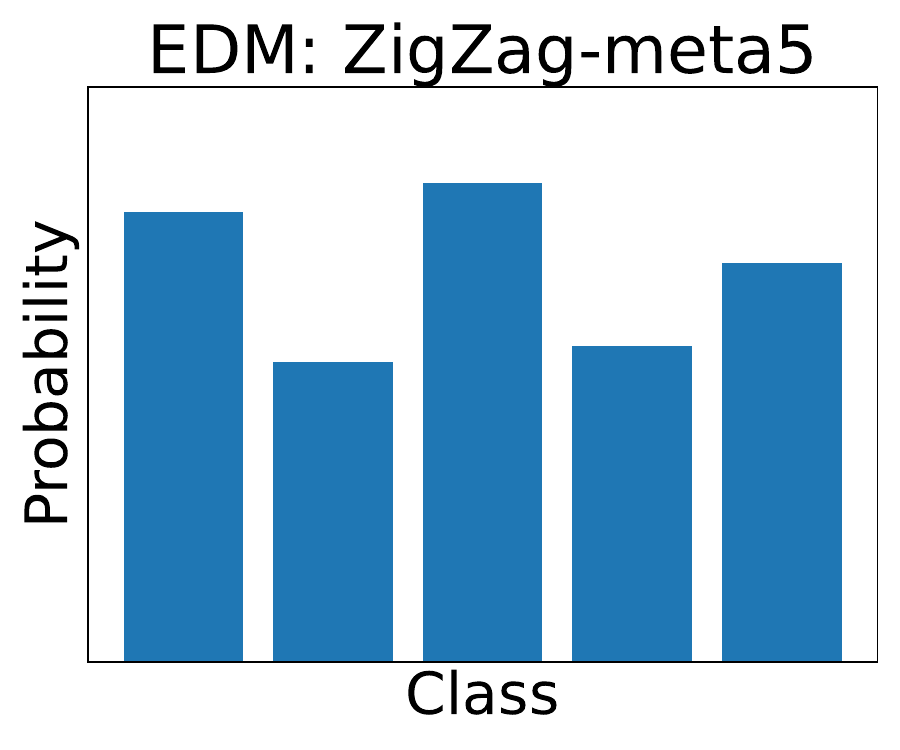}\\
        \includegraphics[width=\linewidth]{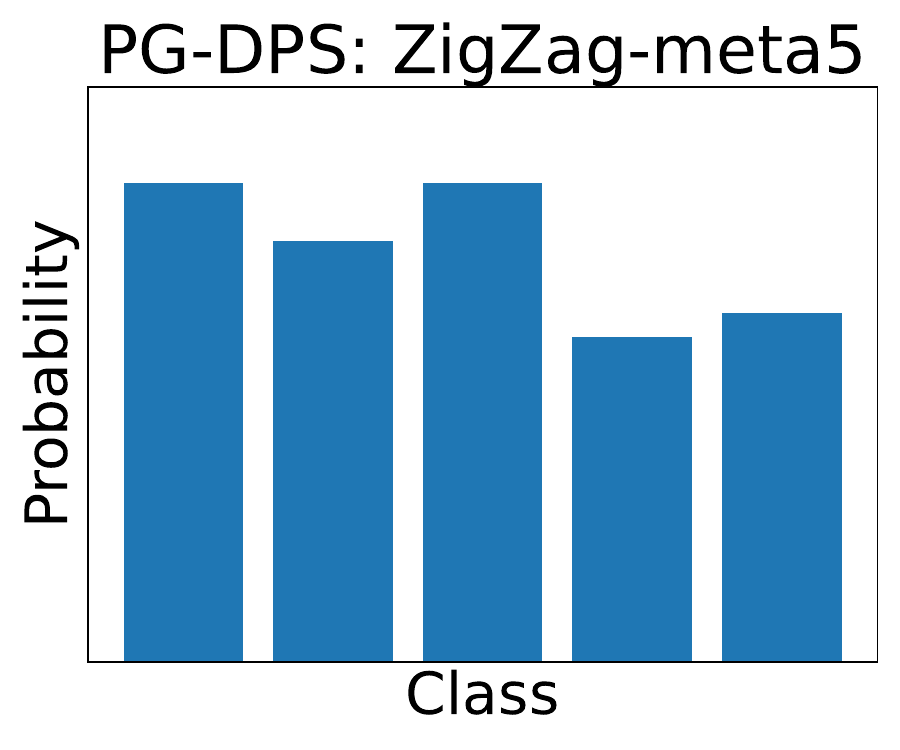}\\
        \includegraphics[width=\linewidth]{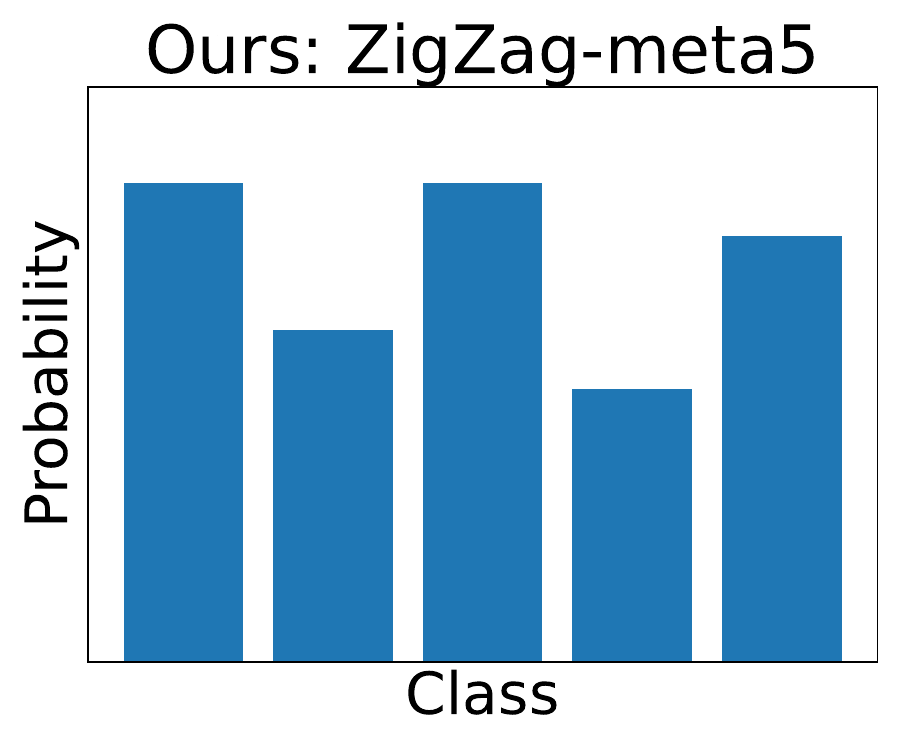}
        \caption{\scriptsize $\mathsf{ZigZag}$-\texttt{meta5}}
    \end{subfigure}
    \hfill
    \begin{subfigure}[t]{0.32\linewidth}
        \centering
        \includegraphics[width=\linewidth]{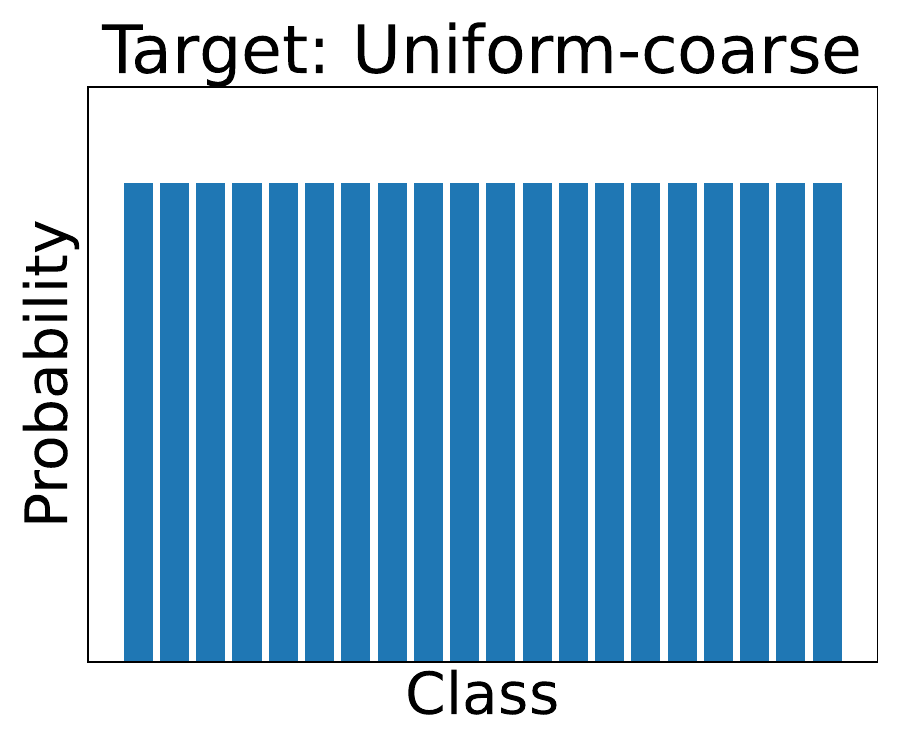}\\
        \includegraphics[width=\linewidth]{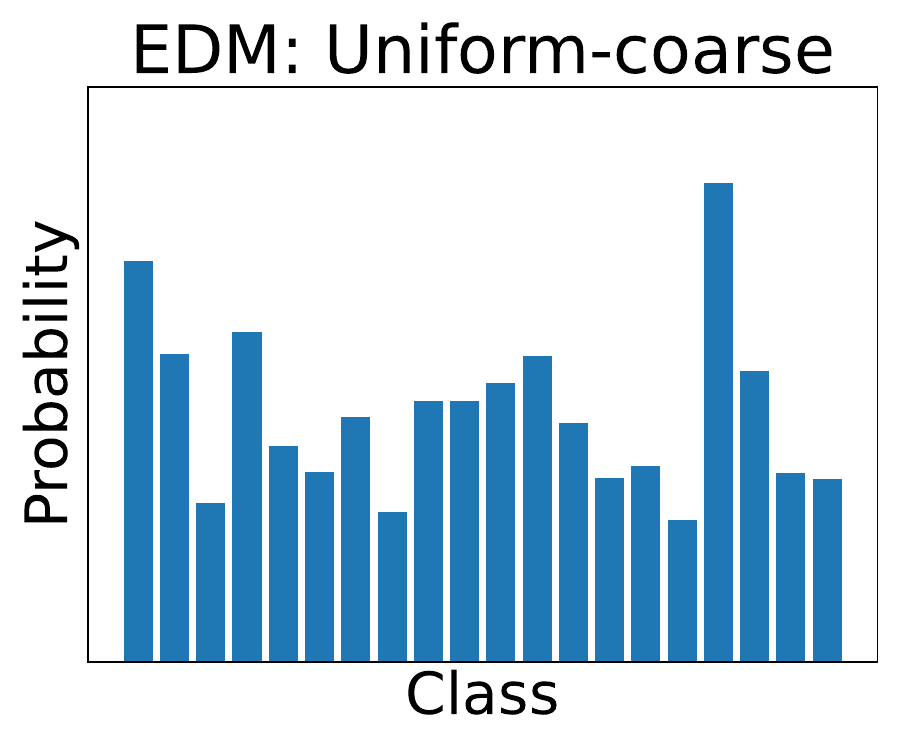}\\
        \includegraphics[width=\linewidth]{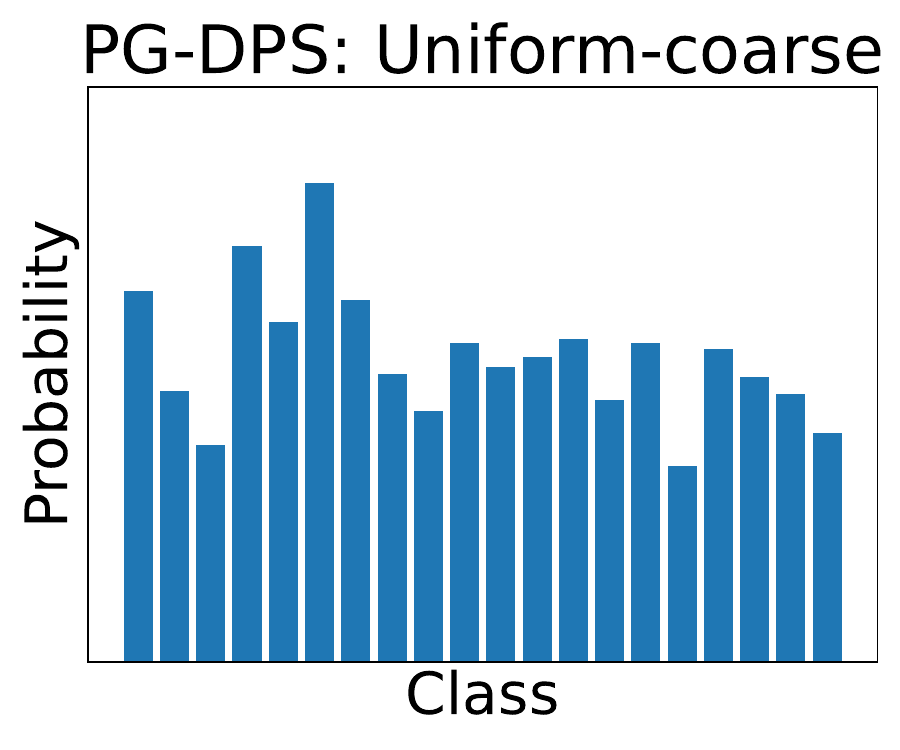}\\
        \includegraphics[width=\linewidth]{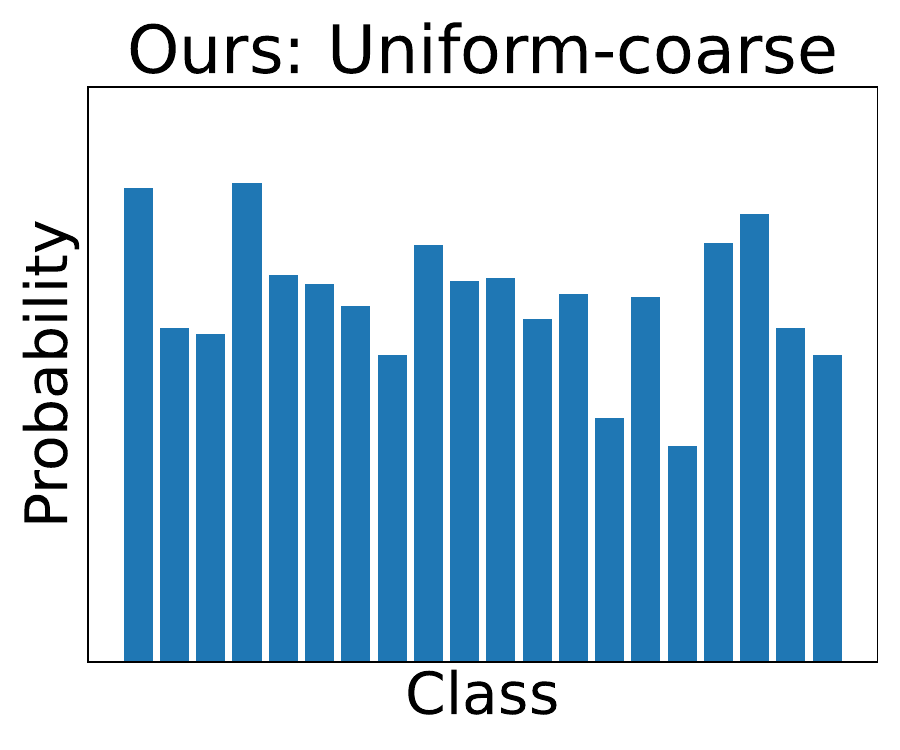}
        \caption{\tiny$\mathsf{Uniform}$-\texttt{coarse}}
    \end{subfigure}
    \hfill
    \begin{subfigure}[t]{0.32\linewidth}
        \centering
        \includegraphics[width=\linewidth]{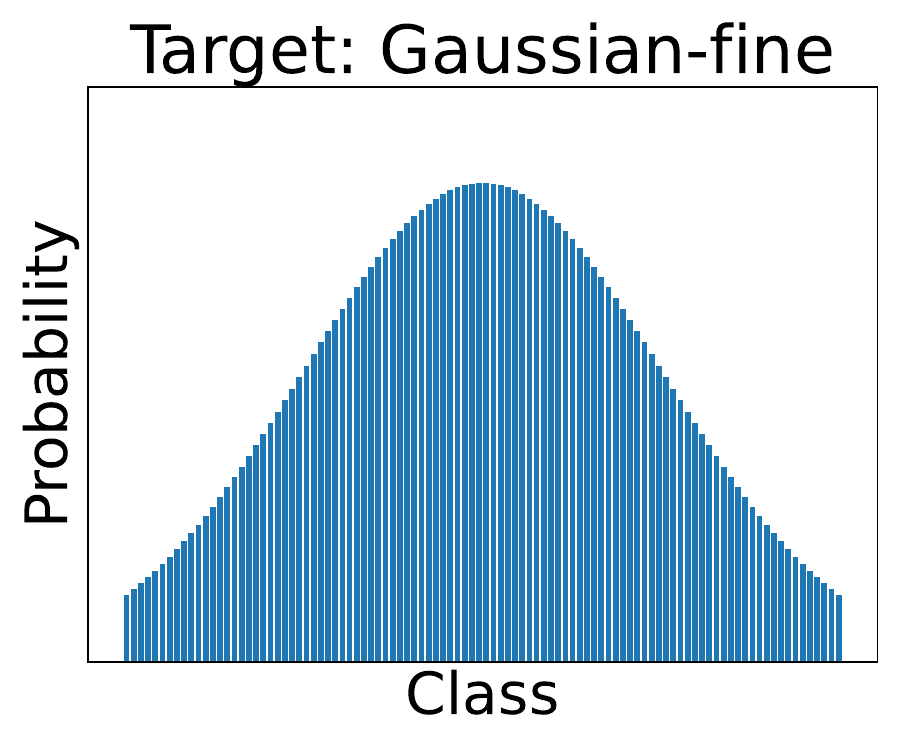}\\
        \includegraphics[width=\linewidth]{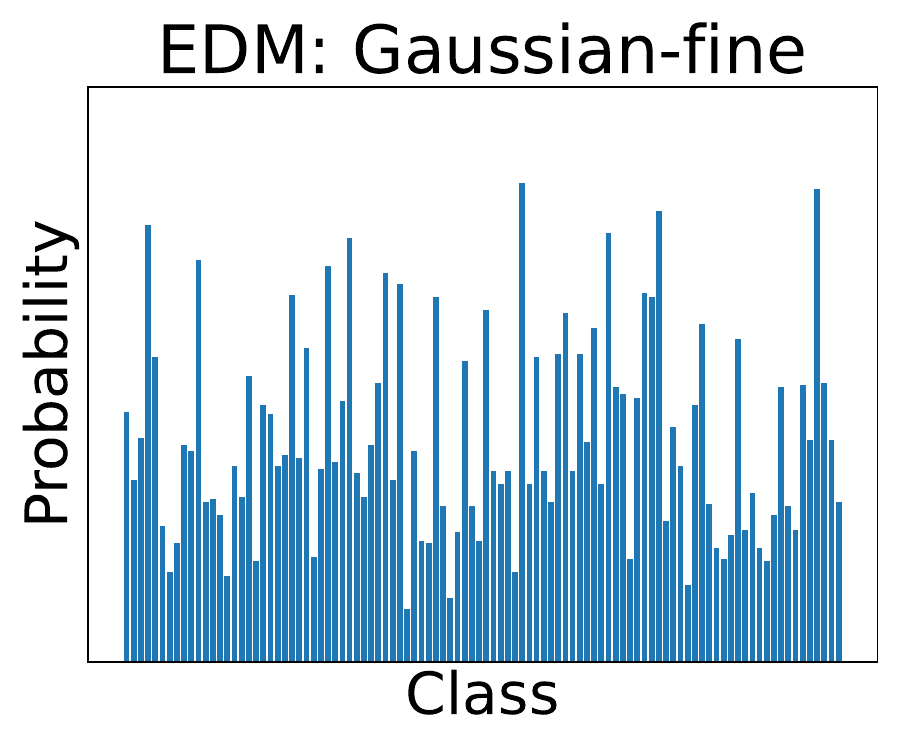}\\
        \includegraphics[width=\linewidth]{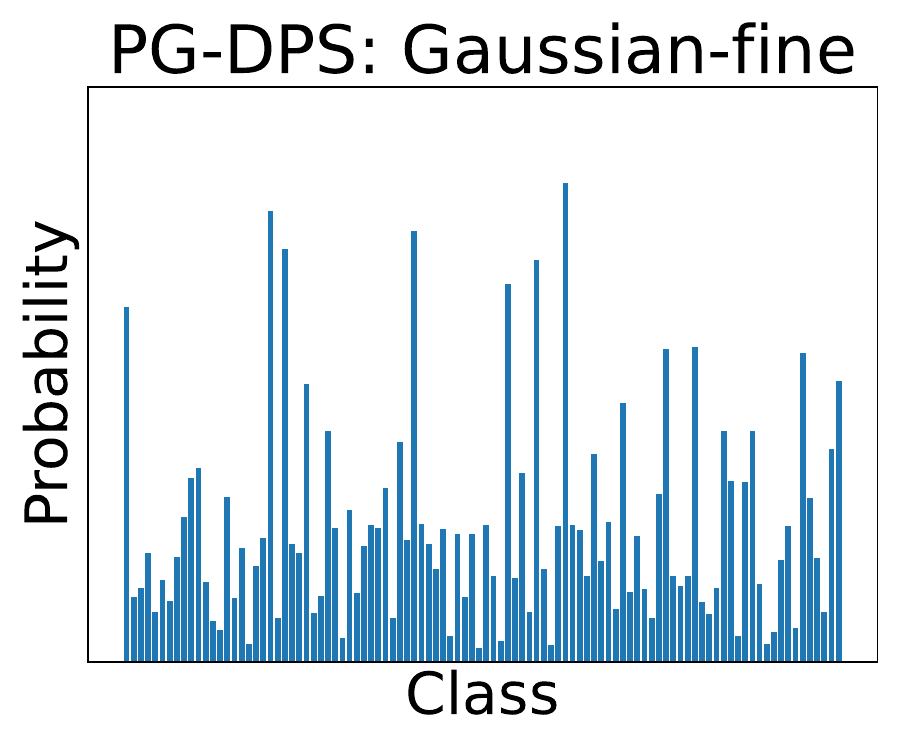}\\
        \includegraphics[width=\linewidth]{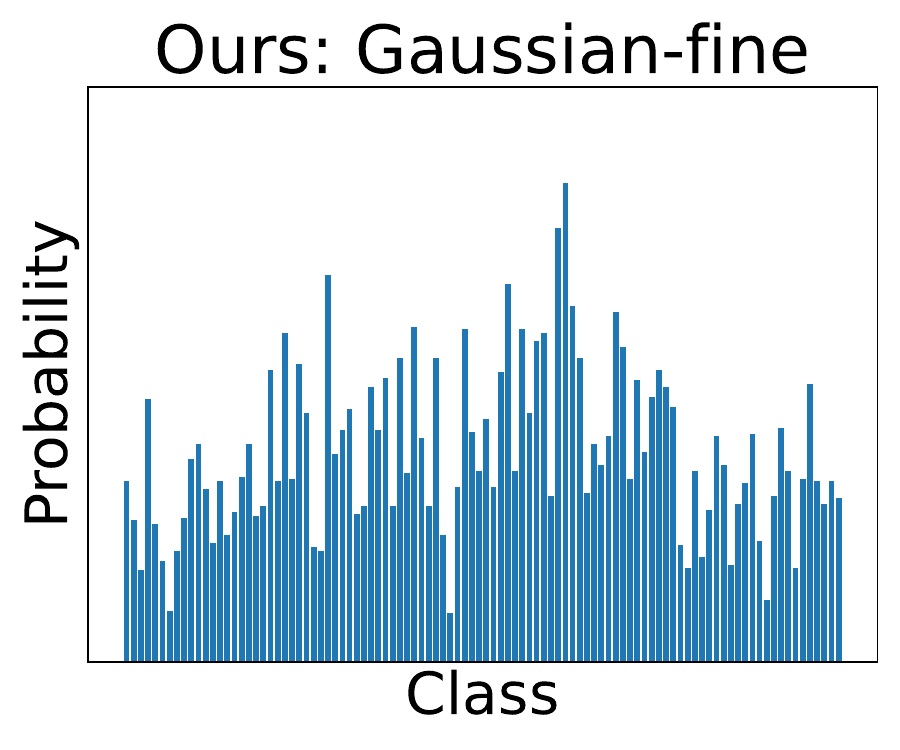}
        \caption{\tiny$\mathsf{Gaussian}$-\texttt{fine}}
    \end{subfigure}
    
    \caption{
        \textbf{Top row}: Test-time \emph{target} attribute distributions (CIFAR-100). 
        \textbf{2nd/3rd rows:} Generated distributions of baselines. 
        \textbf{Bottom row:} Generated distributions of our method.
    }
    \label{fig:cifar_targets}
    \vspace{-3em}
\end{wrapfigure}

\vspace{-0.5em}
As a proof of concept, we first test on generating low-resolution images across semantic classes. 
We treat the semantic class of an image as an attribute and consider flexible test-time target distributions with different support sizes.

\textbf{Setup}: 
We adopt an unconditional EDM model \cite{karras2022elucidating} trained on the CIFAR-100 dataset \cite{krizhevsky2009learning} as the base diffusion model and 
the attribute oracle $\Psi$ is instantiated as ResNet \cite{he2016deep} image classifiers. 
We test our method on three different levels of class labels with a coarse-to-fine hierarchy: \texttt{meta5}, \texttt{coarse}, and \texttt{fine}, with $5$, $20$, and $100$ classes, respectively. 
We choose three different attribute distributions as test-time targets for each level: $\mathsf{Uniform}$, $\mathsf{ZigZag}$, and $\mathsf{Gaussian}$. 
\cref{fig:cifar_targets} (top row) shows instances of the target distributions.

\begin{table*}[t]
    \centering
    \small
    \setlength{\tabcolsep}{1.0mm}
    \caption{Quantitative evaluation metrics (all lower the better) on \texttt{CIFAR-100}. Best in \textbf{bold}, second-best \underline{underlined}.}
    \label{tab:cifar}
    \begin{tabular}{l|cccc|cccc|cccc}
    \toprule
    & \multicolumn{4}{c|}{\textbf{\texttt{meta5}}}& \multicolumn{4}{c|}{\textbf{\texttt{coarse}}}& \multicolumn{4}{c}{\textbf{\texttt{fine}}}\\
    \textsc{Method} & TV$\downarrow$ & JS$\downarrow$ & $\chi^2$$\downarrow$ & FID$\downarrow$ & TV$\downarrow$ & JS$\downarrow$ & $\chi^2$$\downarrow$ & FID$\downarrow$ & TV$\downarrow$ & JS$\downarrow$ & $\chi^2$$\downarrow$ & FID$\downarrow$ \\
    \hline
    \multicolumn{13}{c}{$\mathsf{Gaussian}$} \\
    \hline
    \textsc{EDM} & 0.252 & 0.196 & 0.0749 & \underline{17.4} & 0.274 & 0.227 & 0.0988 & 17.4 & \underline{0.248} & \underline{0.231} & \underline{0.101} & \underline{16.1} \\
    \textsc{PG-DPS} & \underline{0.154} & \underline{0.133} & \underline{0.0350} & 21.3 & \underline{0.195} & \underline{0.165} & \underline{0.0528} & \underline{16} & 0.348 & 0.299 & 0.165 & 33.1 \\
    \rowcolor{green!15}\textsc{Ours} & \textbf{0.136} & \textbf{0.117} & \textbf{0.0272} & \textbf{16.5} & \textbf{0.176} & \textbf{0.146} & \textbf{0.0421} & \textbf{15.7} & \textbf{0.184} & \textbf{0.173} & \textbf{0.0573} & \textbf{13.7} \\
    \hline
    \multicolumn{13}{c}{$\mathsf{Zigzag}$} \\
    \hline
    \textsc{EDM} & \underline{0.067} & \underline{0.0574} & \underline{0.00658} & \underline{15.3} & 0.185 & 0.148 & 0.0433 & 16.1 & \underline{0.24} & \underline{0.205} & \underline{0.0812} & \underline{16.1} \\
    \textsc{PG-DPS} & 0.113 & 0.1 & 0.0199 & 22 & \underline{0.133} & \underline{0.116} & \underline{0.0267} & \underline{15.7} & 0.341 & 0.288 & 0.156 & 33.8 \\
    \rowcolor{green!15}\textsc{Ours} & \textbf{0.0544} & \textbf{0.0495} & \textbf{0.00489} & \textbf{14} & \textbf{0.103} & \textbf{0.0927} & \textbf{0.0171} & \textbf{14.8} & \textbf{0.171} & \textbf{0.15} & \textbf{0.0442} & \textbf{12.6} \\
    \hline
    \multicolumn{13}{c}{$\mathsf{Uniform}$} \\
    \hline
    \textsc{EDM} & 0.083 & 0.0651 & 0.00845 & \underline{15.5} & 0.132 & 0.114 & 0.0255 & 15.5 & \underline{0.186} & \underline{0.159} & \underline{0.0495} & \underline{15.5} \\
    \textsc{PG-DPS} & \underline{0.0692} & \underline{0.0529} & \underline{0.00559} & 22.4 & \underline{0.0805} & \underline{0.0742} & \underline{0.0109} & \underline{15.2} & 0.287 & 0.255 & 0.123 & 30.7 \\
    \rowcolor{green!15}\textsc{Ours} & \textbf{0.0281} & \textbf{0.0292} & \textbf{0.00171} & \textbf{13.1} & \textbf{0.069} & \textbf{0.0655} & \textbf{0.00853} & \textbf{12.6} & \textbf{0.141} & \textbf{0.12} & \textbf{0.0284} & \textbf{12.6} \\
    
    \bottomrule
    \end{tabular}
\vspace{-1em}
\end{table*}

\textbf{Baselines}: 
We compare our method with vanilla EDM, which corresponds to i.i.d. sampling from the learned distribution, and a guidance-based method, Particle Guidance (PG) \citep{corso2024particle}. 
For PG, the guidance potential is the same terminal cost in our method but applied on the estimated terminal sample $\hat\x_T$ via Tweedie's formula \citep{efron2011tweedie}, akin to DPS \cite{chung2023diffusion}. 
See all implementation details in \apdxref{apdx:cifar}.

\textbf{Evaluation Metrics}: 
We use several statistical distances/divergences to quantify semantic-class distribution alignment: Total Variation (TV), Jensen--Shannon divergence (JS), and the $\chi^2$ distance\footnote{
    We adopt the symmetric $\chi^2$ distance \cite{markatou2018statistical}: $d_{\chi^2}(P, Q) = \frac{1}{2} \sum_{i=1}^{K} (P_i-Q_i)^2 / (P_i+Q_i)$ for two discrete distributions $P$ and $Q$ with the same finite support. 
}. 
We use the Fr\'echet Inception Distance (FID) to evaluate image quality.\footnote{
The FID references were computed by sampling from the training dataset according to each target attribute distribution. 
}

\textbf{Results:} 
\cref{fig:cifar_targets} (except the top row) shows the attribute distributions of samples generated by different methods, with the targets in the first row. 
Table~\ref{tab:cifar} presents comprehensive quantitative comparisons across methods and target attribute distributions with different support sizes. 
Across all tested settings, our method achieves the best performance in terms of both attribute-distribution alignment and sample quality. 
Surprisingly, \textsc{PG} performs even worse than vanilla EDM in some cases (\eg, $\mathsf{ZigZag}$-\texttt{meta5} and $\mathsf{Gaussian}$-\texttt{fine}). 
This suggests that naively applying guidance-style approaches with a batch-wise distributional loss may be insufficient for attribute distribution alignment. 
Qualitative results and ablation study are in \apdxref{apdxsub:additional_cifar}.

\subsection{Human Face Generation}

\begin{figure*}[t]
  \begin{center}
    \centerline{
        \includegraphics[width=0.32\textwidth]{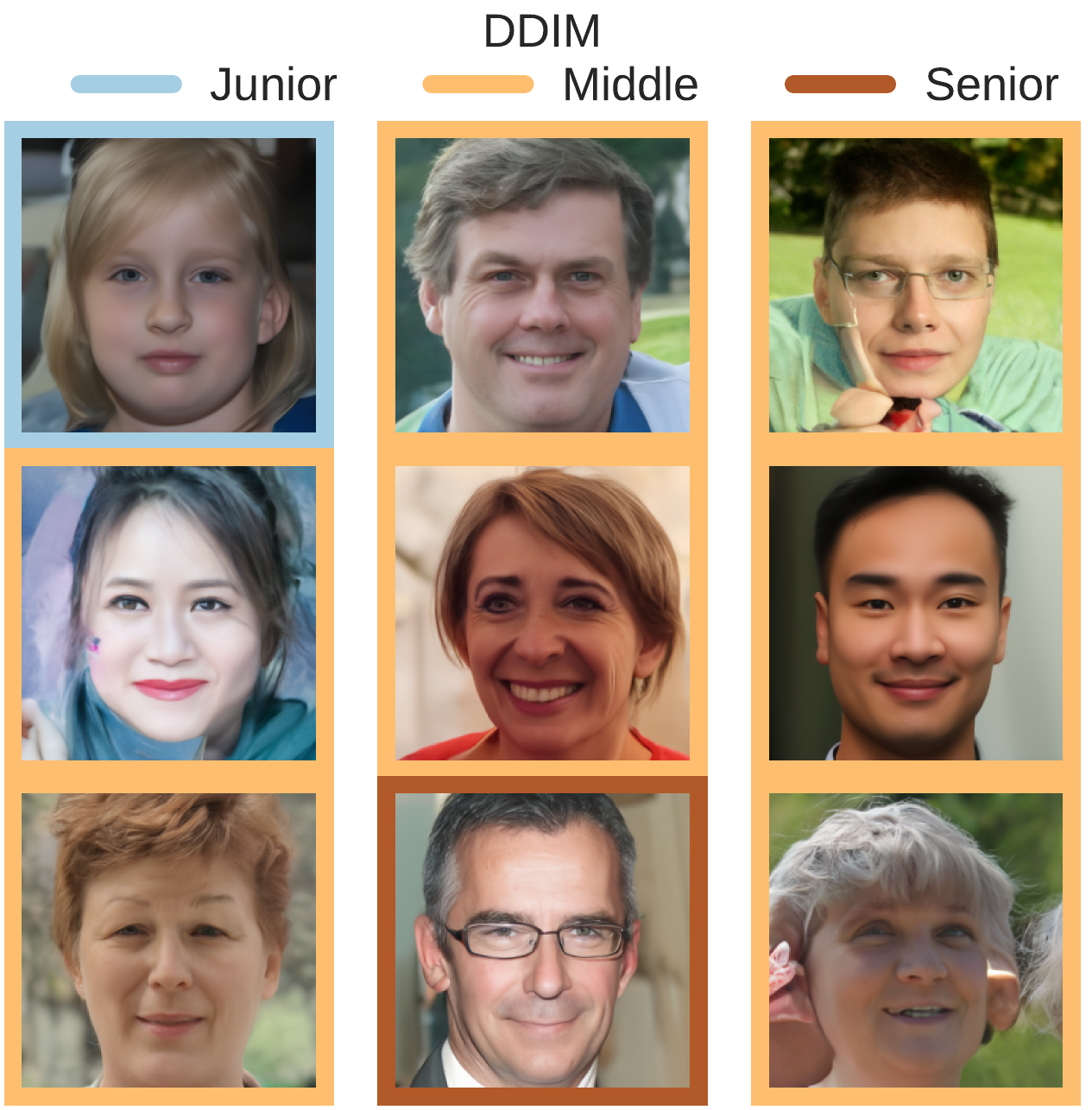}
        \hfill
        \includegraphics[width=0.32\textwidth]{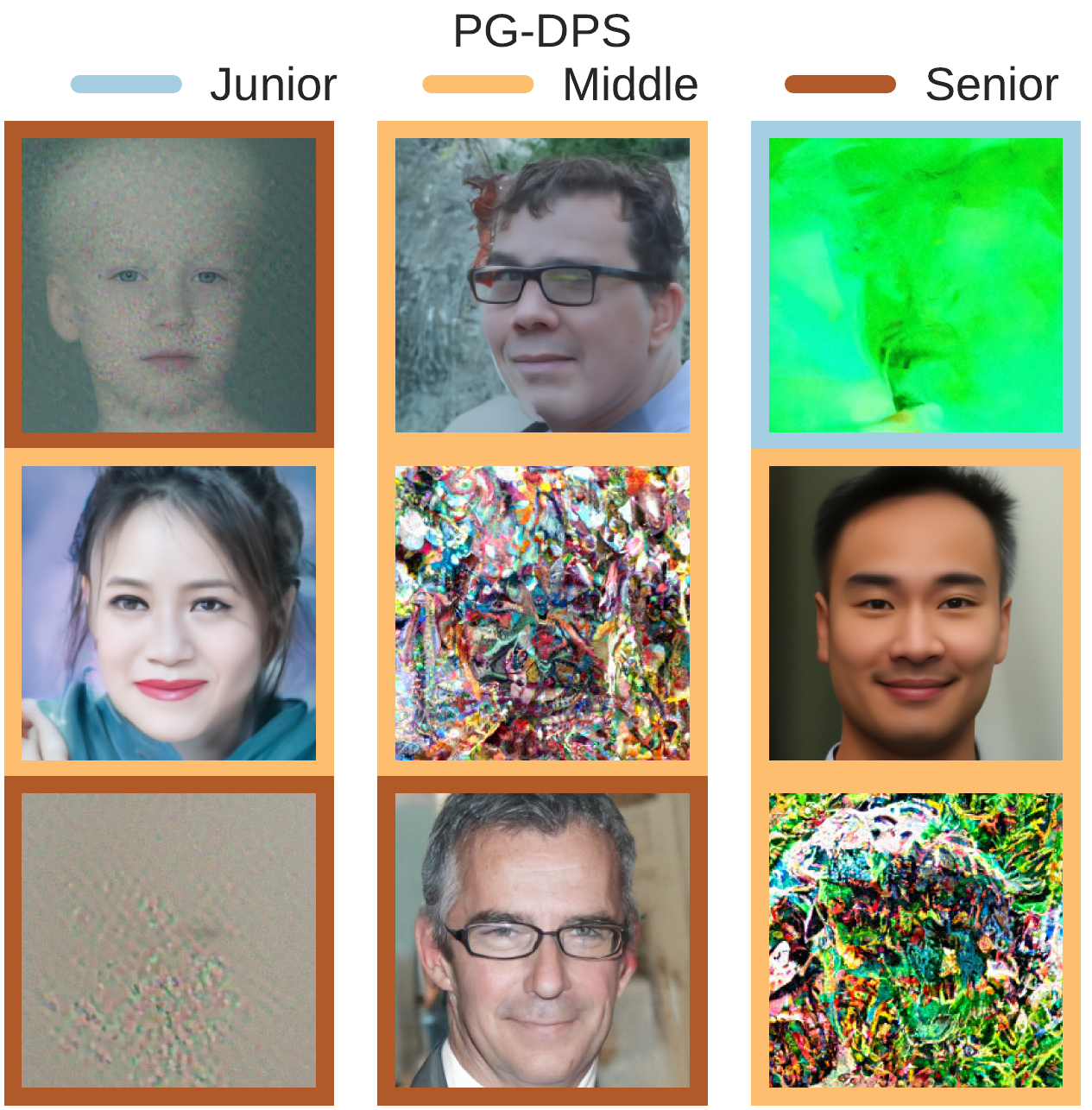}
        \hfill
        \includegraphics[width=0.32\textwidth]{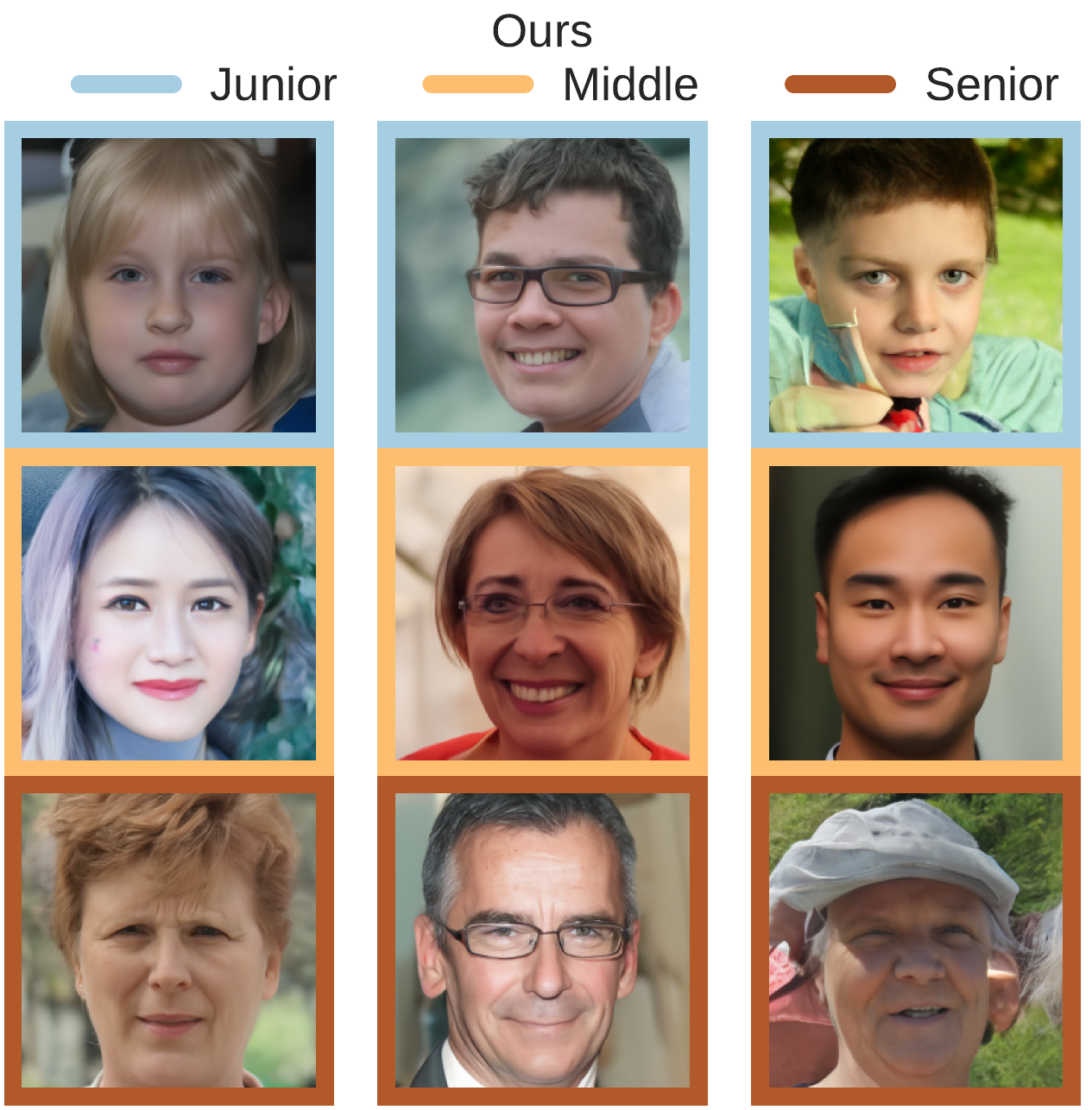}
    }
    \caption{
        Qualitative samples \emph{from a single batch} of our method (right), compared to vanilla DDIM (left) and PG (mid) for generating human faces with a \emph{fair} target distribution over age groups ($\mathsf{Uniform}$-\texttt{age}), where the ratio of faces in three age groups should be $1{:}1{:}1$. 
        The pretrained diffusion model learned a highly imbalanced age distribution from the FFHQ dataset, where most faces are classified as \texttt{Middle}.
        Our approach aligns the generated attribute distribution with the target by minimally editing facial details such as wrinkles and face shapes. 
        See more samples in \apdxref{apdxsub:additional_face}. 
    }
    \label{fig:face_samples}
  \end{center}
  \vspace{-30pt}
\end{figure*}

We aim to achieve human-face image generation with \emph{fairness and controllable ratios} across ages, genders, and races, mitigating potential biases introduced by pretrained diffusion models \emph{without retraining or fine-tuning them}. 
Beyond the standard diffusion paradigm, we also empirically demonstrate the applicability of our method under flow matching in the latent space. 

\textbf{Setup:} 
We use a DDIM \citep{song2021denoising} pretrained by \citet{choi2022perception} and a pretrained Latent Flow Matching (LFM) \citep{dao2023flow} model;  
both models were trained on the FFHQ-256 dataset \cite{karras2019style}. 
We consider three attributes in this task: gender, race, and age. 
For gender, we consider \{\texttt{Female}, \texttt{Male}\}; 
for race, we adopt the 4-way classification by \citet{karkkainen2021fairface}: \{\texttt{Asian, Black, Indian, WMELH}\}\footnote{
    \texttt{WMELH} stands for the merge of White, Middle Eastern, and Latino Hispanic. 
}; 
for age, we partition it into three groups: \{\texttt{Junior, Middle, Senior}\}\footnote{
    \texttt{Junior}: 0--19 years old; \texttt{Middle}: 20--50; \texttt{Senior}: 50--120. 
}. 
We test with both fairness targets and customized skewed targets for each single attribute and \emph{joint} attributes.
The attribute oracle is also instantiated as an image classifier. 
We implement our method under both diffusion and latent flow matching settings. 
See implementation details in \apdxref{apdxsub:face}. 

\textbf{Aligning Joint Attribute Distribution:} 
For \emph{joint} attributes, the alignment target is the factorized joint distribution
$
    p^\text{tar}_\y (\mathbf y) \propto \prod_{i=1}^{N} p^\text{tar}_{\y_i} (\y_i),
$
where we \emph{assume independence among all attributes}\footnote{
Note that this independence assumption is only for the convenience of experiment design, but \emph{not} a restriction of our method, 
as one can always directly compute the joint KL. 
}. 
Note that this independence assumption generally does \emph{not} hold for the generated attribute distribution $p^u_\y(\mathbf y)$. 
With independence, the terminal cost decomposes as
$
    \kldiv{\hat p^u_\y}{p^\text{tar}_\y} = \sum_{i=1}^N \kldiv{\hat p^u_{\y_i}}{p^\text{tar}_{\y_i}} 
    + \sum_{i=1}^N \mathrm{H}(\hat p^u_{\y_i}) - \mathrm{H}(\hat{p}^u_{\y})
,$
where
$\hat p^u_{\y_i}$ is the empirical marginal, and
$\mathrm{H}(\cdot)$ denotes entropy. 

\textbf{Baselines:} 
We compare our method with vanilla \textsc{DDIM}\citep{song2021denoising}, \textsc{LFM}\citep{dao2023flow}, and \textsc{PG} \citep{corso2024particle}. 
The setup of PG is similar to that in the previous experiment. 
See details in \apdxref{apdxsub:face}.

\begin{table*}[tb]
    \centering
    \footnotesize
    \setlength{\tabcolsep}{1.0mm}
    \caption{Quantitative evaluation metrics on face generation with controlling \emph{single} attribute. Best in \textbf{bold}, second-best \underline{underlined}. \textsc{Ours-F} denotes latent-flow-based version of our method.} 
    \label{tab:face_single}
    \begin{tabular}{l|ccccc|ccccc|ccccc}
    \toprule
    & \multicolumn{5}{c|}{\textbf{\texttt{age}}}& \multicolumn{5}{c|}{\textbf{\texttt{gender}}}& \multicolumn{5}{c}{\textbf{\texttt{race}}}\\
    \textsc{Method} 
        & TV$\downarrow$ & JS$\downarrow$ & $\chi^2$$\downarrow$ & FD$\downarrow$ & FID$\downarrow$ 
        & TV$\downarrow$ & JS$\downarrow$ & $\chi^2$$\downarrow$ & FD$\downarrow$ & FID$\downarrow$ 
        & TV$\downarrow$ & JS$\downarrow$ & $\chi^2$$\downarrow$ & FD$\downarrow$ & FID$\downarrow$ \\
    \hline
    \multicolumn{16}{c}{$\mathsf{Uniform}$} \\
    \hline
    \textsc{DDIM} 
        & .21 & .15 & .044 & .25 & {50.70}
        & .042 & .029 & .0017 & {.052} & 50.70 
        & {.56} & {.45} & {.36} & {.65} & {50.70} \\ 
    \textsc{PG-DPS} 
        & .15 & .13 & .033 & .18 & 114.8
        & {.043} & {.030} & .0018 & .047 & 77.99
        & .37 & .32 & .19 & .45 & {104.4}
        \\
    \textsc{LFM}
         & .28 & .042 & .083 & .34 & 50.40
         & .16 & .013 & .026 & .22 & 50.40
         & .57 & .20 & .36 & .66 & 50.40 \\ 

    \rowcolor{green!15}\textsc{Ours}
        & \underline{.023} & \underline{.018} & \underline{6.4e-4} & \underline{.028}  & \textbf{48.59}
        & \underline{.0052} & \underline{.0037} & \underline{2.7e-5} & \underline{.0087}  & \underline{48.37}
        & \underline{.028} & \underline{.025} & \underline{.0012} & \underline{.038}  & \textbf{45.57}
        \\
    \rowcolor{green!15}\textsc{Ours-F}
         & \textbf{.018} & \textbf{1.8e-4} & \textbf{3.6e-4} & \textbf{.023} & \underline{48.72}
         & \textbf{0.0} & \textbf{0.0} & \textbf{0.0} & \textbf{5.2e-4} & \textbf{48.14}
         & \textbf{.024} & \textbf{3.0e-4} & \textbf{6.0e-4} & \textbf{.026} & \underline{47.04} \\
    \hline
    \multicolumn{16}{c}{$\mathsf{Custom}$-$\mathsf{1}$} \\
    & \multicolumn{5}{c|}{$[4:1:3]$} & \multicolumn{5}{c|}{$[2:8]$} & \multicolumn{5}{c}{$[4:3:2:1]$}\\
    \hline
    \textsc{DDIM} 
        & {.42} & {.32} & {.20} & {.51}  & {49.65} 
        & {.26} & {.20} & {.076} & 37  & {48.58}  
        & {.41} & {.35} & {.22} & {.50}  & {48.71}  \\   
    \textsc{PG-DPS}
         & {.38} & {.31} & {.18} & {.46} & 123.1        
         & {.24} & {.18} & {.065} & {.33} & 129.6        
         & {.27} & {.27} & {.13} & {.30} & 165.3 \\
    \textsc{LFM}
         & .49 & .14 & .26 & .60 & 51.95
         & .14 & .013 & .025 & .20 & 45.87
         & .42 & .13 & .23 & .51 & 48.80 \\ 
    \rowcolor{green!15}\textsc{Ours}
         & \underline{.049} & \underline{.035} & \underline{.0025} & \underline{.063} & \textbf{46.14}
         & \underline{.0094} & \underline{.0082} & \underline{1.4e-4} & \underline{.017} & \underline{46.52}
         & \underline{.043} & \underline{.040} & \underline{.0032} & \underline{.062} & \textbf{43.03} \\
    \rowcolor{green!15}\textsc{Ours-F}
         & \textbf{.016} & \textbf{1.4e-4} & \textbf{2.7e-4} & \textbf{.018} & \underline{48.25}
         & \textbf{.0073} & \textbf{4.1e-5} & \textbf{8.2e-5} & \textbf{.010} & \textbf{45.76}
         & \textbf{.029} & \textbf{4.9e-4} & \textbf{9.8e-4} & \textbf{.034} & \underline{44.73} \\
    \hline
    \multicolumn{16}{c}{$\mathsf{Custom}$-$\mathsf{2}$} \\
    & \multicolumn{5}{c|}{$[2:3:4]$} & \multicolumn{5}{c|}{$[7:3]$} & \multicolumn{5}{c}{$[1:1:4:4]$}\\
    \hline
    \textsc{DDIM} 
        & {.23} & {.18} & {.066} & {.32}  & {51.39} 
        & {.24} & {.17} & {.060} & {.33}  & {55.74}  
        & {.70} & {.56} & {.53} & {.84}   & {51.45}  \\
    \textsc{PG-DPS}
        & .30 & .24 & .11 & .36 & 90.79        
        & {.043} & {.032} & {.0021} & {.084} & 156.9        
        & {.31} & {.28} & {.14} & {.36} & 253.2 \\ 
    \textsc{LFM}
         & .29 & .058 & .11 & .40 & 51.50
         & .36 & .066 & .13 & .51 & 56.16
         & .71 & .30 & .53 & .84 & 51.12 \\ 
    \rowcolor{green!15}\textsc{Ours}
         & \textbf{.013} & \underline{.012} & \textbf{2.8e-4} & \textbf{.018} & \textbf{46.25}
         & \textbf{1.2e-8} & \textbf{1.0e-4} & \textbf{1.7e-16} & \textbf{.0011} & \textbf{48.69}
         & \underline{.053} & \underline{.041} & \underline{.0034} & \underline{.075} & \underline{46.53} \\
    \rowcolor{green!15}\textsc{Ours-F}
         & \underline{.022} & \textbf{2.8e-4} & \underline{5.6e-4} & \underline{.031} & \underline{47.83}
         & \underline{.020} & \underline{2.3e-4} & \underline{4.6e-4} & \underline{.029} & \underline{51.22}
         & \textbf{.016} & \textbf{1.8e-4} & \textbf{3.7e-4} & \textbf{.020} & \textbf{46.09} \\ 
    \bottomrule
    \end{tabular}
\vspace{-1em}
\end{table*}

\textbf{Evaluation Metrics:} 
As in the previous experiment, we measure attribute-distribution alignment using JS, TV, and $\chi^2$, and we use FID to evaluate image quality. 
We also measure the Fairness Discrepancy (FD), following prior work on fair generation \citep{choi2020fair, parihar2024balancing}:
$
    \mathrm{FD} := \left\| p^\text{tar}_\y - \mathbb{E}_{\mathbf x \sim p^{\theta,\mathbf u}_{\mathbf x} (\mathbf x)} \hat p_\y(\mathbf y \mid \mathbf x) \right\|,
$
where $\hat p_\y(\mathbf y \mid \mathbf x)$ is the \emph{softmax output} of the attribute oracle model $\Psi(\mathbf x)$.

\begin{wraptable}{r}{0.46\textwidth}
  \setlength{\tabcolsep}{1.0mm}
  \captionof{table}{
    Quantitative results for face generation with \emph{joint attribute control}.
    Best in \textbf{bold}, second-best \underline{underlined}.
  }
  \label{tab:face_multi}
  \resizebox{0.48\textwidth}{!}{
  \begin{tabular}{l|ccccc}
    \toprule
    \multicolumn{6}{c}{\textbf{\texttt{age$\times$gender$\times$race}}}\\
    \textsc{Method} 
        & TV$\downarrow$ & JS$\downarrow$ & $\chi^2$$\downarrow$ & FD$\downarrow$ & FID$\downarrow$ \\
    \hline
    \multicolumn{6}{c}{$\mathsf{UniformJoint}$} \\
    \hline
    \textsc{DDIM} 
        & {0.566} & {0.477} & {0.393} & {0.308}  & {46.26} \\
    \textsc{PG-DPS}
        & {0.444} & {0.404} & {0.281} & {0.238}  &  {127.0} \\
    \textsc{LFM}
         & 0.584 & 0.251 & 0.424 & 0.359 & 46.20 \\ 
    \rowcolor{green!15}\textsc{Ours}
        & \underline{0.112} & \underline{0.107} & \underline{0.0225} & \underline{0.0606}  & \textbf{41.28} \\
    \rowcolor{green!15}\textsc{Ours-F}
         & \textbf{0.0963} & \textbf{0.00783} & \textbf{.0155} & \textbf{.0504} & \underline{41.54} \\
    \hline
    \multicolumn{6}{c}{$\mathsf{CustomJoint}$} \\
    \multicolumn{6}{c}{$[5:2:3] \times [3:7] \times [4:3:2:1]$} \\
    \hline
    \textsc{DDIM} 
        & {0.493} & {0.444} & {0.337} & {0.305}  & {44.28} \\
    \textsc{PG-DPS}
        & {0.531} & {0.442} & {0.335} & {0.287}  & {74.77} \\
    \textsc{LFM}
         & 0.500 & 0.213 & 0.361 & 0.346 & 44.05 \\ 

    \rowcolor{green!15}\textsc{Ours}
        & \underline{0.142} & \underline{0.131} & \underline{0.034} & \underline{0.090}  & \underline{42.19} \\
    \rowcolor{green!15}\textsc{Ours-F}
         & \textbf{0.0968} & \textbf{0.00674} & \textbf{0.0134} & \textbf{0.0493} & \textbf{38.65} \\
    \bottomrule
  \end{tabular}
  }
\end{wraptable}

\textbf{Results:} 
Tables~\ref{tab:face_single} and \ref{tab:face_multi} report alignment performance and sample quality for individual and joint target attribute distributions, respectively, with qualitative samples shown in \cref{fig:face_samples}.
For $\mathsf{Uniform}$ targets, pretrained \textsc{DDIM} and \textsc{LFM} show pronounced age and race biases inherited from training data. 
\textsc{PG} reduces these biases but at a large cost in image quality. 
In contrast, our method under both diffusion and latent flow matching paradigms achieves better distributional alignment across metrics while preserving sample quality.
Under the $\mathsf{Custom}$ targets with strongly skewed distributions, our method performs consistently better in matching the target distributions \emph{and} maintains sample quality, whereas \textsc{PG} yields inconsistent performance given different targets.
Similar trends hold for joint alignment: \textsc{PG} trades quality for alignment and does not reliably outperform vanilla \textsc{DDIM}, while our method improves alignment without degrading sample quality. See additional results in \apdxref{apdxsub:additional_face}.

\section{Discussions and Conclusion}
\label{sec:conclusion}
We formulate and study the attribute distribution alignment problem for pretrained unconditional diffusion models. 
We propose an inference-time, plug-and-play method that does not require any extra training or fine-tuning. 
Our results show that the proposed method is effective in aligning the generation attribute distribution to flexible test-time targets, while better preserving sample quality compared to training-free baselines. 
Limitations of our method include extra computational overheads and the need for a moderate batch size for attribute distribution estimation. 
Detailed discussions are in \apdxref{apdx:limitations}. 
Future work includes extensions to finding optimal guidance weights in conditional diffusion models and applying our approach in other domains, \eg, robotics, graphs, \etc

\newpage
\section*{Broader Impact}
This work advances inference-time control of unconditional diffusion generative models by enabling alignment to user-specified attribute distributions without training or finetuning. 
This may benefit applications such as fairness-aware data generation, controllable simulation, and robotic autonomous systems that require calibrated mixtures of behaviors. 
At the same time, the ability to steer attribute distributions could also be misused by malicious parties to amplify societal biases, manipulate demographic representation, or generate content that appears balanced in some attributes while hiding other harmful properties, depending on how attributes are defined and measured. 
We encourage practitioners to use validated models as attribute models, audit distributional outcomes across relevant subgroups and contexts, and apply standard safeguards (dataset governance, filtering, and usage policies) upon deployment of the proposed method. 
Overall, we view this paper as providing a general technical tool for distribution-level controllability whose societal impact depends on careful choice of attributes and responsible deployment.

\bibliography{refs.bib}
\bibliographystyle{unsrtnat}

\newpage
\appendix
\onecolumn

\addtocontents{toc}{\protect\setcounter{tocdepth}{3}}
{\renewcommand{\contentsname}{Table of Contents for Appendix}\tableofcontents}
\clearpage

\section{Theoretical Derivations}
\label{apdx:derivations}

\begin{proposition}[Proposition~\ref{prop:optimal_u}]
    For any fixed $t\in[0,T]$ and fixed $(\bm x_t,\bm\nu_t)$, with $\rho+\gamma>0$,  
    the optimal control $\mathbf u_t^\ast$ specified in \eqref{eq:pmp_u} is given by 
    \[
        \mathbf u_t^*\in \Pi_{\mathcal U}(\bar{\bm u}_t)
        :=\argmin_{\mathbf u\in\mathcal U}\|\mathbf u-\bar{\bm u}_t\|^2,
    \]
    where
    \[
        \bar{\bm u}_t
        :=\frac{1}{\rho+\gamma}\left(
            \gamma\,\bm u_t^{\rm ref} - \mathbf g(\bm x_t,t)^\top \bm\nu_t
        \right),
    \]
    and $\Pi_{\mathcal U}(\bar{\bm u}_t)$ is nonempty. 
    Further, $\mathbf u_t^\ast$ is unique if $\mathcal U$ is convex.  
\end{proposition}

\begin{proof}
    For any fixed $t,(\bm x_t,\bm\nu_t)$, extract the part of $H$ that depends on $\mathbf u$ as
    \[
    \Lambda_t(\mathbf u)
    :=\frac{\rho}{2} \|\mathbf u \|^2+\frac{\gamma}{2}\| \mathbf u - \bm u_t^\textrm{ref}\|^2
    +\bm\nu_t^\top \mathbf g(\bm x_t,t)\mathbf u 
    .\]
    Since $\rho+\gamma>0$, completing squares yields
    \[
        \Lambda_t(\mathbf u)
        =\frac{\rho+\gamma}{2}\|\mathbf u-\bar{\bm u}_t\|^2
        +c_t,
    \]
    with
    \[
        \bar{\bm u}_t=\frac{1}{\rho+\gamma}\left(
            \gamma\,\bm u_t^{\rm ref}-\mathbf g(\bm x_t,t)^\top\bm\nu_t
        \right)
    \]
    and 
    \begin{align*}
        c_t
        &=\frac{\gamma}{2}\|\bm u_t^\textrm{ref}\|^2-\frac{\rho+\gamma}{2}\|\bar{\bm u}_t\|^2
    .\end{align*}
    Therefore,
    \[
        H(\bm x_t,\bm\nu_t,\mathbf u,t)
        =\frac{\rho+\gamma}{2}\|\mathbf u-\bar{\bm u}_t\|^2+\bm\nu_t^\top \mathbf f^\theta(\bm x_t,t)+c_t.
    \]
    The last two terms do not depend on $\mathbf u$, hence
    \[
        \argmin_{\mathbf u\in\mathcal U}\widetilde H(\bm x_t,\bm\nu_t,\mathbf u,t)
        =\argmin_{\mathbf u\in\mathcal U}\|\mathbf u-\bar{\bm u}_t\|^2
        =\Pi_{\mathcal U}(\bar{\bm u}_t).
    \]
    Since $\mathcal U$ is a nonempty and closed set and $\mathbf u \mapsto \|\mathbf u-\bar{\bm u}_t\|^2$ is continuous and coercive, the minimum is attained per \cite[Theorem~2.32]{beck2014introduction}. 
    Hence, $\Pi_{\mathcal U}(\bar{\bm u}_t)$ is nonempty. 
    Further, if set $\mathcal U$ is convex, since $\mathbf u \mapsto \|\mathbf u-\bar{\bm u}_t\|^2$ is \emph{strictly convex} in $\mathbf u$, the set of minimizer, \ie, $\Pi_{\mathcal U}(\bar{\bm u}_t)$, contains \emph{at most one point} \cite[Sec.~4.2.1]{boyd2004convex}; its uniqueness is therefore obtained when combined with non-emptiness. 
\end{proof}

\clearpage
\section{Method Details}
\label{apdx:method_details}

    \subsection{PF-ODE Instances}
    \label{apdxsub:pfodes}
    In this work, we focus on the \emph{\gls{pfode} of the reverse diffusion process}, which appears in various forms across literature. 
    We identify two instances here. 
    With the formulation of EDM \citep{karras2022elucidating}, the \gls{pfode} of a reverse diffusion process reads: 
    \begin{equation}
        \dot{\mathbf x}_t = -t \mathbf \, \mathbf{s}^\theta(\x_t, t) 
    .\end{equation}
    The associated \gls{pfode} of DDIM, which corresponds to that of the ``Variance-Exploding'' \gls{sde} in \cite{song2021scorebased}, is as follows \cite{song2021denoising}: 
    \begin{equation}
    \label{eq:ddim_ode}
        \mathrm{d}\tilde{\x}_t = \boldsymbol{\epsilon}^{\theta}\left(\frac{\tilde{\x}_t}{\sqrt{1 + \sigma(t)^2}}\right) {\mathrm{d}\sigma(t)} 
        \quad
        \text{with}\quad
        \tilde{\x}_t := \frac{\x_t}{\sqrt{\alpha_t}} \quad \text{and} \quad
        \sigma(t) := \sqrt{(1-\alpha_t) / {\alpha_t}} 
    \end{equation}
    with $\alpha_t$ being a decreasing sequence related to the diffusion noising schedule\footnote{
        There is a notional mismatch in $\alpha_t$ across different work. 
        The $\alpha_t$ in \citep{song2021denoising} corresponds to the $\bar{\alpha}_t$ in \citep{ho2020denoising}. 
        We refer readers to \cite{ho2020denoising, song2021denoising} for the details. 
        In the above, we adopt the notations by \citet{song2021denoising}. 
    },
    and $\epsilon^\theta$ a learned noise prediction model parameterized by $\theta$. 
    Without loss of generality, we denote both of them by
    \begin{equation}
        \dot{\mathbf x}_t = \mathbf f^\theta(\mathbf x_t, t), \quad
        \mathbf x_0 \sim \mathcal N(\mathbf 0, \I_n),
        \quad t \in [0, T].
    \end{equation}

    \paragraph{Flow ODE}
    The ODE under the flow matching paradigm \cite{lipman2023flow,liu2023flow} is straightforward: 
    \begin{equation}
    \label{eq:flowode}
        \dot{\mathbf x}_t = v^\theta(\x_t, t)
    \end{equation}
    where $t\in [0,1]$ and $v^\theta: \mathbb R^n \times [0,1] \to \mathbb R^n$ is a learned vector field model parameterized by $\theta$. 
    The inference generation process is simply integrating the \eqref{eq:flowode} given a prior distribution $p_1(\x_t)$. 
    Under the latent Flow Matching (LFM) \cite{dao2023flow} setting, the state $\x$ lies in the latent space of some pretrained VAE. 

    \paragraph{Perturbed PF-ODE Instances} 
    In this work, we instantiate the perturbed \gls{pfode} \eqref{eq:ctrl_ode} under the EDM \cite{karras2022elucidating} and DDIM \cite{song2021denoising} formulations. 
    For EDM, we take the following \emph{perturbed} ODE: 
    \begin{equation}
    \label{eq:ctrl_ode_edm}
        \mathbf{f}^\theta_\text{EDM}(\mathbf x_t, \mathbf u_t, t) 
        := \dot{\mathbf x}_t = (T - t) \mathbf s^\theta(\x_t, T-t) + \mathbf u_t
    ,\end{equation}
    setting $\mathbf g(\mathbf x_t, t) = \I_n$. 
    For DDIM, we adopt  
    \begin{equation}
        \mathbf{f}^\theta_\text{DDIM}(\tilde{\x}_\sigma, \mathbf u_\sigma, \sigma) 
        := \frac{\mathrm d \tilde{\x}_\sigma}{\mathrm d \sigma}  =  \epsilon^{\theta}\left(\frac{\tilde{\x}_\sigma}{\sqrt{1 + \sigma^2}}\right) + \u_\sigma 
    ,\end{equation}
    and perform control in the spaces of $\tilde{\x}_\sigma$ and $\sigma(T-t)$. 

    The perturbed ODE under the FM setting adopted in this work is plainly: 
    \begin{equation}
    \label{eq:ctrl_ode_fm}
        \mathbf{f}^\theta_\text{FM}(\mathbf x_t, \mathbf u_t, t) 
        := \dot{\mathbf x}_t = v^\theta(\x_t, t) + \mathbf u_t
    .\end{equation}

    \subsection{Algorithm}
    \label{apdxsub:algo}

\begin{algorithm}[t]
\caption{DADA via E-MSA with Euler discretization}
\label{alg:dada_emsa_1}
\begin{algorithmic}[1]
\Require
    Dynamics $\mathbf F^\theta$; 
    terminal cost $\Phi$; 
    batch init $\bm X_{\rm init}$;
    time grid $\{t_k\}_{k=0}^K$ with $t_0=0,t_K=T$; 
    $\rho>0$, $\xi\ge0$; 
    max iteration $I$; 
    optional $\mathcal U$.
\State $\eta \gets (1-\xi)/\rho$ 
\State $h_k \gets t_{k+1}-t_k$ for $k=0,\dots,K-1$ 
\State $\mathbf U_k\gets \bm0$ for $k=0,\dots,K-1$
\For{$j=1$ to $I$}
    \LComment{(1) Forward pass with control $\mathbf U$}
    \State $\X_0 \gets \X_{\rm init}$
    \For{$k=0,\dots,K-1$}
        \State $\X_{k+1} \gets \X_k + h_k\, \mathbf F^\theta(\X_k, \mathbf U_k, t_k)$
    \EndFor
    \LComment{(2) Cost evaluation and backward pass}
    \State $\mathbf N_K \gets \nabla_\X \Phi \left(\hat p^u_{\y}(\X_K)\right)$
    \For{$k=K-1,\dots,0$}
        \LComment{Computed via VJP}
        \State $\mathbf V_k \gets \big(\nabla_\X \mathbf F^\theta(\X_k, \mathbf U_k, t_k)\big)^\top \mathbf N_{k+1}$ 
        \State $\mathbf N_k \gets \mathbf N_{k+1} + h_k \mathbf V_k$
    \EndFor
    \LComment{(3) Closed-form control update}
    \For{$k=0,\dots,K-1$}
        \State $\mathbf U_k^{\rm new} \gets \xi \mathbf U_k - \eta\,\mathbf G(X_k,t_k)^\top \mathbf N_{k+1}$
        \State $\mathbf U_k \gets \Pi_{\mathcal U^M}(\mathbf U_k^{\rm new})$ if $\mathcal U$ specified else $\mathbf U_k^{\rm new}$
    \EndFor
    \If{Converged} \textbf{break} \EndIf
\EndFor
\LComment{Forward simulation with final control}
\For{$k=0,\dots,K-1$}
    \State $\X_{k+1} \gets \X_k + h_k\, \mathbf F^\theta(\X_k, \mathbf U_k, t_k)$
\EndFor
\State \textbf{return} $\X_K$ 
\end{algorithmic}
\end{algorithm}

    The algorithmic description of our method is in \cref{alg:dada_emsa_1}.

    \subsection{Computational Complexity}
    \label{apdxsub:complexity}
    \paragraph{Time Complexity.}
    Let $C_{\rm NFE}$ be the cost of one \emph{batched} forward evaluation of the pretrained diffusion or flow network $\mathbf f^\theta$ on the $M$-sample batch, \ie, one neural function evaluation (NFE), and $C_{\rm VJP}$ the cost of one batched reverse-mode vector-Jacobian product (VJP) through $\mathbf f^\theta$ on the same batch.
    Let $C_\Phi$ further denote the cost of one batched forward-plus-backward through the differentiable attribute model $\Psi$ (composed with the density estimator) on the $M$-sample batch.
    Standard autodiff frameworks such as PyTorch gives $C_{\rm VJP}\approx 2\text{--}3\, C_{\rm NFE}$.
    Since $C_{\rm NFE}$, $C_{\rm VJP}$, and $C_\Phi$ all scale linearly with $M$ in compute, we absorb the batch size $M$ into them.
    Each outer iteration of \cref{alg:dada_emsa_1} consists of
    (i) a forward pass of $K$ batched NFEs at cost $O(K\,C_{\rm NFE})$;
    (ii) one forward and backward pass through $\Psi$ to compute the terminal adjoint $\mathbf N_K$ 
    at cost $C_\Phi$, independent of $K$ 
    and is invoked only once per outer iteration;
    (iii) a backward (adjoint) pass of $K$ batched VJPs at cost $O(K\,C_{\rm VJP})$;
    and (iv) a closed-form control update \eqref{eq:emsa_u_update} that is elementwise with an \emph{optional} projection onto $\mathcal U^M$ (most commonly a clipping operation and negligible to the NFEs, so considered as $O(1)$ cost). 
    A final forward pass of $K$ NFEs generates the resulting samples.
    Overall, the algorithm runs in
    \begin{equation}
        O\big(I\cdot \left[K\,(C_{\rm NFE} + C_{\rm VJP}) + C_\Phi\right]\big)
    .\end{equation}
    As such, our method incurs an $I\cdot\big(1 + C_{\rm VJP}/C_{\rm NFE} + C_\Phi/(K\,C_{\rm NFE})\big)$ multiplicative overhead over the $O(K\, C_{\rm NFE})$ cost of vanilla unconditional sampling, where the last term is typically the smallest.

    \paragraph{Space Complexity.}
    Let $|\theta|$ and $A_{\rm net}$ denote the parameter size and the peak forward-plus-VJP activation memory of $\mathbf f^\theta$ on the $M$-sample batch, and let $|\Psi|$ and $A_\Phi$ denote the corresponding quantities for the differentiable attribute model $\Psi$ (composed with the density estimator). 
    Note that all four are independent of $K$. 
    Across one outer iteration of \cref{alg:dada_emsa_1}, the persistent buffers are
    (i) the state trajectory $\{\X_k\}_{k=0}^{K}$ at $O(KMn)$;
    (ii) the control trajectory $\{\mathbf U_k\}_{k=0}^{K-1}$ at $O(KMm)$;
    and (iii) the adjoint/costate trajectory $\{\mathbf N_k\}_{k=0}^{K}$ at $O(KMn)$, but this is reduced to $O(Mn)$ since the update \eqref{eq:emsa_u_update} is fused into the backward loop in our implementation (note that the control update only depends on the state and costate of the current time step).
    Each per-step VJP is computed locally on one $\mathbf f^\theta$ call and discarded before the next, so transient activations are bounded by $O(A_{\rm net})$ at any step within the backward loop. 
    The terminal-cost gradient $\nabla_\X\Phi(\hat p^u_\y(\X_K))$ contributes one transient pass through $\Psi$ of activation cost $O(A_\Phi)$, released before the adjoint loop begins. 
    All of the aforementioned buffers are reused across outer iterations, so peak memory \emph{does not} grow with $I$.
    In sum, the space complexity is
    \begin{equation}
        O\big(|\theta| + |\Psi| + KM(n+m) + A_{\rm net} + A_\Phi\big),
    \end{equation}
    and independent of number of iterations $I$. 

    \paragraph{Remarks on Avoided Costs.}
    Two hypothetical memory costs are avoided by the structure of \cref{alg:dada_emsa_1}.
    First, the Jacobian \eqref{eq:jac_blkdiag} is never explicitly formed; an explicit materialization would have incurred $O(M^2 n^2)$ memory.
    Second, end-to-end backpropagation through the entire unrolled $K$-step sampler under standard reverse-mode auto-differentiation would have incurred activations of size $O(K\, A_{\rm net})$, but our adjoint-based formulation instead retains only the state of size $O(KMn)$ and reconstructs the per-step activations on demand, which is a memory--compute trade-off equivalent to applying per-step gradient checkpointing to the unrolled sampler.

\clearpage
\section{Experiment Details}
\label{apdx:expt_details}

\paragraph{Software and Codebase}
All experiments were run using PyTorch \cite{paszke2019pytorch}. 
For the CIFAR experiment, our implementation is built upon \cite{karras2022elucidating} and the implementation of the image classifier used as the attribute model therein is from \cite{chenyaofo_pytorch}. 
For the human face generation experiment, diffusion-based implementations are adapted from \cite{choi2022perception}, and flow-based implementations are built upon \cite{dao2023flow}. 
The attribute model implementation is from \cite{luan2025projected}. 

\paragraph{Computational Hardware}
The training of the unconditional EDM model in the CIFAR experiment and \emph{all evaluation}, including computing all metrics, were run on a workstation with 1 AMD Ryzen Threadripper PRO 5995WX 64-Core CPU, 504 GB RAM, and 2 NVIDIA RTX A6000 GPUs each with 48GB VRAM. 
\emph{Inference} of all methods was run on a high-performance-computing (HPC) cluster with NVIDIA H200 GPUs. 
For each experiment run, only 1 GPU was utilized. 

\subsection{CIFAR Experiment}
\label{apdx:cifar}

\paragraph{Hierarchical attribute distributions}
We construct three different levels of class labels with a coarse-to-fine hierarchy: \texttt{meta5}, \texttt{coarse}, and \texttt{fine}, with numbers of classes of $5$, $20$, and $100$, respectively. 
The \texttt{coarse} and \texttt{fine} levels of class labels are native in CIFAR-100. 
The \texttt{meta5} level is obtained by merging every 4 \texttt{coarse} classes. 
For each class level, we choose three different attribute distributions as test-time targets : $\mathsf{Uniform}$, $\mathsf{ZigZag}$, and $\mathsf{Gaussian}$. 
$\mathsf{ZigZag}$ has an alternating ``high-low'' pattern where every even-numbered class is exactly twice as likely to be selected as any odd-numbered class. 
For $\mathsf{Gaussian}$, the target resembles a smooth bell curve centered in the middle-numbered class, with the standard deviation as $1/4$ of the support size.

\paragraph{Base diffusion model and attribute models}
We trained a base EDM model on CIFAR-100 dataset \cite{krizhevsky2009learning} with the default training split and the network backbone is the UNet in \cite{song2019generative} implemented by \citet{karras2022elucidating}. 
The EDM training followed default hyperparameters disclosed in \cite{karras2022elucidating}. 
All attribute models are ResNet56 trained on CIFAR-100 training set, and we adopt the implementations and training hyperparameters in \cite{chenyaofo_pytorch}.

\paragraph{Diffusion inference and evaluation}
For diffusion inference for all methods, we adopt the default settings provided by \cite{karras2022elucidating} with $K=18$ sampling steps, but first-order Euler discretization instead of second-order Heun. 
For each method, we sample 10240 images and evaluate all metrics based on the empirical distribution of attributes. 

\paragraph{Particle Guidance implementation}
We implement the Particle Guidance (PG) \citep{corso2024particle} method by taking the same cost function used in our method as the potential field for PG. 
Concretely, PG operating on the \gls{pfode} is as follows: 
\begin{align}
    \frac{\mathrm d \x_t^{[i]}}{\mathrm dt} = -f(\x_t^{[i]}, t) + \frac{1}{2} g(t)^2 \left( \mathbf s^\theta \left(\x_t^{[i]}, t\right) + \nabla_{\x_t^{[i]}} \log \mathcal L \left(\x_t^{[1]}, \ldots, \x_t^{[M]} \right) \right) 
,\end{align}
where $\log \mathcal L$ is the potential field defined over $M$ samples. 
We set  
\begin{align}
    \log \mathcal L\left(\x_t^{[1]}, \ldots, \x_t^{[M]} \right) 
    :=& - \lambda \kldiv{\hat p^t_\mathbf{y}}{p^\text{tar}_\mathbf{y}} , 
    \qquad
    \hat p^t_\mathbf{y} := \frac{1}{M} \sum_{i=1}^M \textrm{softmax}\left( \Psi\left( \hat{\mathbf{x}}_T^{[i]}(\x_t^{[i]}) \right) \right)
\end{align}
where $\lambda>0$ is a hyperparameter, 
and $\hat{\mathbf{x}}_T^{[i]}(\x_t^{[i]})$ is a ``predicted clean sample'' given a noisy sample at time $t$,  obtained via Tweedie's formula \cite{efron2011tweedie}. 
Under EDM \citep{karras2022elucidating} with a learned score model $\mathbf s^\theta$, it takes the form of 
\begin{align}
    \hat{\mathbf{x}}_T^{[i]}(\x_t^{[i]}) 
    = \x_t^{[i]} + t^2  \mathbf s^\theta(\x_t^{[i]}, t)
.\end{align}
Under DDIM\citep{song2021denoising} with a learned noise prediction model $\bm \epsilon^\theta$, we take the \gls{pfode} in $\tilde\x_t$ (see \eqref{eq:ddim_ode}) and this term reads 
\begin{align}
    \hat{\tilde{\mathbf{x}}}_T^{[i]}(\tilde{\x}_t^{[i]}) = \tilde{\x}_t^{[i]} - \sigma(t)  \bm \epsilon^\theta \left(\tilde\x_t^{[i]}/\sqrt{1+\sigma(t)^2} ,\, t \right)
\end{align}
where $\tilde\x_t$ and $\sigma(t)$ are defined in \eqref{eq:ddim_ode}.

\paragraph{Hyperparameters}
The hyperparameters used to obtain the results in \tabref{tab:cifar} are reported in \tabref{tab:cifar_hyperparams_ours} and \tabref{tab:cifar_hyperparams_pg}. 

\begin{table}[htbp]
    \centering
    \caption{Hyperparameters used for our method in CIFAR. }
    \label{tab:cifar_hyperparams_ours}
    \begin{tabular}{lcccc}
    \toprule 
    Target & $I$ & $\rho$ & $M$ & $\xi$ \\
    \midrule
    $\mathsf{Uniform}$-\texttt{meta5} & 10 & 0.1 & 32 & 0.99 \\
    $\mathsf{Uniform}$-\texttt{coarse} & 10 & 0.1 & 64 & 0.99 \\
    $\mathsf{Uniform}$-\texttt{fine} & 10 & 0.1 & 256 & 0.99 \\
    $\mathsf{ZigZag}$-\texttt{meta5} & 10 & 0.05 & 32 & 0.99 \\
    $\mathsf{ZigZag}$-\texttt{coarse} & 10 & 0.05 & 64 & 0.99 \\
    $\mathsf{ZigZag}$-\texttt{fine} & 10 & 0.05 & 256 & 0.99 \\
    $\mathsf{Gaussian}$-\texttt{meta5} & 10 & 0.1 & 64 & 0.95 \\
    $\mathsf{Gaussian}$-\texttt{coarse} & 10 & 0.1 & 128 & 0.95 \\
    $\mathsf{Gaussian}$-\texttt{fine} & 10 & 0.1 & 256 & 0.95 \\
    \bottomrule
    \end{tabular}
\end{table}

\begin{table}[htbp]
    \centering
    \caption{Hyperparameters used for \textsc{PG-DPS} in CIFAR. }
    \label{tab:cifar_hyperparams_pg}
    \begin{tabular}{lcc}
    \toprule 
    Target & $w$ & $M$ \\
    \midrule
    $\mathsf{Uniform}$-\texttt{meta5} & 4.0  & 32 \\
    $\mathsf{Uniform}$-\texttt{coarse} & 4.0 & 64  \\
    $\mathsf{Uniform}$-\texttt{fine} & 4.0 & 256  \\
    $\mathsf{ZigZag}$-\texttt{meta5} & 4.0 & 32 \\
    $\mathsf{ZigZag}$-\texttt{coarse} & 4.0 & 64 \\
    $\mathsf{ZigZag}$-\texttt{fine} & 4.0 & 256  \\
    $\mathsf{Gaussian}$-\texttt{meta5} & 4.0 & 64  \\
    $\mathsf{Gaussian}$-\texttt{coarse} & 4.0 & 128  \\
    $\mathsf{Gaussian}$-\texttt{fine} & 4.0  & 256  \\
    \bottomrule
    \end{tabular}
\end{table}

\subsection{Human Face Generation Experiment}
\label{apdxsub:face}

\paragraph{Implementation details} 
We use the pretrained DDIM weights from \cite{choi2022perception} for the base diffusion model, and pretrained LFM weights from \citep{dao2023flow} for the base latent flow model. 
For all diffusion-based methods, we take $K=25$ sampling steps for during diffusion inference with FP16 precision, with other inference hyperparameters following the defaults set in \cite{choi2022perception}. 
For latent flow models, we take $K=20$ sampling steps for inference, with other settings default as in LFM~\citep{dao2023flow}, including using a standard pretrained variational Autoencoder \citep{kingma2013auto} from Stable Diffusion \citep{Rombach_2022_CVPR}. 
We sample 960 images for each method in the single-attribute distribution alignment experiments (\cref{tab:face_single}) for evaluation, and sample 1080 images for each method in the joint attribute distribution alignment cases (\cref{tab:face_multi}).  

\paragraph{Attribute model}
The attribute model is a lightweight ResNet image classifier implemented by \citet{luan2025projected}. 
We train the classifier on FFHQ-Aging.
Since the FFHQ-Aging dataset does not have a label for attribute \texttt{race}, we leverage a pretrained race classifier from \cite{karkkainen2021fairface} to label all images in FFHQ-Aging and use them as ground truth. 

\paragraph{Hyperparameters}
The hyperparameters specific to each method used to obtain the results in \tabref{tab:face_single} and \tabref{tab:face_multi} are reported in \tabref{tab:face_hyperparams_ours} \tabref{tab:face_hyperparams_ours_flow}, and \tabref{tab:face_hyperparams_pg}. 
The selection of hyperparameters is based on each method's best performance in the TV metric. 

\begin{table}[htbp]
    \centering
    \caption{Hyperparameters used for our method (diffusion-based) in human face generation. }
    \label{tab:face_hyperparams_ours}
    \begin{tabular}{lcccc}
    \toprule 
    Target & $I$ & $\rho$ & $M$ & $\xi$ \\
    \midrule
    $\mathsf{Custom1}$-\texttt{age} & 12 & 0.0015 & 24 & 0.95 \\
    $\mathsf{Custom1}$-\texttt{gender} & 10 & 0.00125 & 20 & 0.95 \\
    $\mathsf{Custom1}$-\texttt{race} & 10 & 0.0005 & 20 & 0.95 \\
    $\mathsf{Custom2}$-\texttt{age} & 10 & 0.0010 & 24 & 0.95 \\
    $\mathsf{Custom2}$-\texttt{gender} & 10 & 0.00075 & 20 & 0.95 \\
    $\mathsf{Custom2}$-\texttt{race} & 10 & 0.0010 & 20 & 0.95 \\
    $\mathsf{CustomJoint}$ & 10 & 0.0002 & 72 & 0.95 \\
    $\mathsf{Uniform}$-\texttt{age} & 10 & 0.002 & 24 & 0.95 \\
    $\mathsf{Uniform}$-\texttt{gender} & 10 & 0.001 & 20 & 0.95 \\
    $\mathsf{Uniform}$-\texttt{race} & 10 & 0.0005 & 20 & 0.95 \\
    $\mathsf{UniformJoint}$ & 10 & 0.0002 & 72 & 0.95 \\
    \bottomrule
    \end{tabular}
\end{table}

\begin{table}[htbp]
    \centering
    \caption{Hyperparameters used for our method (flow-based) in human face generation. }
    \label{tab:face_hyperparams_ours_flow}
    \begin{tabular}{lcccc}
    \toprule 
    Target & $I$ & $\rho$ & $M$ & $\xi$ \\
    \midrule
    $\mathsf{Custom1}$-\texttt{age} & 10 & 0.00075 & 24 & 0.99 \\
    $\mathsf{Custom1}$-\texttt{gender} & 10 & 0.001 & 20 & 0.99 \\
    $\mathsf{Custom1}$-\texttt{race} & 10 & 0.0005 & 20 & 0.99 \\
    $\mathsf{Custom2}$-\texttt{age} & 10 & 0.00075 & 24 & 0.99 \\
    $\mathsf{Custom2}$-\texttt{gender} & 10 & 0.001 & 20 & 0.99 \\
    $\mathsf{Custom2}$-\texttt{race} & 10 & 0.0005 & 20 & 0.99 \\
    $\mathsf{CustomJoint}$ & 10 & 0.0001 & 72 & 0.99 \\
    $\mathsf{Uniform}$-\texttt{age} & 10 & 0.001 & 24 & 0.99 \\
    $\mathsf{Uniform}$-\texttt{gender} & 10 & 0.00125 & 20 & 0.99 \\
    $\mathsf{Uniform}$-\texttt{race} & 10 & 0.0005 & 20 & 0.99 \\
    $\mathsf{UniformJoint}$ & 10 & 0.00025 & 72 & 0.99 \\
    \bottomrule
    \end{tabular}
\end{table}

\begin{table}[htbp]
    \centering
    \caption{Hyperparameters used for \textsc{PG-DPS} in human face generation. }
    \label{tab:face_hyperparams_pg}
    \begin{tabular}{lcc}
    \toprule 
    Target  & $w$ & $M$  \\
    \midrule
    $\mathsf{Custom1}$-\texttt{age} & 10 & 24 \\
    $\mathsf{Custom1}$-\texttt{gender} & 50 & 20 \\
    $\mathsf{Custom1}$-\texttt{race} & 50 & 20 \\
    $\mathsf{Custom2}$-\texttt{age} & 10 & 24 \\
    $\mathsf{Custom2}$-\texttt{gender} & 50 & 20 \\
    $\mathsf{Custom2}$-\texttt{race} & 100 & 20 \\
    $\mathsf{CustomJoint}$ & 4 & 72 \\
    $\mathsf{Uniform}$-\texttt{age} & 15 & 24 \\
    $\mathsf{Uniform}$-\texttt{gender} & 10 & 20 \\
    $\mathsf{Uniform}$-\texttt{race} & 10 & 20 \\
    $\mathsf{UniformJoint}$ & 7 & 72 \\
    \bottomrule
    \end{tabular}
\end{table}

\clearpage

\section{Additional Results}
\label{apdx:additional_results}

\subsection{Additional Results for CIFAR Experiment}
\label{apdxsub:additional_cifar}

\subsubsection{Ablation Study}

\paragraph{Batch size $M$}
We perform ablation study in the batch size $M$ used for empirical distribution calculation. 
We set $M \in \{8, 16, 32, 64, 128, 256\}$ and test all methods with all of the three types of target distributions ($\mathsf{Uniform}$, $\mathsf{Guassian}$, $\mathsf{ZigZag}$) with all three levels of attributes (\texttt{meta5}, \texttt{coarse}, \texttt{fine}). 
The resulting metrics are plotted in \cref{fig:cifar_uniform_all}, \cref{fig:cifar_gaussian_all}, and \cref{fig:cifar_zigzag_all}. 
The figures show that across all target distributions with different support sizes, the performance (in terms of distributional metrics) appears to first improve as the batch size $M$ increases, and then start to slowly degrade if $M$ keeps growing (note the log scale of the x-axis in all figures). 
A similar pattern is also reported by \citet{parihar2024balancing} when using sample batches to perform distribution guidance. 
The turning point appears to vary across different targets: $\sim 5$ times of the support size in $\mathsf{Uniform}$, while up to $~10$ times of the support size in $\mathsf{Guassian}$ and $\mathsf{ZigZag}$. 
For baseline PG-DPS, the performance change appears to be minor as the batch size $M$ changes. 
With a reasonable batch size $M$ for estimating the empirical distribution (considering both attribute distribution support size and the target), our method can achieve better performance than the baselines. 

\paragraph{Runtime and Memory Benchmark}
We perform ablations to empirically show how runtime and memory scale over batch size $M$, inference steps $K$, and max number of iterations $I$. 
The results are in \cref{tab:batch-size-sweep}, \cref{tab:maxiter-sweep}, \cref{tab:diffusion-steps-sweep}.  
The total batch runtime profile of our method is also shown in \cref{fig:cifar_runtime} as the $M$, $I$, and $K$ varies. 
The observed patterns from the empirical results appear to match our analyses in \apdxref{apdxsub:complexity}: 
\begin{itemize}
    \item Runtime per sample grows roughly linearly with max number of iterations $I$ and inference steps $K$. 
    \item Runtime per sample decreases with larger batch size $M$ because of batching while peak memory increases linearly with $M$. 
    \item Peak memory usage is mostly dominated by batch size $M$. 
\end{itemize}

\subsubsection{Additional Samples}
We provide qualitative samples generated by baselines and our methods in \cref{fig:cifar_samples_meta5_gaussian}, \cref{fig:cifar_samples_meta5_uniform}, \cref{fig:cifar_samples_meta5_zigzag}.

\begin{table}[t]
\centering
\caption{Runtime per sample and peak memory as batch size $M$ varies.}
\label{tab:batch-size-sweep}
\begin{tabular}{rrrrr}
\toprule
$M$ & $K$ & $I$ & Runtime per sample (ms) & Peak Memory (GB) \\
\midrule
8   & 18 & 10 & $912 \pm 1.46$   & 1.21 \\
16  & 18 & 10 & $473 \pm 0.981$  & 2.11 \\
32  & 18 & 10 & $260 \pm 0.239$  & 3.93 \\
64  & 18 & 10 & $221 \pm 0.663$  & 7.59 \\
128 & 18 & 10 & $209 \pm 0.238$  & 14.9 \\
256 & 18 & 10 & $203 \pm 0.105$  & 29.5 \\
\bottomrule
\end{tabular}
\end{table}

\begin{table}[t]
\centering
\caption{Runtime per sample and peak memory as max iterations $I$ varies.}
\label{tab:maxiter-sweep}
\begin{tabular}{rrrrr}
\toprule
$I$ & $M$ & $K$ & Runtime per sample (ms) & Peak Memory (GB) \\
\midrule
4  & 32 & 18 & $108 \pm 0.15$  & 3.93 \\
6  & 32 & 18 & $159 \pm 0.291$ & 3.93 \\
8  & 32 & 18 & $210 \pm 0.309$ & 3.93 \\
10 & 32 & 18 & $260 \pm 0.181$ & 3.93 \\
12 & 32 & 18 & $308 \pm 0.274$ & 3.93 \\
14 & 32 & 18 & $358 \pm 0.195$ & 3.93 \\
\bottomrule
\end{tabular}
\end{table}

\begin{table}[t]
\centering
\caption{Runtime per sample and peak memory as inference steps $K$ varies.}
\label{tab:diffusion-steps-sweep}
\begin{tabular}{rrrrr}
\toprule
$K$ & $M$ & $I$ & Runtime per sample (ms) & Peak Memory (GB) \\
\midrule
10 & 32 & 10 & $146 \pm 0.242$ & 3.92 \\
14 & 32 & 10 & $202 \pm 0.239$ & 3.93 \\
18 & 32 & 10 & $260 \pm 0.202$ & 3.93 \\
22 & 32 & 10 & $317 \pm 0.206$ & 3.94 \\
26 & 32 & 10 & $370 \pm 0.306$ & 3.94 \\
30 & 32 & 10 & $428 \pm 0.213$ & 3.95 \\
\bottomrule
\end{tabular}
\end{table}

\begin{figure}[ht]
  \begin{center}
    \centerline{
        \includegraphics[width=0.33\columnwidth]{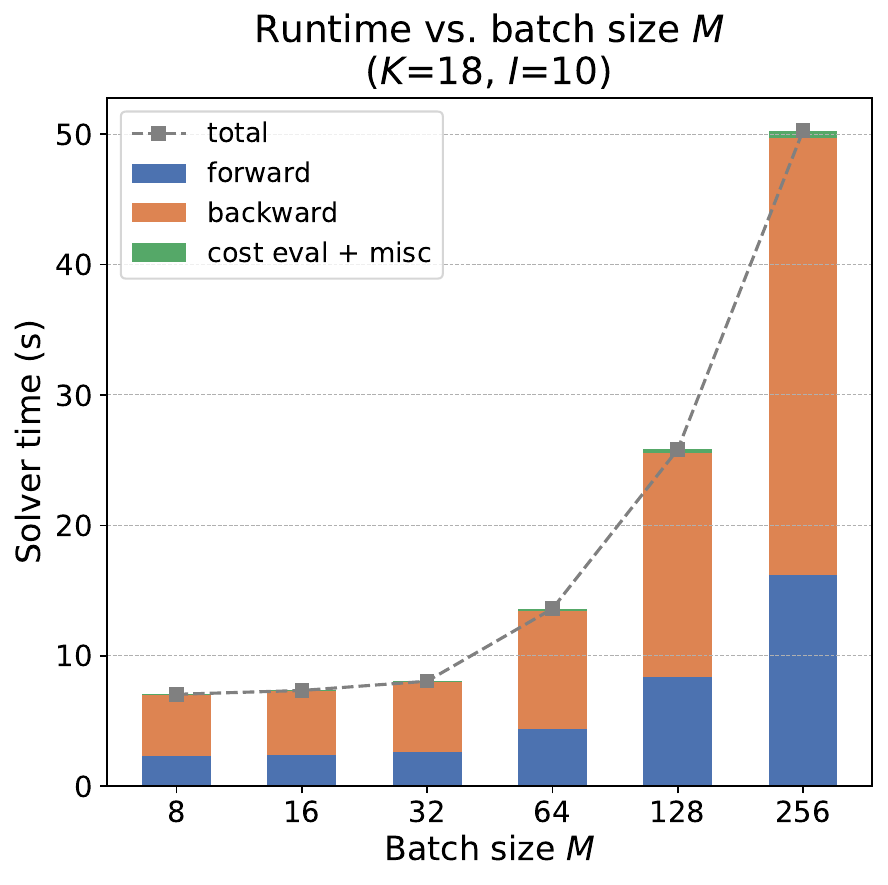}
        \includegraphics[width=0.33\columnwidth]{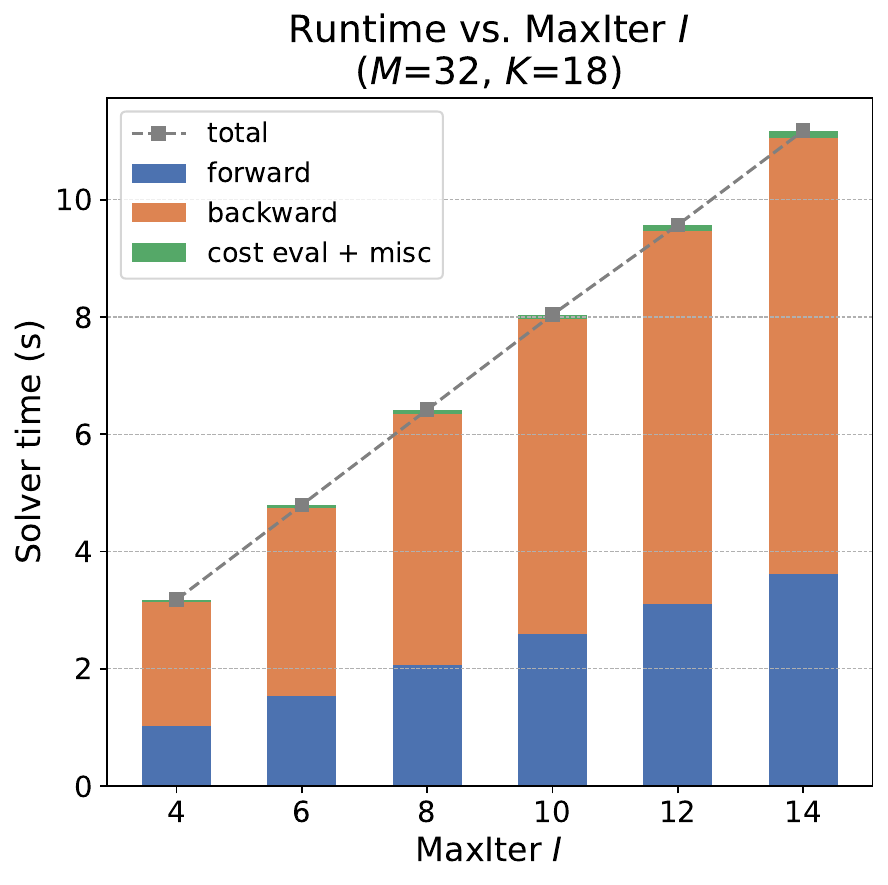}
        \includegraphics[width=0.33\columnwidth]{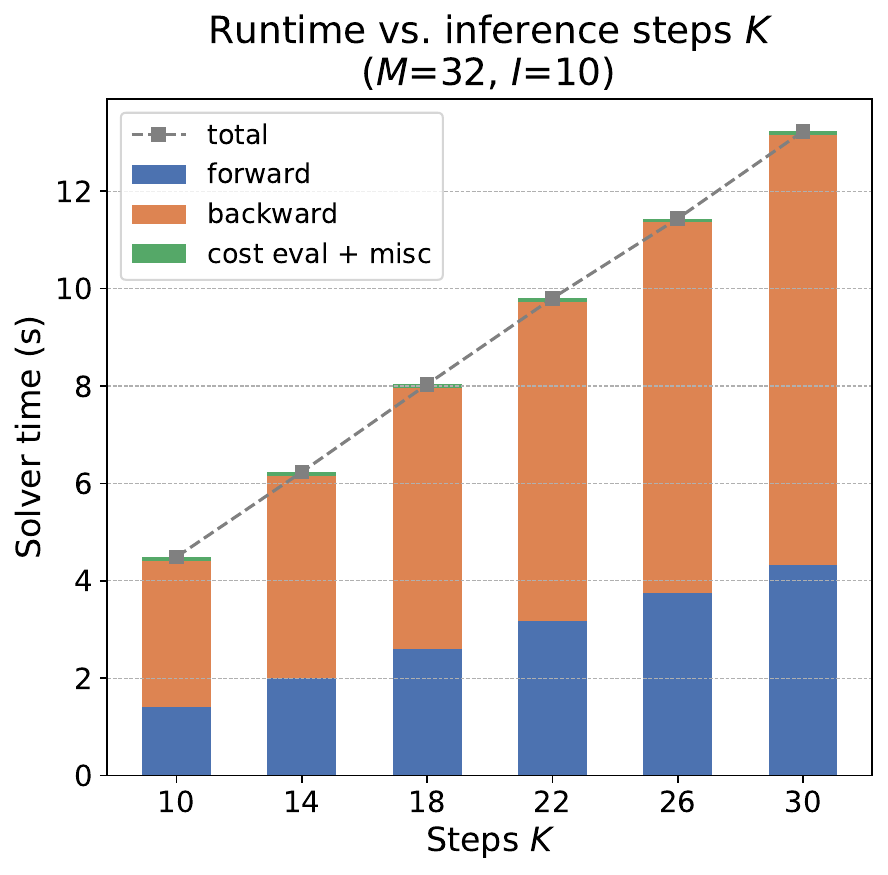}
    }
    \caption{
        Total runtime profile of our method as batch size $M$, max iterations $I$ and inference steps $K$ scale. 
        Note the \emph{log scale of the x-axis} in the figure for batch $M$ (left). 
    }
    \label{fig:cifar_runtime}
  \end{center}
\end{figure}

\clearpage
\begin{figure}[htbp]
  \begin{center}
    \centerline{
        \includegraphics[width=0.5\columnwidth]{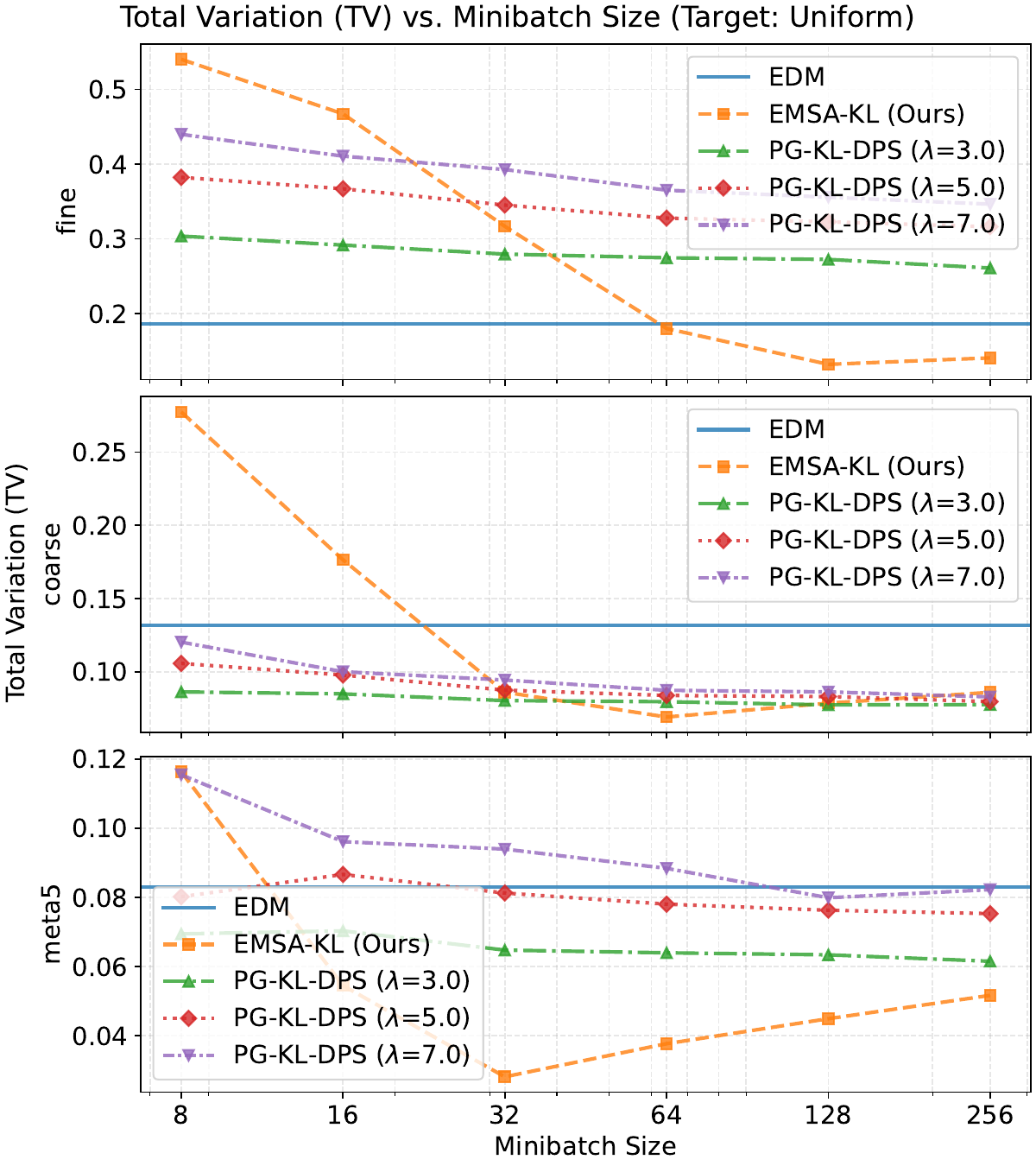}
        \includegraphics[width=0.5\columnwidth]{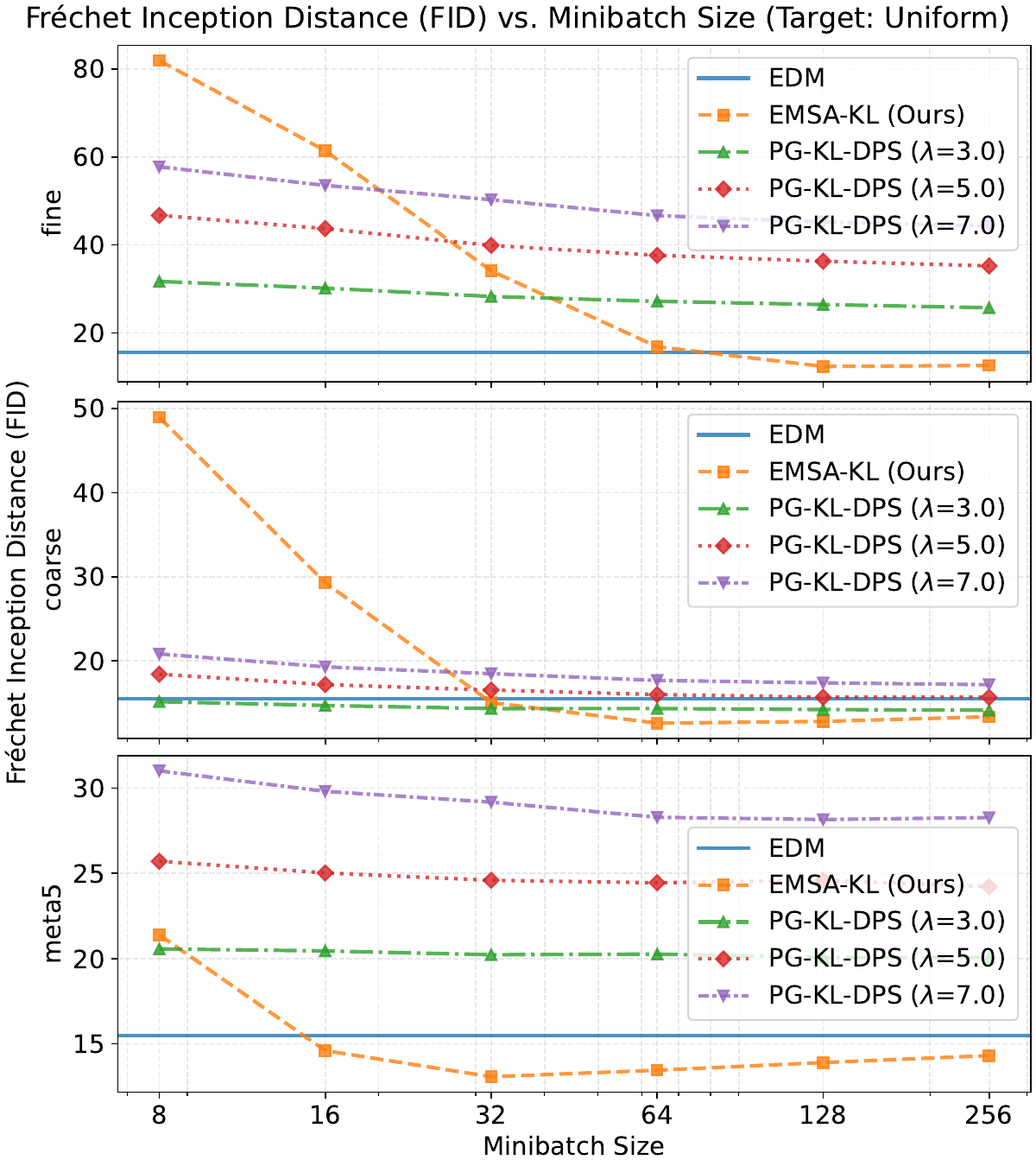}
    }
    \centerline{
        \includegraphics[width=0.5\columnwidth]{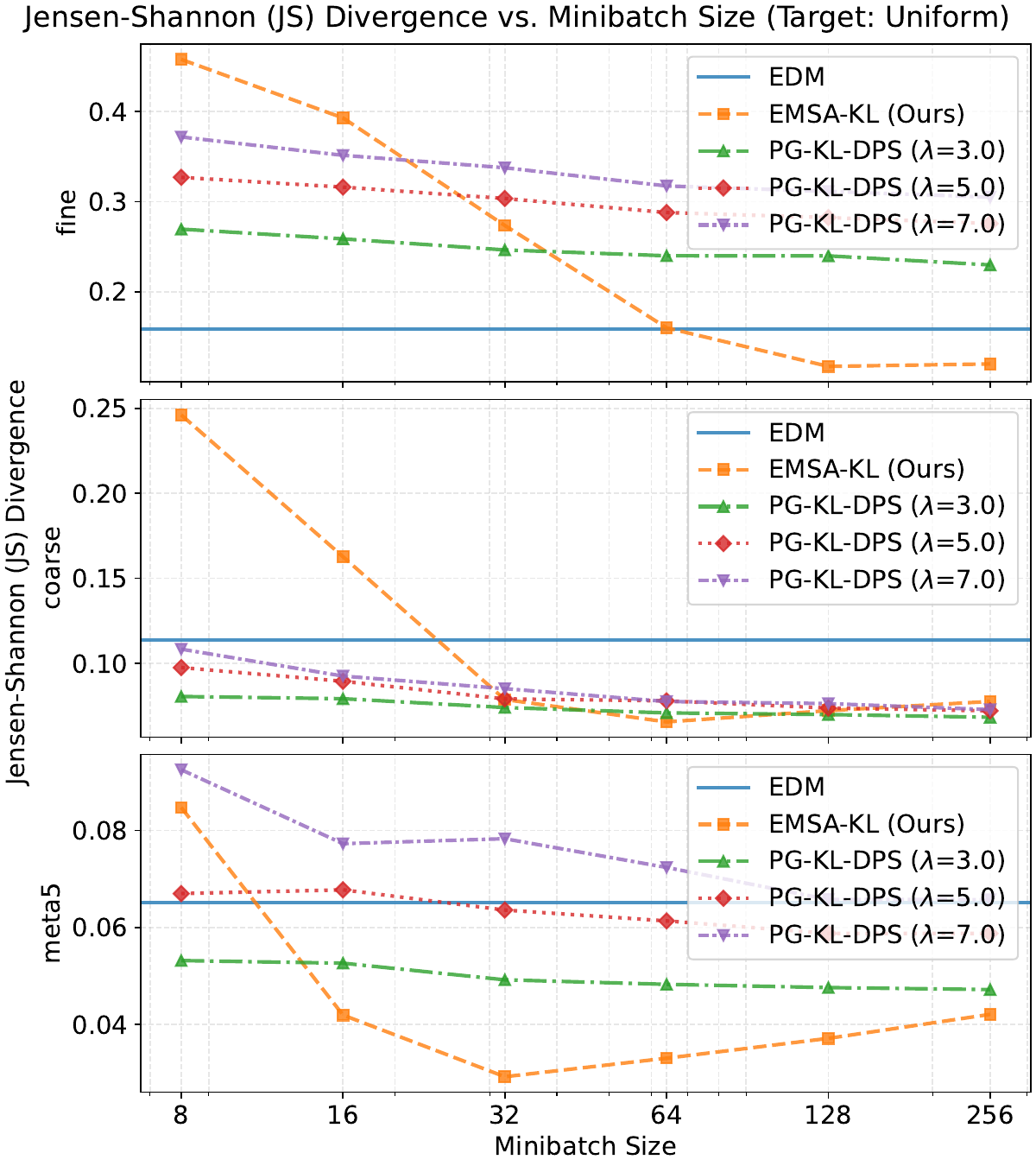}
        \includegraphics[width=0.5\columnwidth]{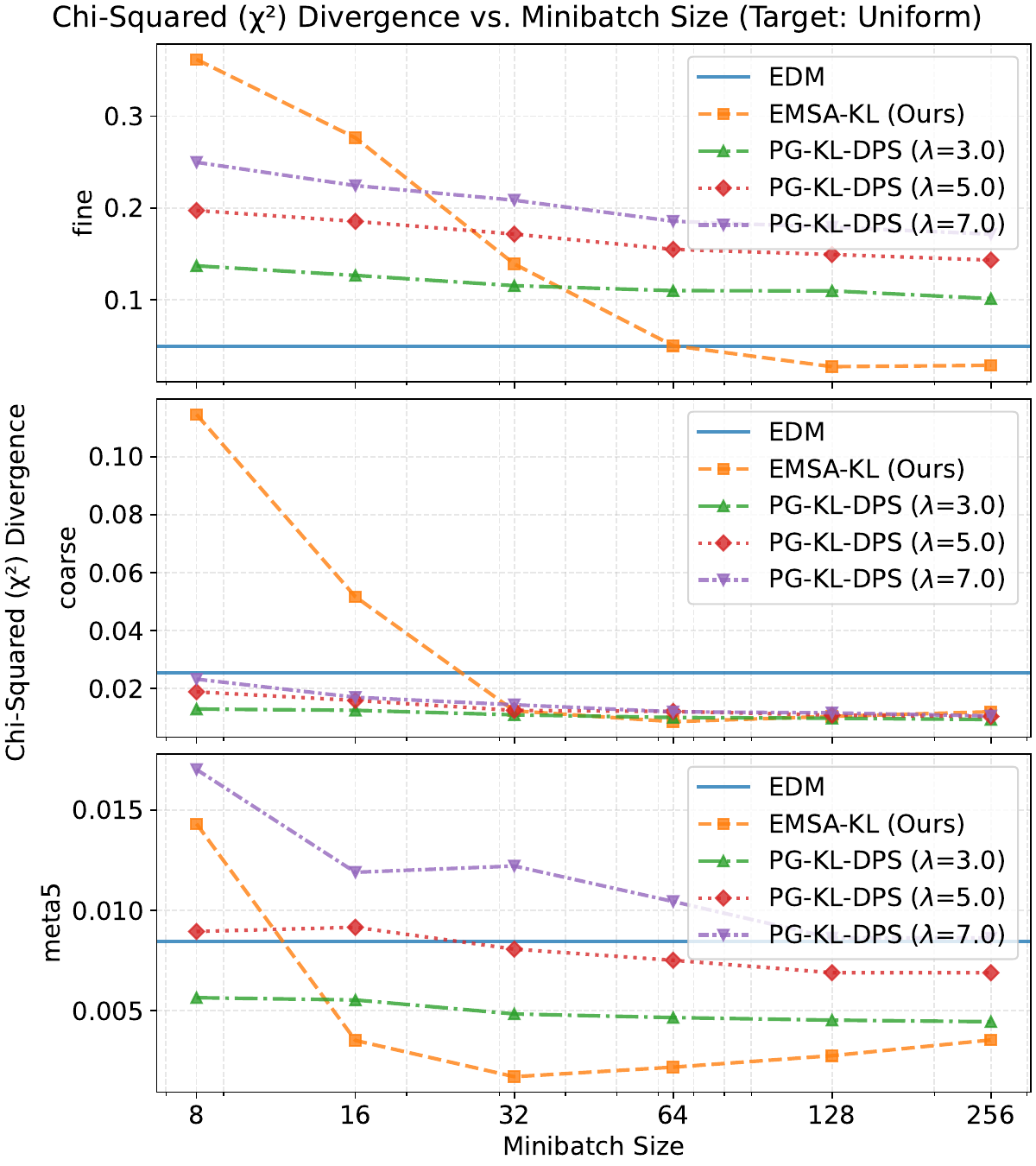}
    }
    \caption{
        Total Variation (top left), FID (top right), Jensen-Shannon divergence, and $\chi^2$ divergence metrics with different batch size $M$ for target $\mathsf{Uniform}$.  
    }
    \label{fig:cifar_uniform_all}
  \end{center}
\end{figure}

\begin{figure}[htbp]
  \begin{center}
    \centerline{
        \includegraphics[width=0.5\columnwidth]{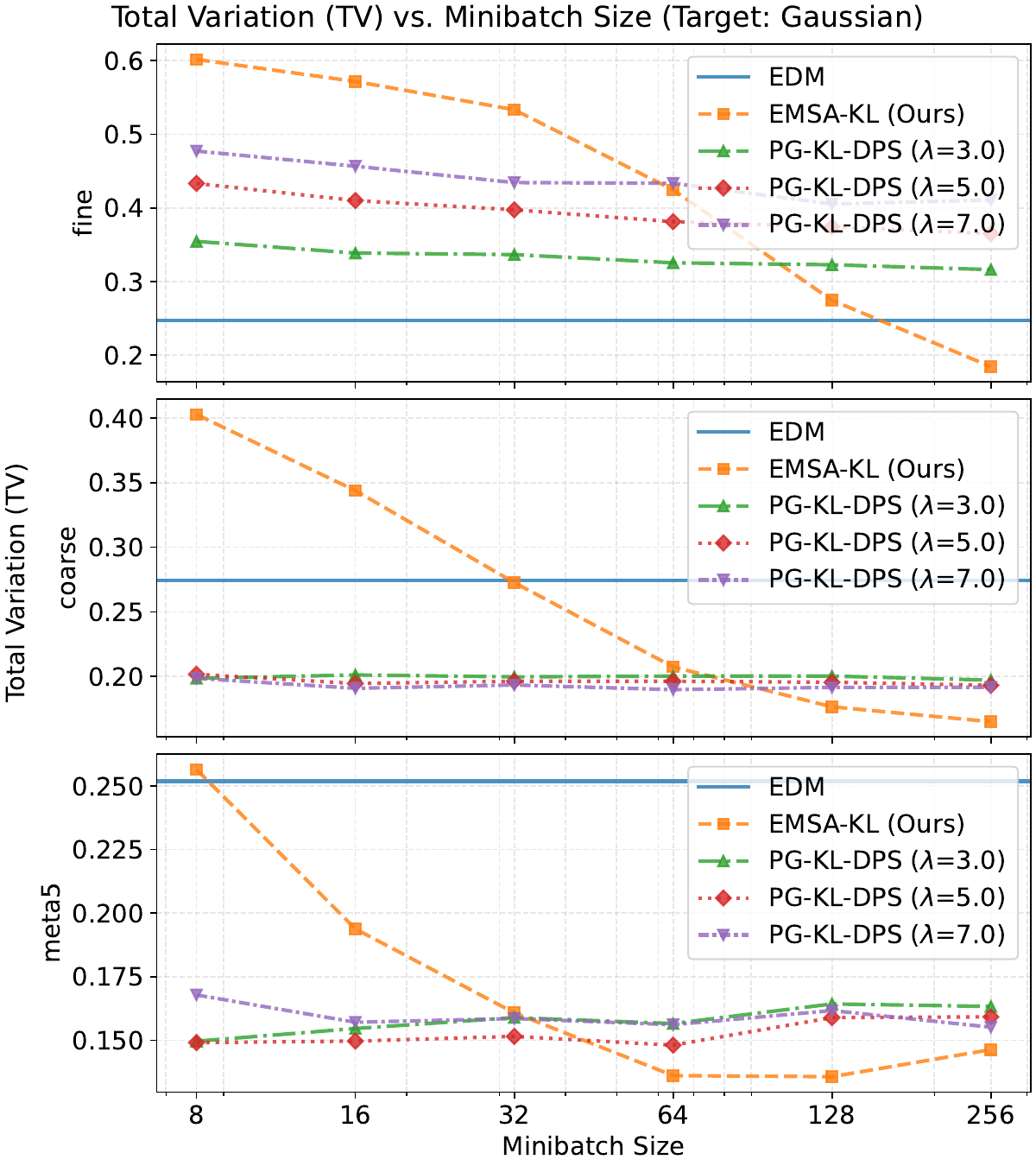}
        \includegraphics[width=0.5\columnwidth]{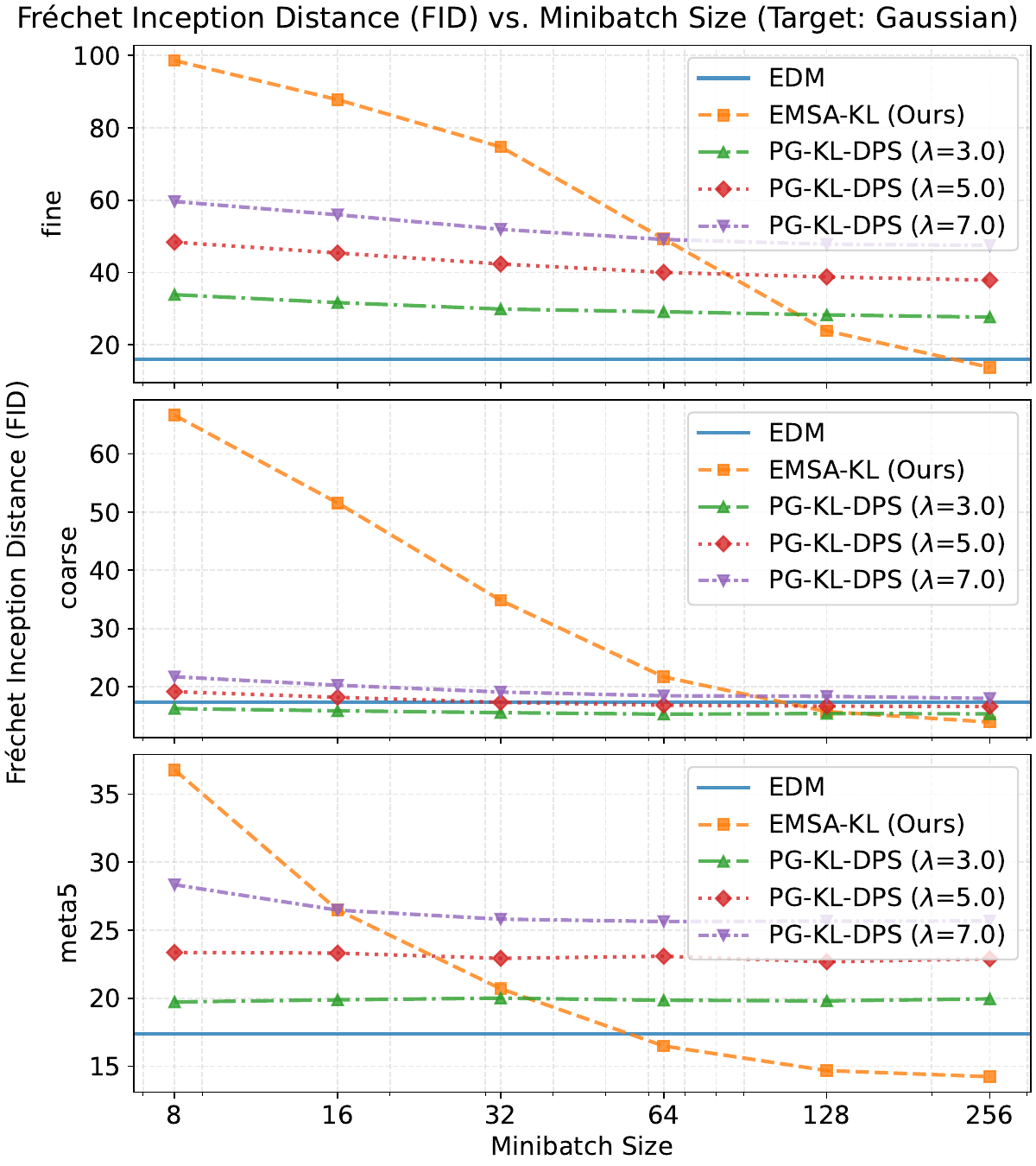}
    }
    \centerline{
        \includegraphics[width=0.5\columnwidth]{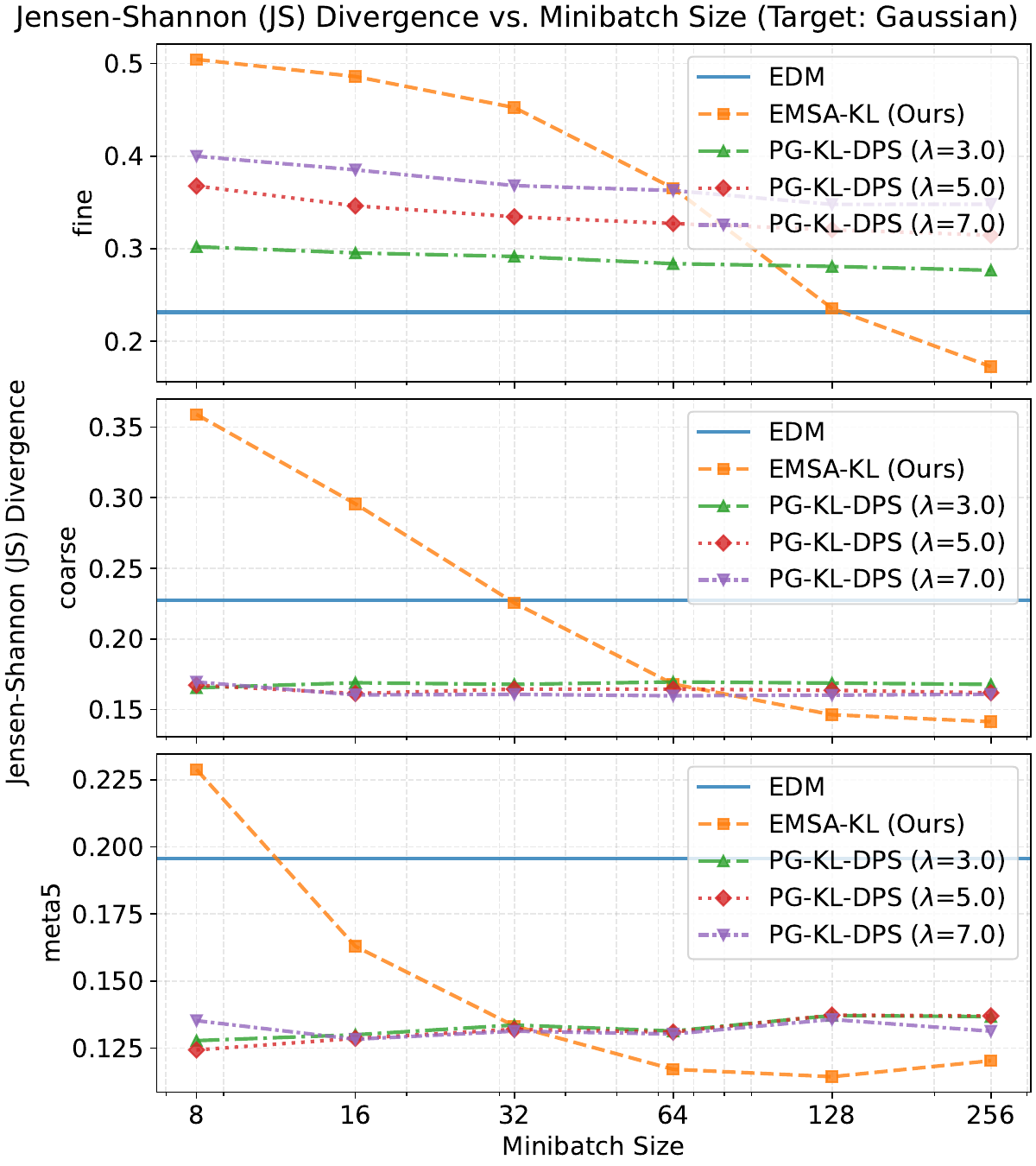}
        \includegraphics[width=0.5\columnwidth]{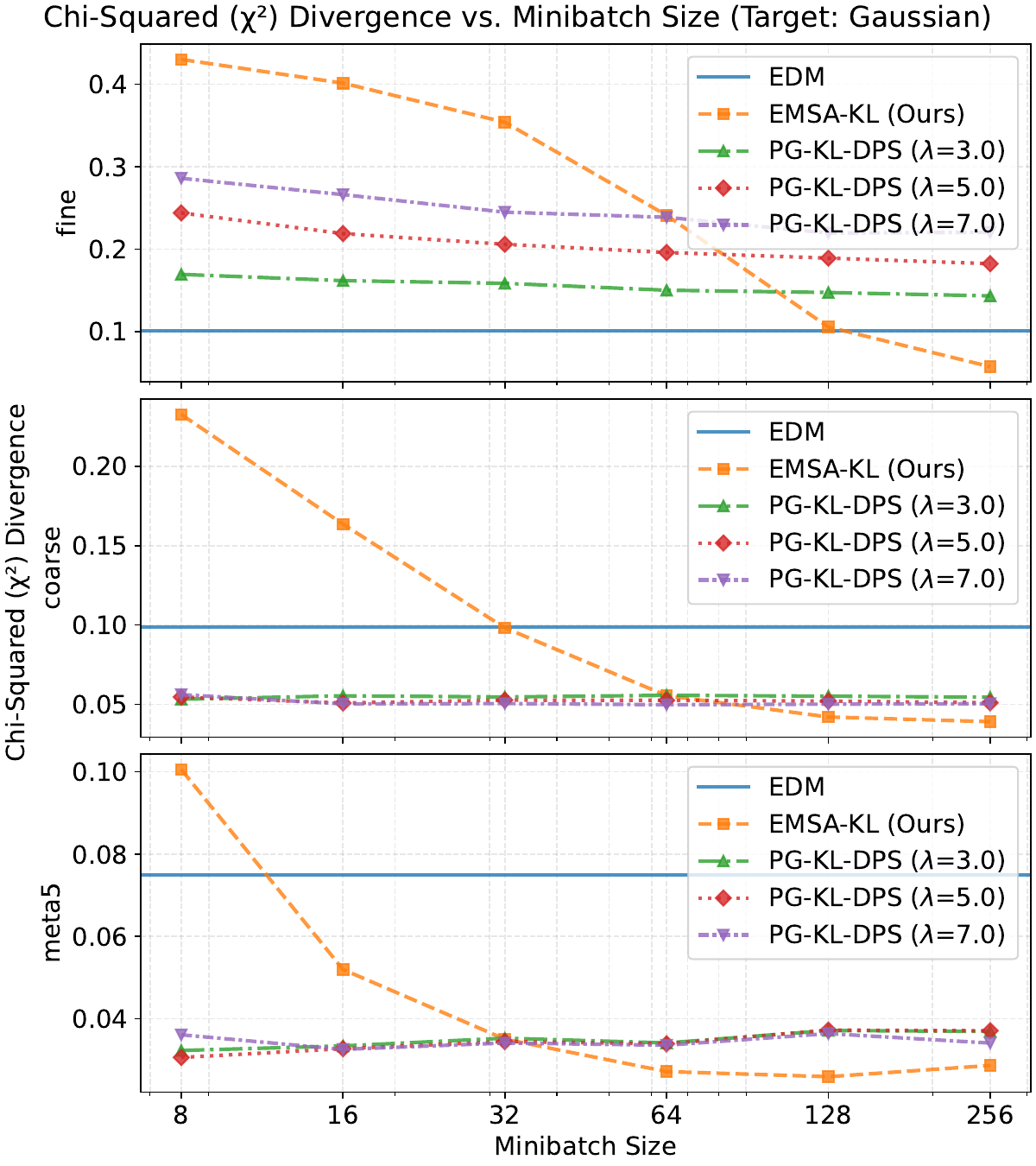}
    }
    \caption{
        Total Variation (top left), FID (top right), Jensen-Shannon divergence, and $\chi^2$ divergence metrics with different batch size $M$ when target is $\mathsf{Gaussian}$.  
    }
    \label{fig:cifar_gaussian_all}
  \end{center}
\end{figure}

\begin{figure}[htbp]
  \begin{center}
    \centerline{
        \includegraphics[width=0.5\columnwidth]{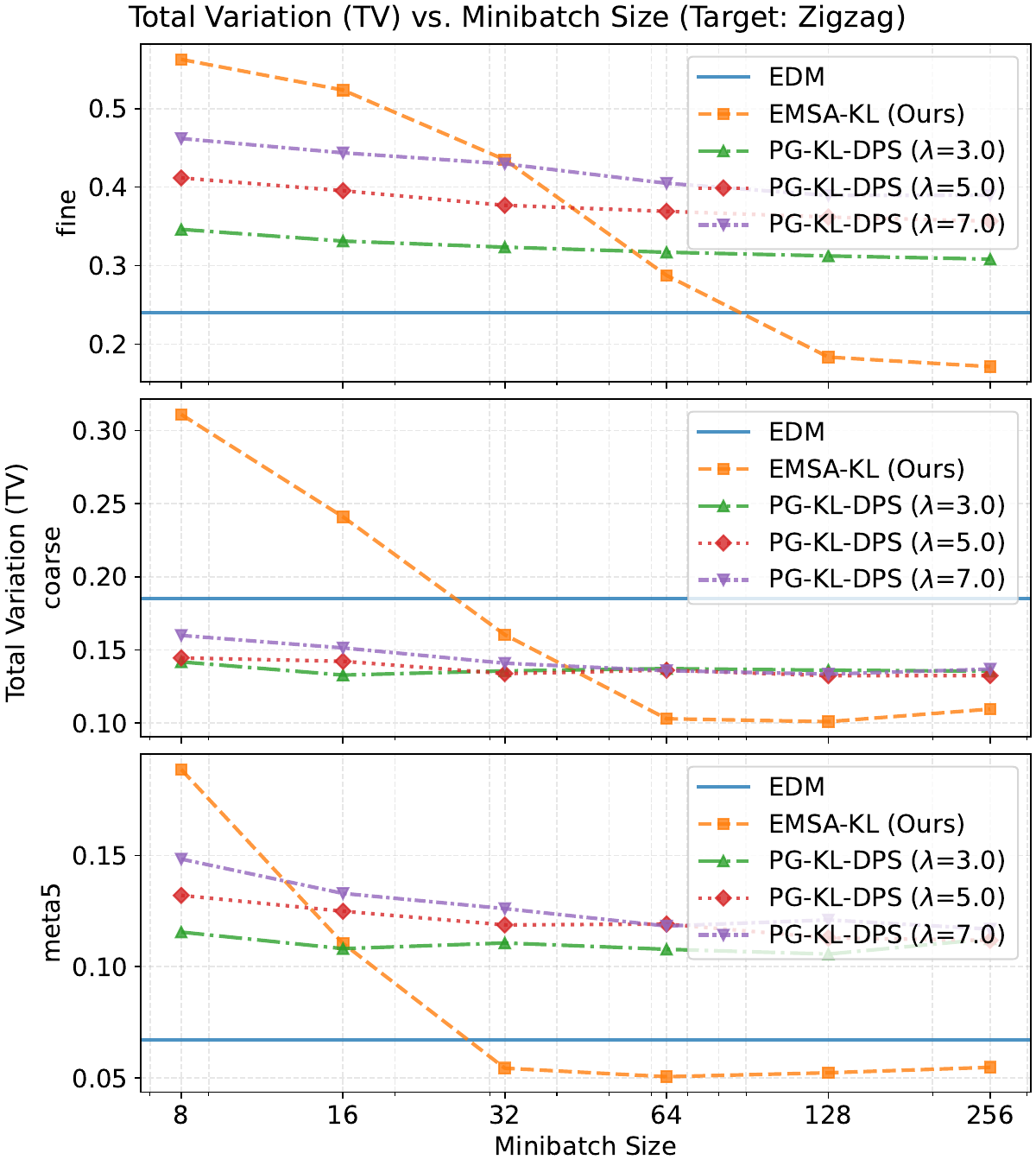}
        \includegraphics[width=0.5\columnwidth]{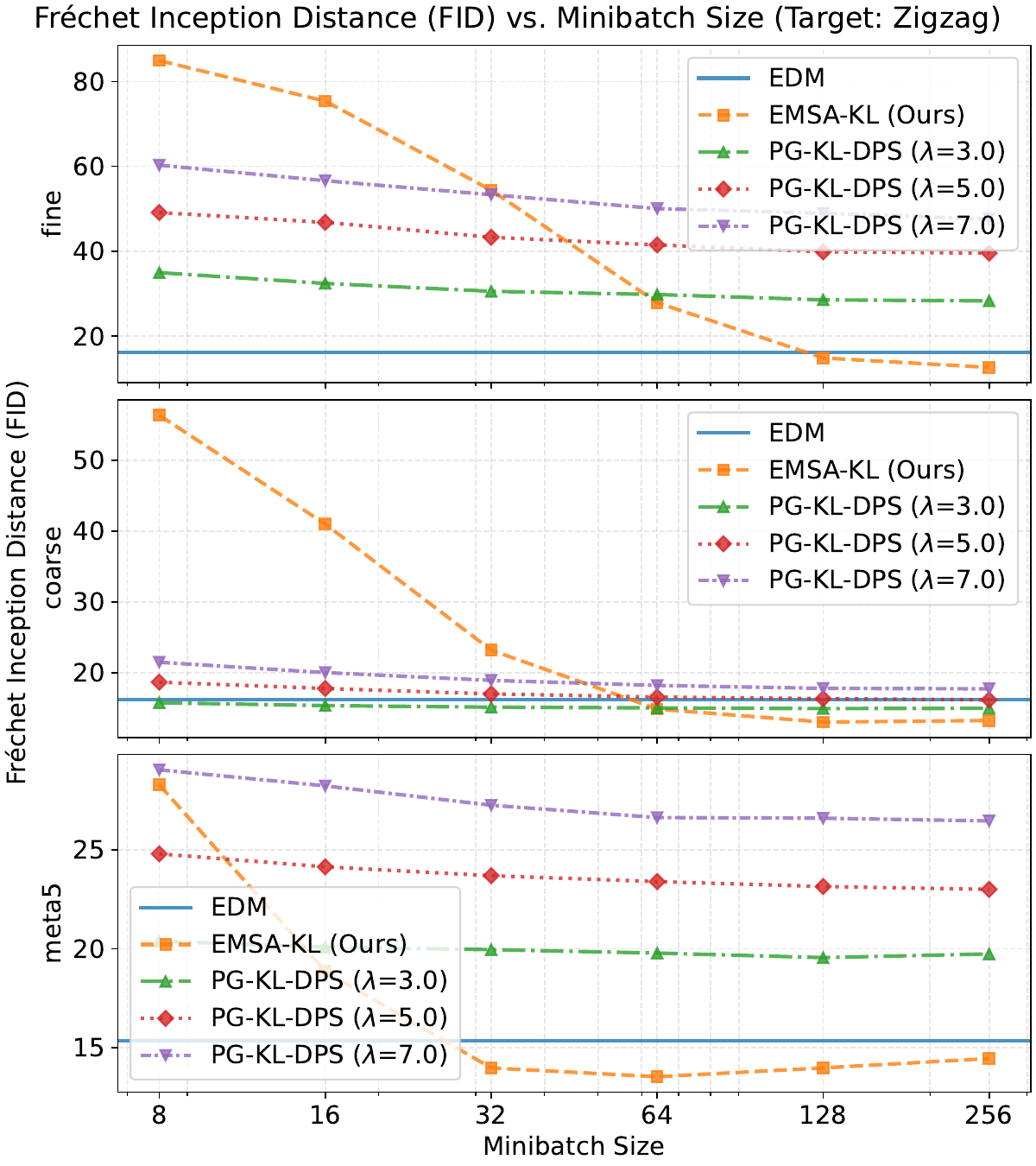}
    }
    \centerline{
        \includegraphics[width=0.5\columnwidth]{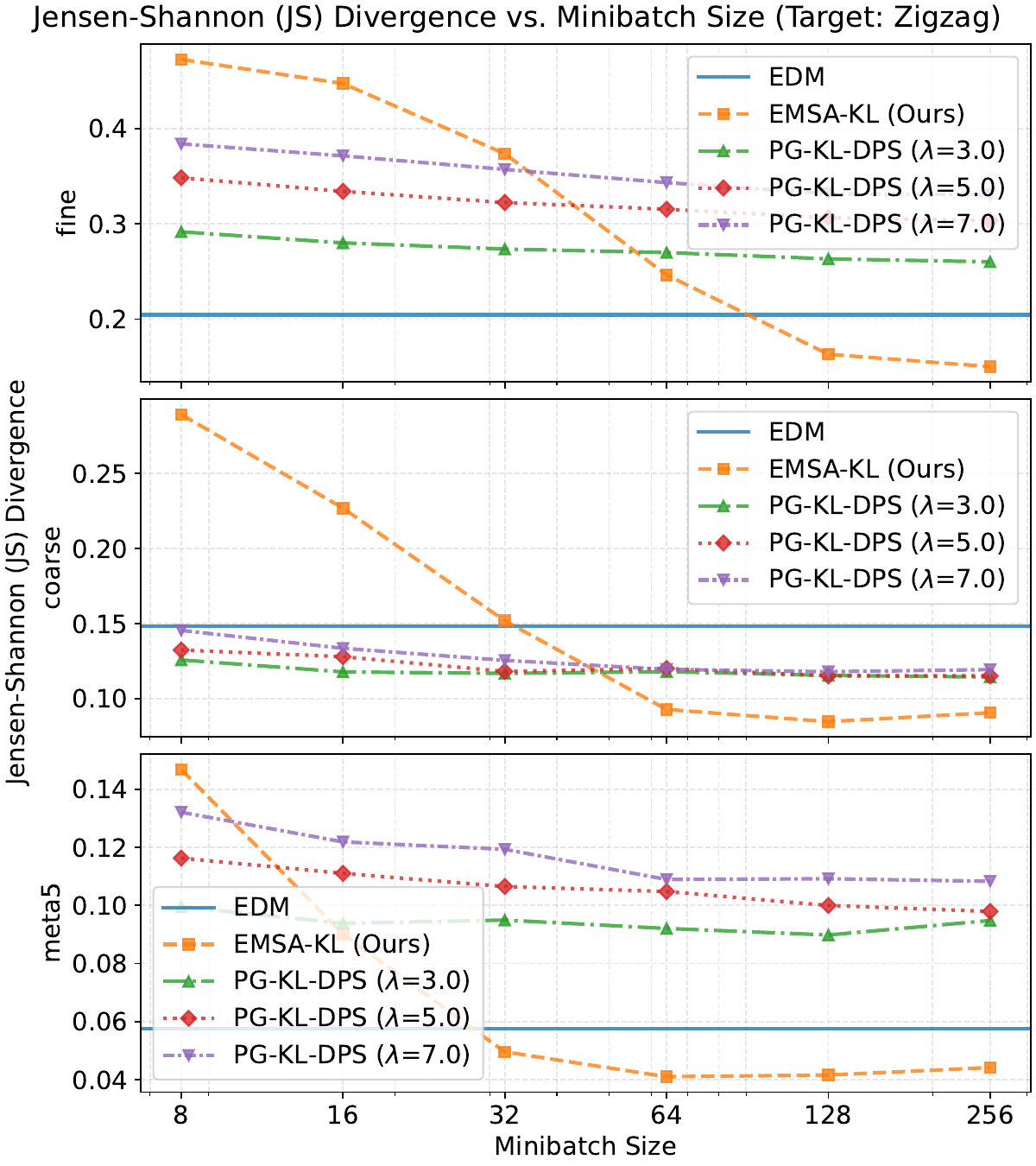}
        \includegraphics[width=0.5\columnwidth]{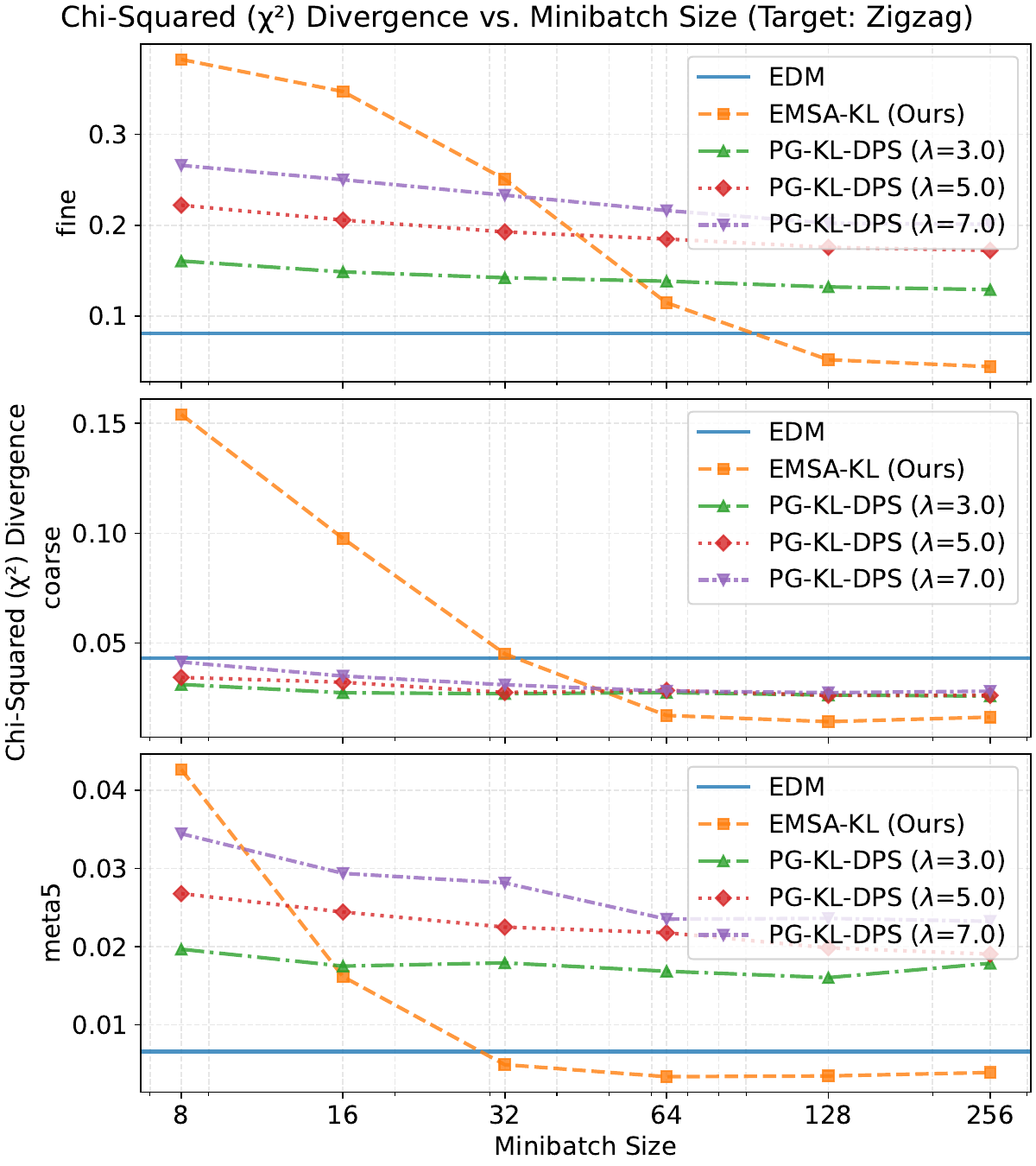}
    }
    \caption{
        Total Variation (top left), FID (top right), Jensen-Shannon divergence, and $\chi^2$ divergence metrics with different batch size $M$ when target is $\mathsf{ZigZag}$.  
    }
    \label{fig:cifar_zigzag_all}
  \end{center}
\end{figure}

\begin{figure}[htbp]
  \begin{center}
    \centerline{
        \includegraphics[width=\columnwidth]{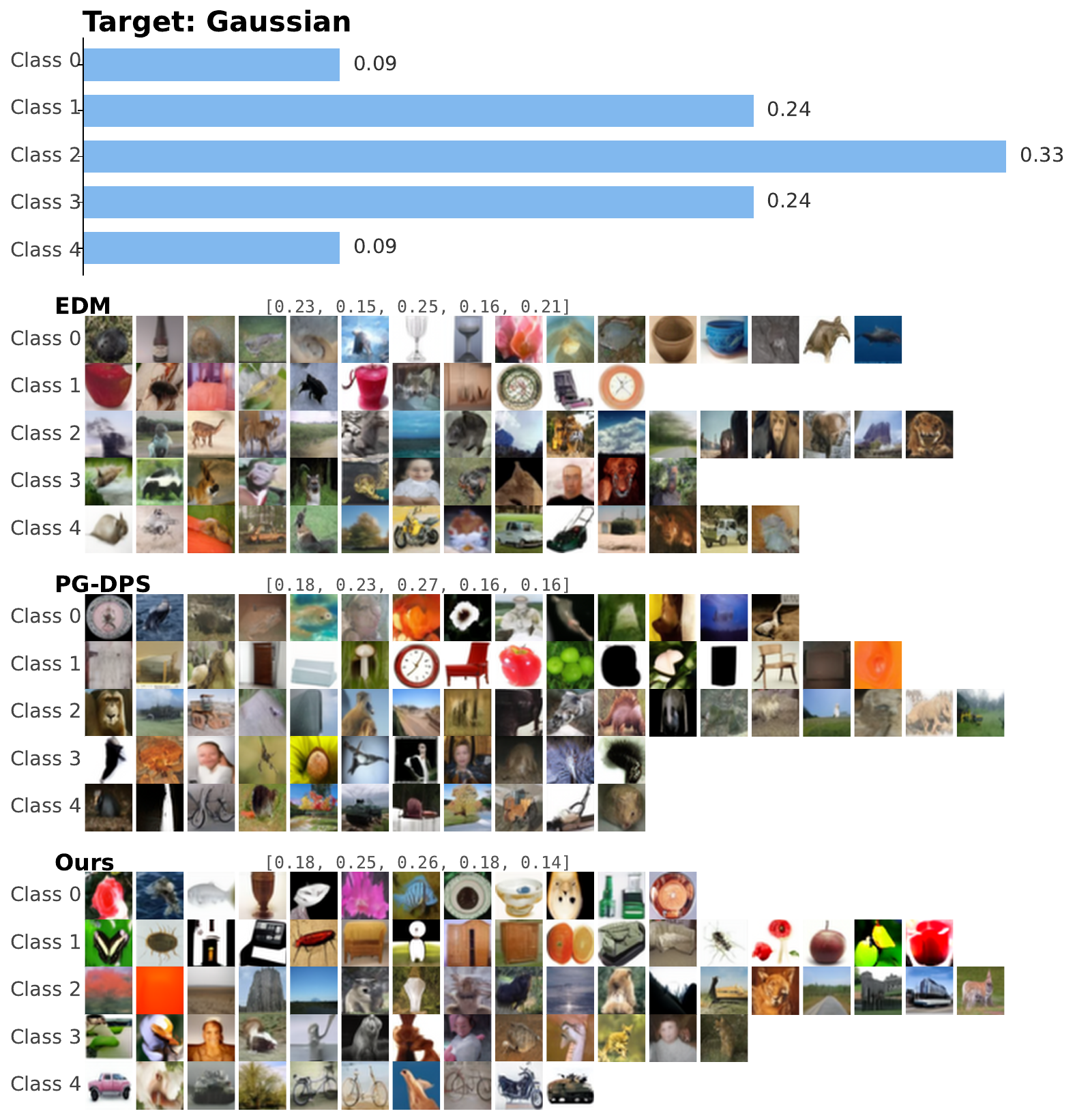}
    }
    \caption{
        Given a target attribute distribution as $\texttt{meta5}$-$\mathsf{Guassian}$, this figure provides a qualitative comparison of samples generated by different methods in the CIFAR experiment reported in \cref{tab:cifar}. 
        The top block shows the target attribute distribution. 
        Each block along the vertical direction exhibits samples generated by one method, clustered by the samples' attribute values, and the empirical attribute distribution is also marked.  
    }
    \label{fig:cifar_samples_meta5_gaussian}
  \end{center}
\end{figure}

\begin{figure}[htbp]
  \begin{center}
    \centerline{
        \includegraphics[width=\columnwidth]{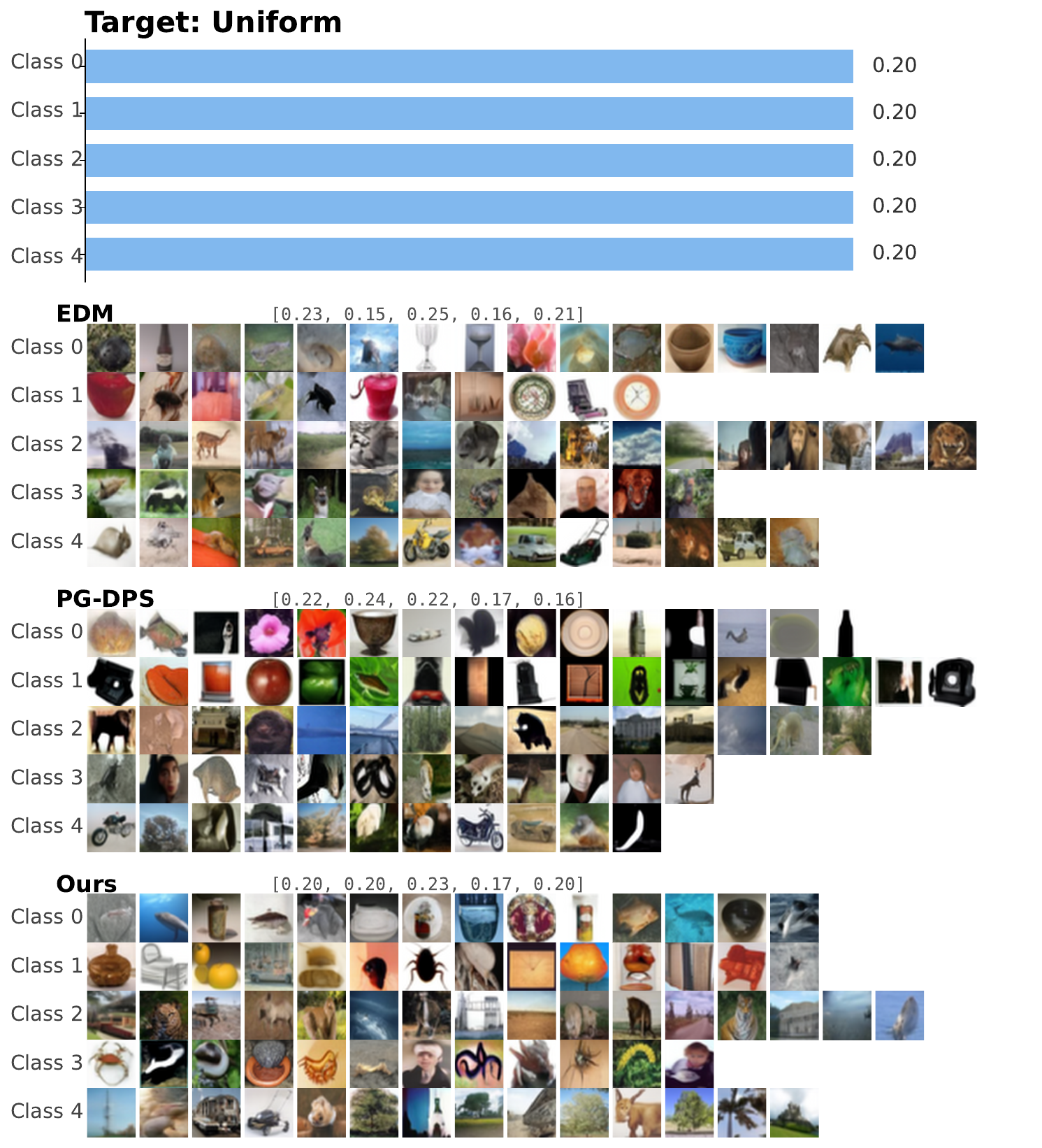}
    }
    \caption{
        Given a target attribute distribution as $\texttt{meta5}$-$\mathsf{Uniform}$, this figure provides a qualitative comparison of samples generated by different methods in the CIFAR experiment reported in \cref{tab:cifar}. 
        The top block shows the target attribute distribution. 
        Each block along the vertical direction exhibits samples generated by one method, clustered by the samples' attribute values, and the empirical attribute distribution is also marked.  
    }
    \label{fig:cifar_samples_meta5_uniform}
  \end{center}
\end{figure}

\begin{figure}[htbp]
  \begin{center}
    \centerline{
        \includegraphics[width=\columnwidth]{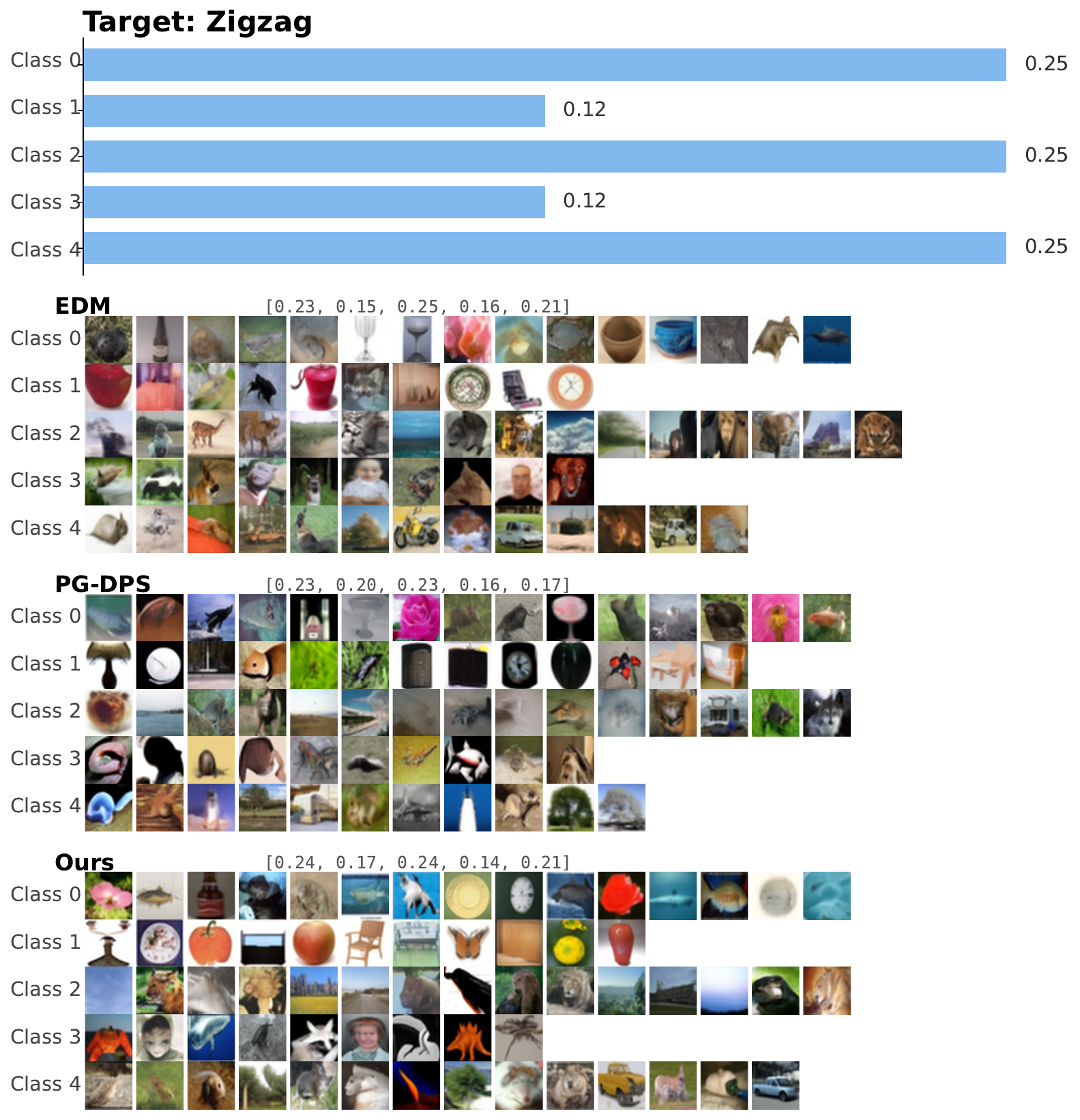}
    }
    \caption{
        Given a target attribute distribution as $\texttt{meta5}$-$\mathsf{Zigzag}$, this figure provides a qualitative comparison of samples generated by different methods in the CIFAR experiment reported in \cref{tab:cifar}. 
        The top block shows the target attribute distribution. 
        Each block along the vertical direction exhibits samples generated by one method, clustered by the samples' attribute values, and the empirical attribute distribution is also marked.  
    }
    \label{fig:cifar_samples_meta5_zigzag}
  \end{center}
\end{figure}

\subsection{Additional Results for Face Generation Experiment}
\label{apdxsub:additional_face}

\paragraph{Additional Samples.}
We present additional qualitative samples in \cref{fig:face_1}, 
\cref{fig:face_2}, \cref{fig:face_3}.

\begin{figure}[htpb]
  \begin{center}
    \centerline{
        \includegraphics[width=0.48\textwidth]{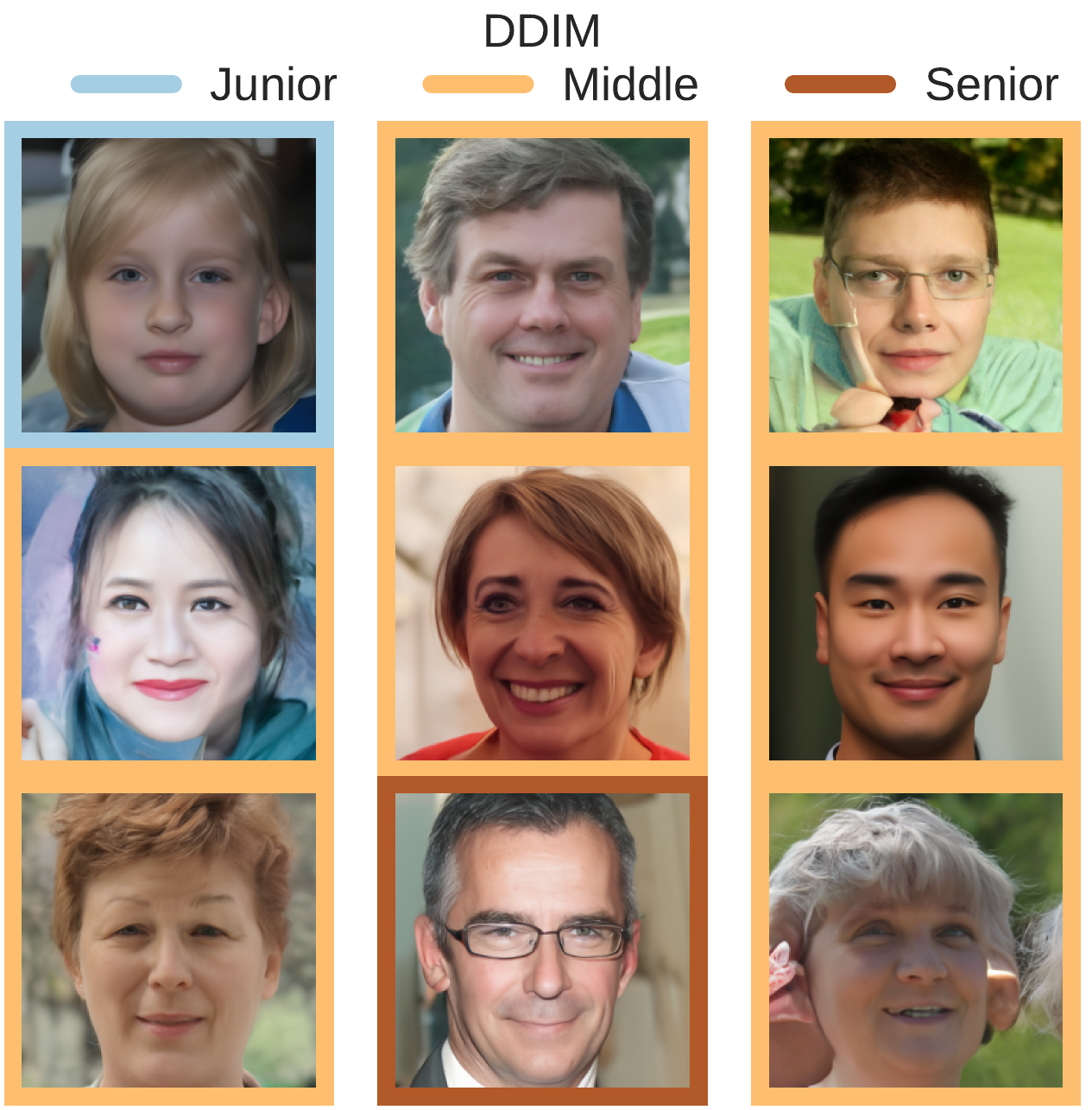}
        \hfill
        \includegraphics[width=0.48\textwidth]{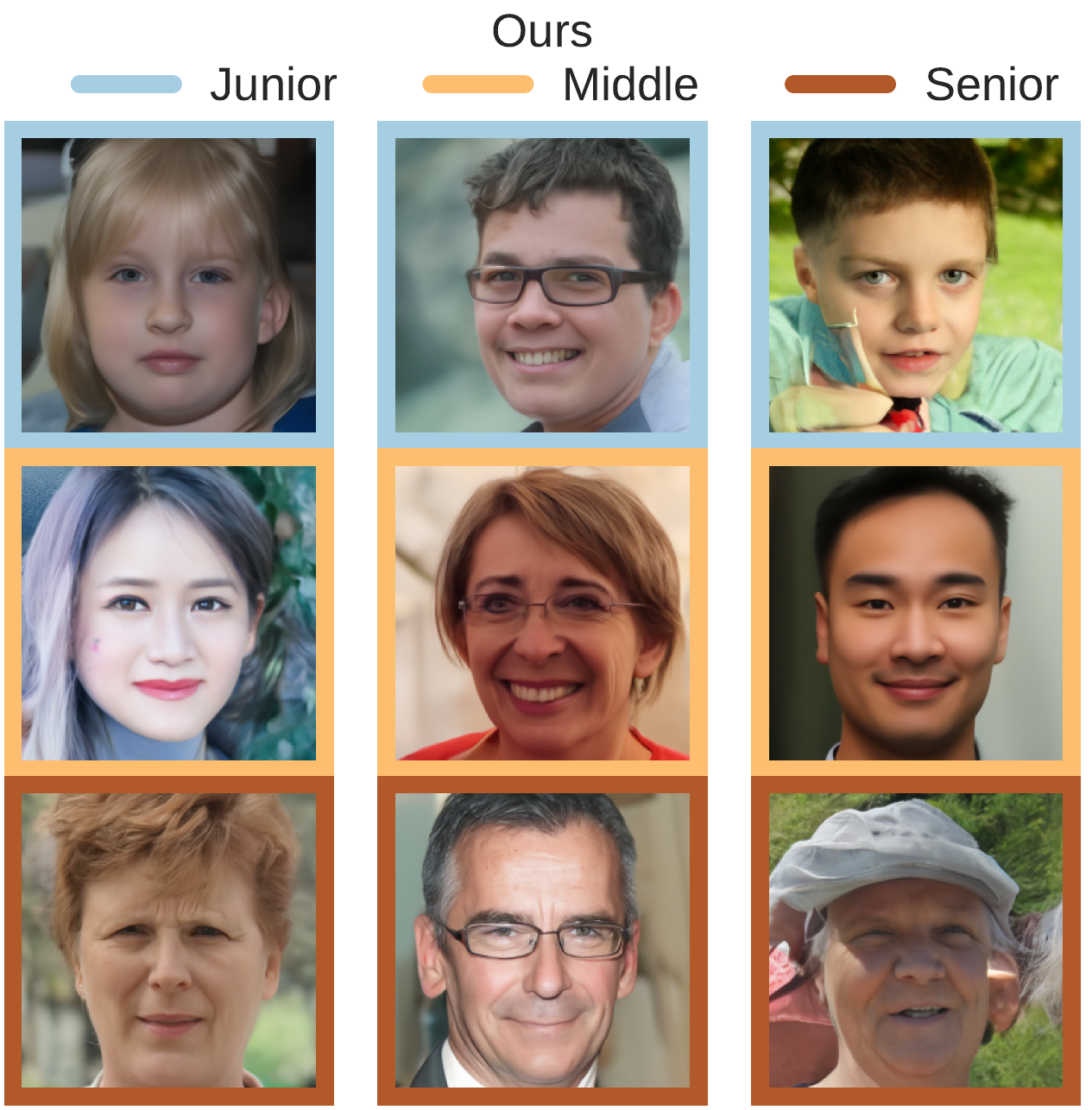}
    }
    \centerline{
        \includegraphics[width=0.48\textwidth]{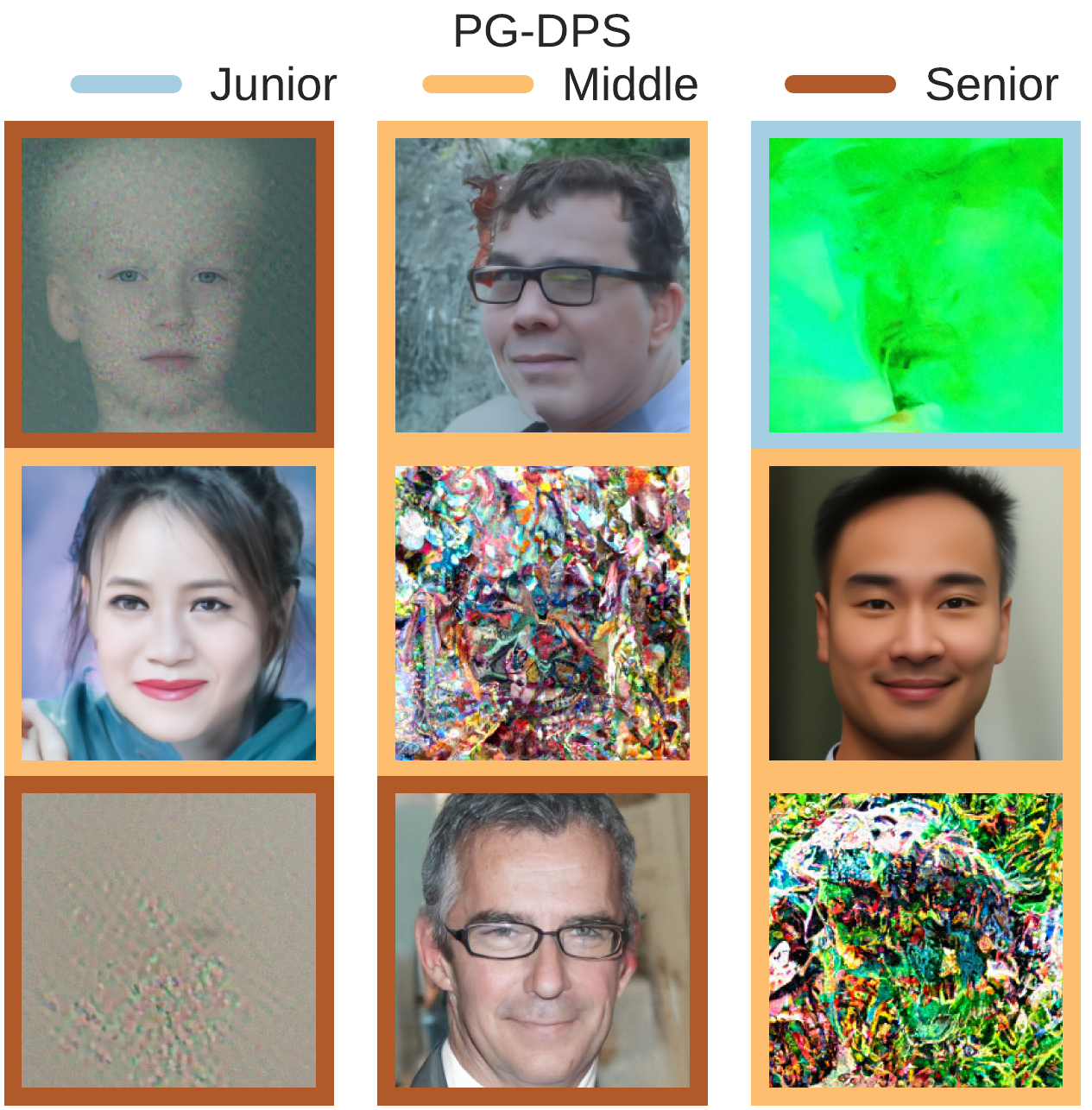}
    }
    \caption{
        Additional samples from all methods in the face generation experiment. 
    }
    \label{fig:face_1}
  \end{center}
\end{figure}

\begin{figure}[htpb]
  \begin{center}
    \centerline{
        \includegraphics[width=0.48\textwidth]{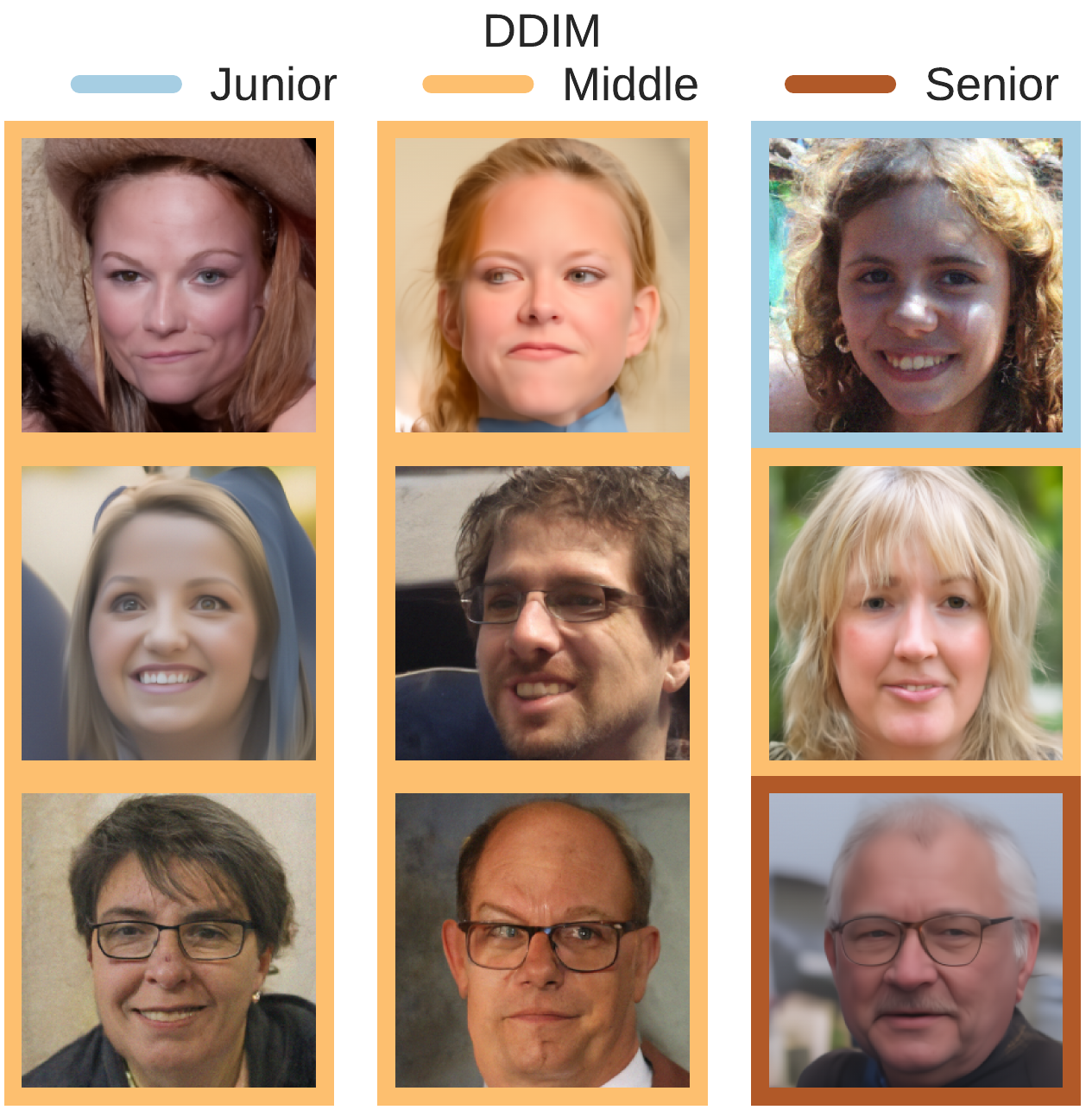}
        \hfill
        \includegraphics[width=0.48\textwidth]{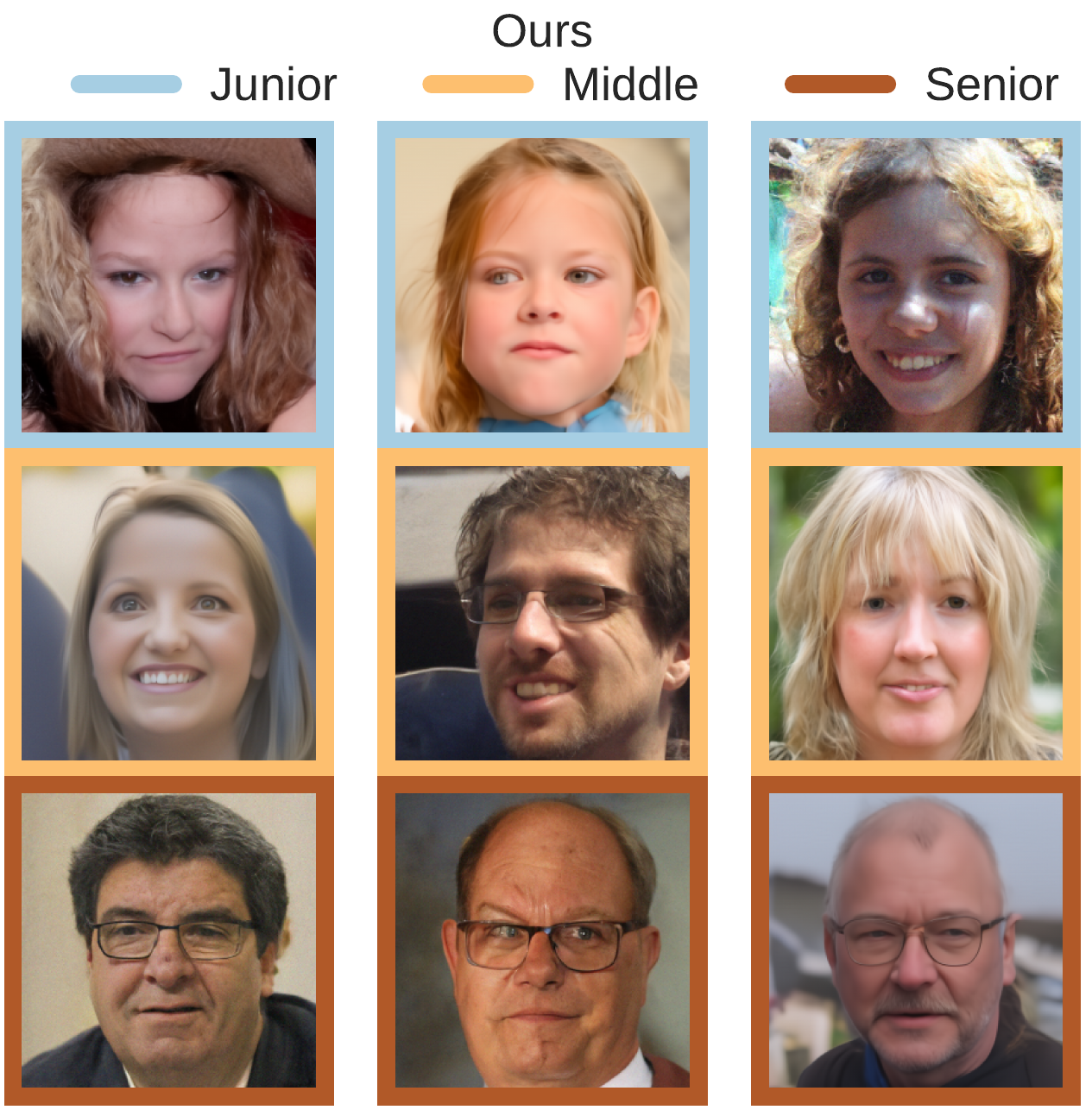}
    }
    \centerline{
        \includegraphics[width=0.48\textwidth]{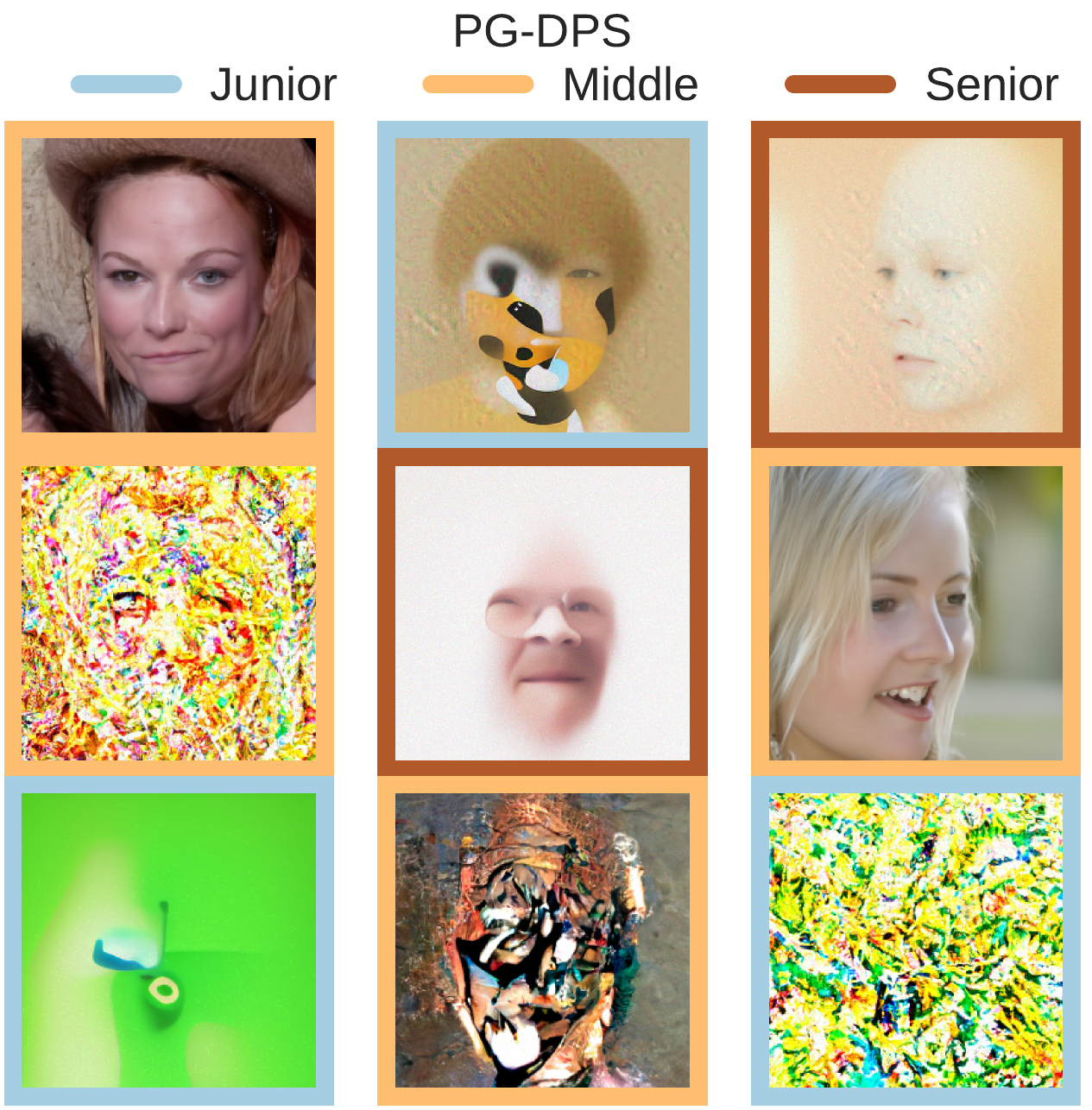}
    }
    \caption{
        Additional samples from all methods in the face generation experiment. 
    }
    \label{fig:face_2}
  \end{center}
\end{figure}

\begin{figure}[htpb]
  \begin{center}
    \centerline{
        \includegraphics[width=0.48\textwidth]{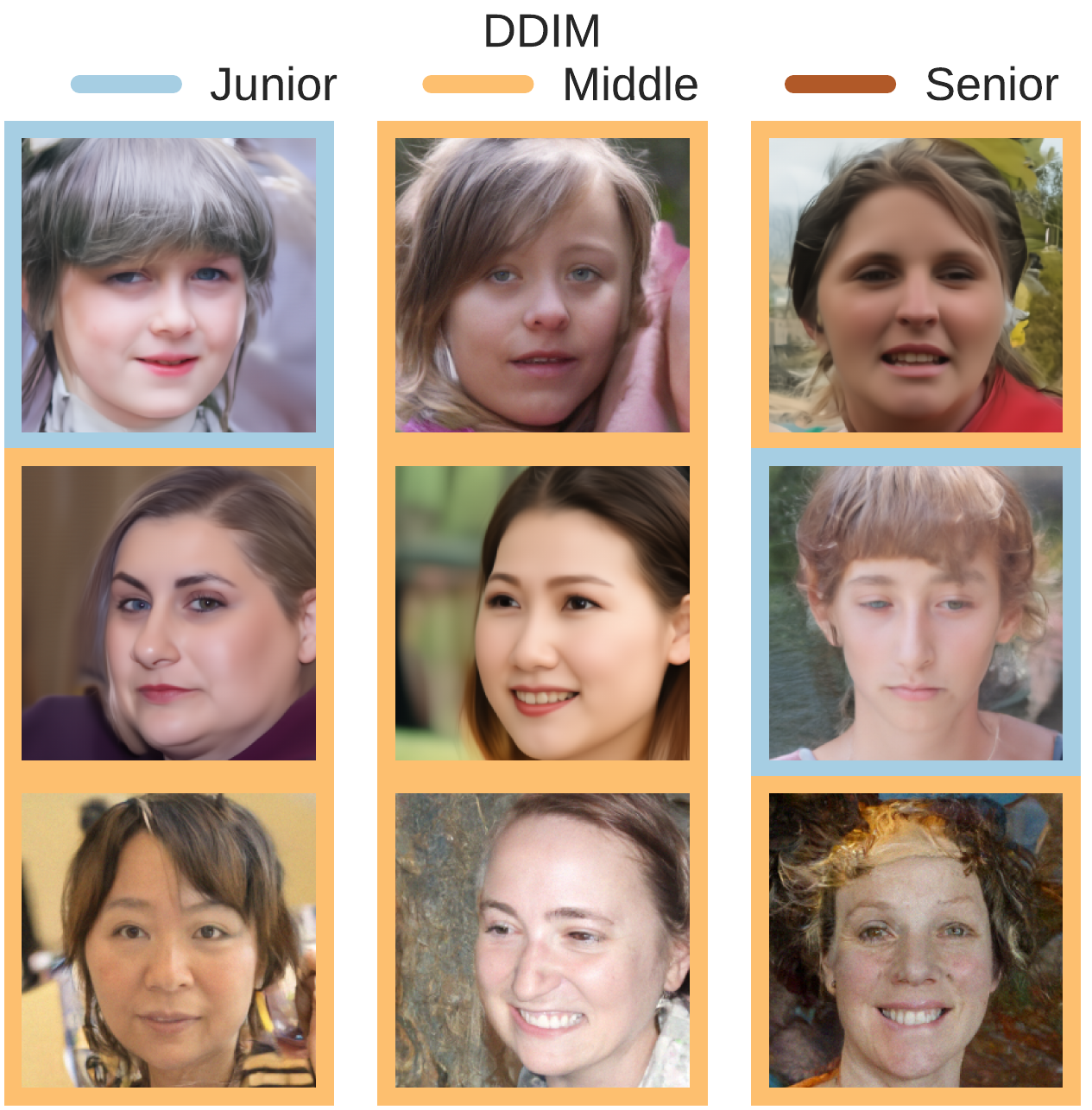}
        \hfill
        \includegraphics[width=0.48\textwidth]{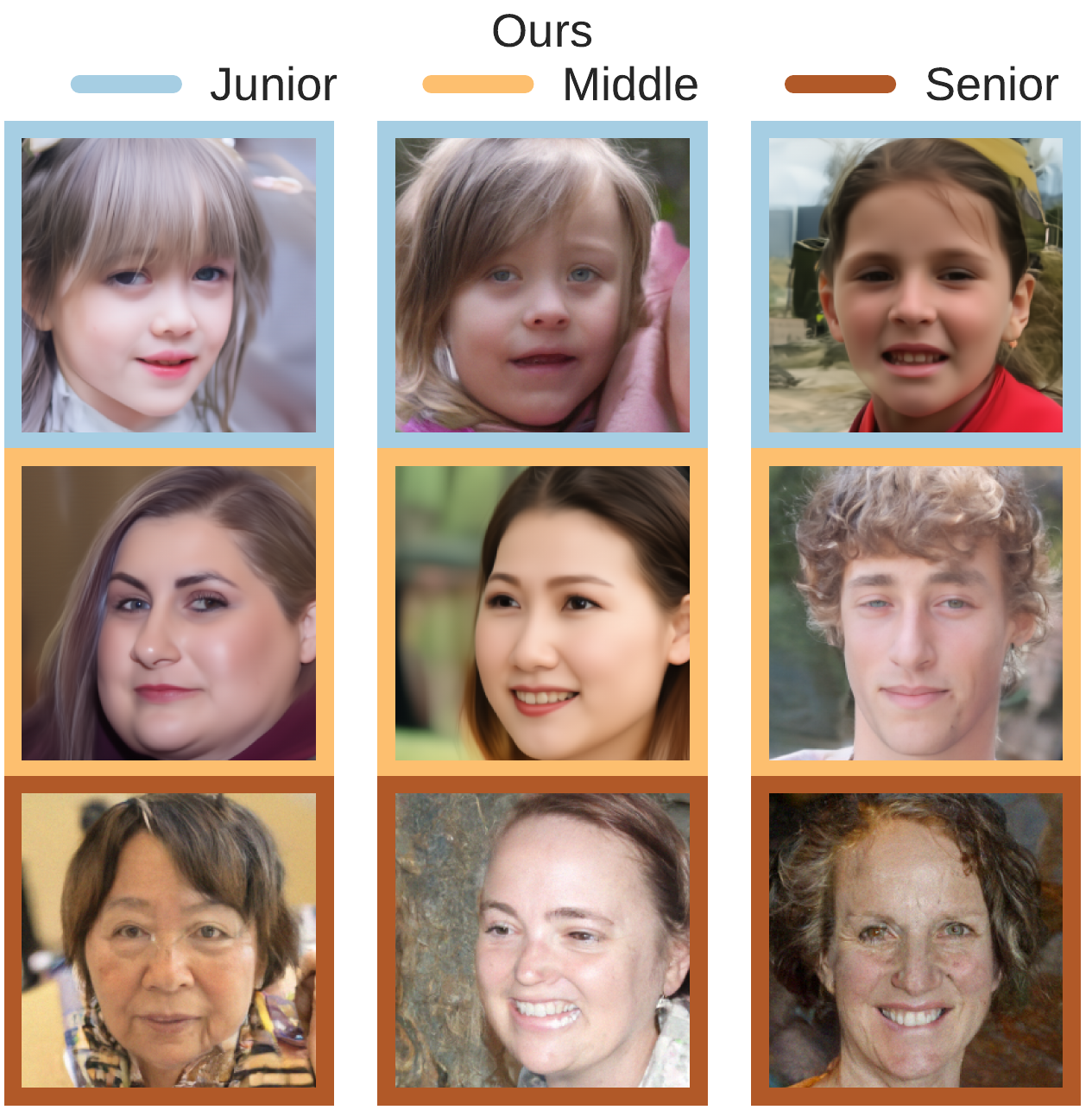}
    }
    \centerline{
        \includegraphics[width=0.48\textwidth]{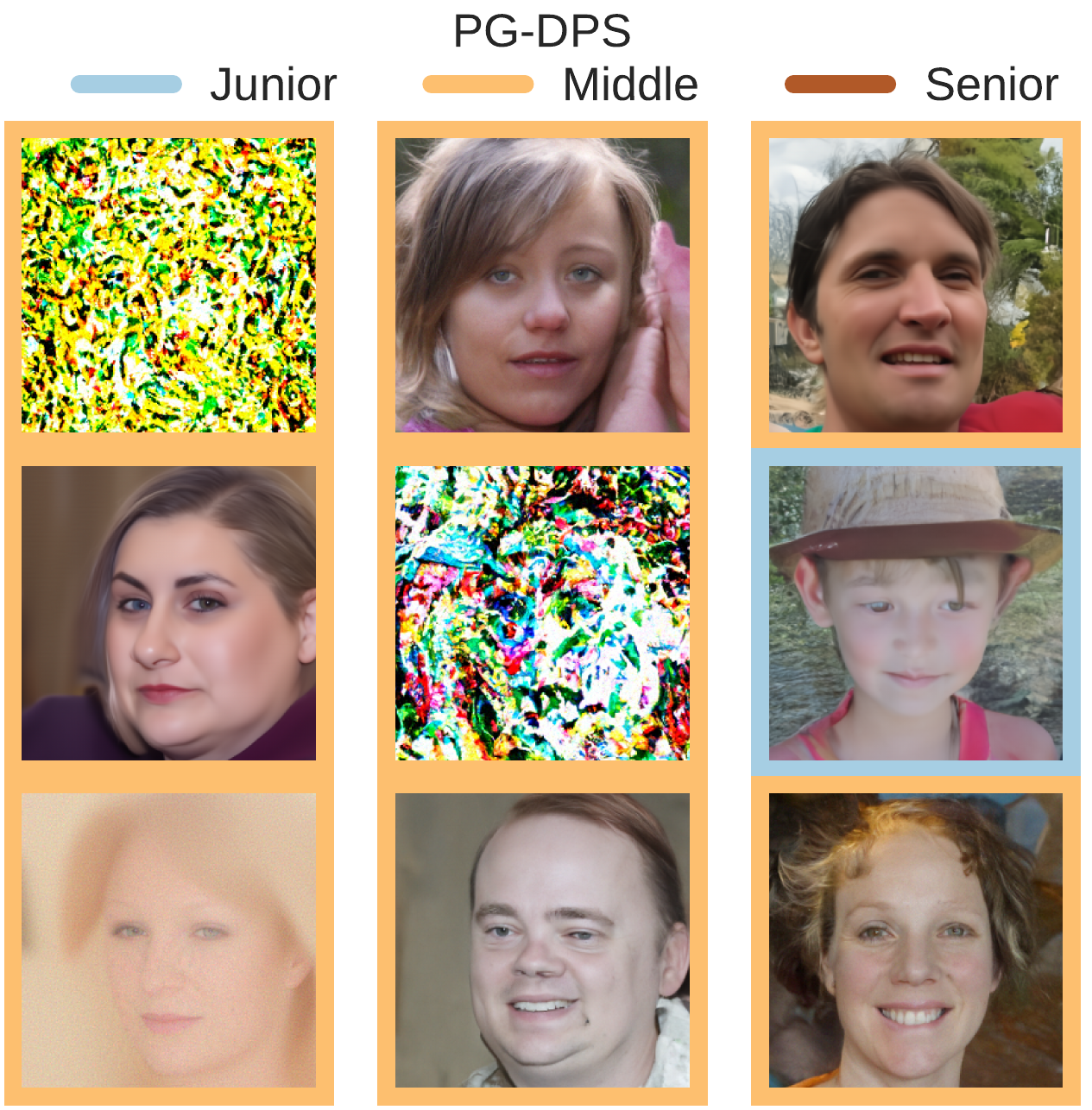}
    }
    \caption{
        Additional samples from all methods in the face generation experiment. 
    }
    \label{fig:face_3}
  \end{center}
\end{figure}

\clearpage
\section{Discussions and Limitations}
\label{apdx:limitations}

\paragraph{Computational Cost}
As analyzed in \apdxref{apdxsub:complexity}, the computational complexity of our method scales linearly with the number of optimization iterations $I$, inference steps $K$, and batch size $M$. 
The primary bottleneck is the per-step VJP computation, which necessitates differentiating through the diffusion model. 
Empirically, we find $I\approx 6\text{ -- }12$ generally suffices, and early stopping can further reduce runtime when the distributional objective plateaus.
While our approach incurs a higher inference cost than the compared baselines, it \emph{effectively eliminates the need for expensive model finetuning or retraining} and \emph{yields stronger distribution-level alignment} in our experiments. 
We view this computational overhead as a necessary trade-off for high-fidelity distributional matching without model parameter updates.

\paragraph{Batch Size for Empirical Distribution Estimation}
Our method needs a reasonable batch size $M$ for estimating the empirical attribute distribution. 
As shown in the ablation study in \cref{apdxsub:additional_cifar}, when the batch size $M$ is too small to provide a reliable empirical distribution for an attribute distribution with a potentially large support size, the distribution alignment performance would degrade. 
However, the same issue exists for other approaches including training-based methods such as \cite{parihar2024balancing}.

\end{document}